\documentclass[twoside]{article}

\usepackage[accepted]{aistats2020}
\usepackage[utf8]{inputenc} 
\usepackage[T1]{fontenc}    
\usepackage{hyperref}       
\usepackage{url}            
\usepackage{booktabs}       
\usepackage{amsfonts}       
\usepackage{nicefrac}       
\usepackage{microtype}      

\usepackage{helvet}  
\usepackage{courier}  
\usepackage{adjustbox}
\usepackage{xcolor}
\usepackage{amsmath}

\usepackage{ marvosym }

\usepackage{graphicx} 
\usepackage{subfigure} 
\graphicspath{{figures/}}
\usepackage{wrapfig}
\usepackage{algorithm}
\usepackage{algorithmic}

\usepackage{multirow}

\usepackage{breakcites}

\allowdisplaybreaks[1]

\ifx\proof\undefined
\newenvironment{proof}{\par\noindent{\bf Proof\ }}{\hfill\BlackBox\\[2mm]}
\fi

\ifx\theorem\undefined
\newtheorem{theorem}{Theorem}
\fi
\ifx\example\undefined
\newtheorem{example}{Example}
\fi
\ifx\lemma\undefined
\newtheorem{lemma}[theorem]{Lemma}
\fi

\ifx\corollary\undefined
\newtheorem{corollary}[theorem]{Corollary}
\fi

\ifx\assumption\undefined
\newtheorem{assumption}{Assumption}
\fi

\ifx\definition\undefined
\newtheorem{definition}{Definition}
\fi

\ifx\proposition\undefined
\newtheorem{proposition}[theorem]{Proposition}
\fi

\ifx\remark\undefined
\newtheorem{remark}{Remark}
\fi

\ifx\conjecture\undefined
\newtheorem{conjecture}[theorem]{Conjecture}
\fi

\ifx\factoid\undefined
\newtheorem{factoid}[theorem]{Fact}
\fi

\ifx\axiom\undefined
\newtheorem{axiom}[theorem]{Axiom}
\fi

\newcommand{\RN}[1]{%
	\textup{\lowercase\expandafter{\it \romannumeral#1}}%
}

\def\ReLU{\textsf{ReLU}} 

\input{Definitions}

\begin{document}

%

%

\twocolumn[

\aistatstitle{Stochastic Particle-Optimization Sampling and the Non-Asymptotic Convergence Theory}

\aistatsauthor{Jianyi Zhang$^1$ \And Ruiyi Zhang$^1$ \And Lawrence Carin$^1$ \And Changyou Chen$^2\href{mailto:cchangyou@gmail.com}{\text{\Email}}$}

\aistatsaddress{$^1$Duke University \And  $^2$University at Buffalo, SUNY } ]

\begin{abstract}
	Particle-optimization-based sampling (POS) is a recently developed effective sampling technique that interactively updates a set of particles. 
	A representative algorithm is the Stein variational gradient descent (SVGD). 
	We prove,  under certain conditions, SVGD experiences a theoretical pitfall, {\it i.e.}, particles tend to collapse. As a remedy, we generalize POS to a stochastic setting by injecting random noise into particle updates, thus yielding particle-optimization sampling (SPOS). Notably, for the first time, we develop {\em non-asymptotic convergence theory} for the SPOS framework (related to SVGD), characterizing algorithm convergence in terms of the 1-Wasserstein distance w.r.t.\! the numbers of particles and iterations. Somewhat surprisingly, with the same number of updates (not too large) for each particle, our theory suggests adopting more particles does not necessarily lead to a better approximation of a target distribution, due to limited computational budget and numerical errors. This phenomenon is also observed in SVGD and verified via an experiment on synthetic data. 
	Extensive experimental results verify our theory and demonstrate the effectiveness of our proposed framework.
\end{abstract}

\section{Introduction}
\vspace{-0.0cm}
Recently there has been extensive development of scalable Bayesian sampling algorithms, such as stochastic gradient MCMC (SG-MCMC)  \cite{WellingT:ICML11,ChenFG:ICML14,DingFBCSN:NIPS14,ChenDC:NIPS15} and Stein variational gradient descent (SVGD) \cite{LiuW:NIPS16}. SG-MCMC is a family of scalable Bayesian sampling algorithms built on It\'{o} diffusions, stochastic differential equations (SDEs) with appropriately designed coefficients whose stationary distributions match the target distributions. One potential issue of SG-MCMC is that samples may be highly correlated partially due to the nature of Markov chains, leading to undesired low sample efficiency.
SVGD, on the other hand, belongs to the family of particle-optimization-based sampling methods that optimize a set of interacting particles to minimize some distance metric ({\it e.g.}, KL-divergence) between the target distribution and the particle-induced approximate  distribution. By optimization, one seeks to maintain an optimal set of particles. Recent development of SVGD has shown that the underlying mathematical principle is based on a family of {\em nonlinear} partial differential equations (PDEs) \cite{liu2017stein_flow}. Although achieving significant practical successes
\cite{LiuW:NIPS16,FengWL:UAI17,liu2017stein,HaarnojaTAL:ICML17,ZhangCLC:ICML,ZhangWCC:AISTATS19,LiuZ:AAAI18}, little theory is available to fully understand its {\em non-asymptotic} convergence properties. A recent theoretical development has interpreted SVGD as a special type of gradient flows, and developed theory to disclose its {\em asymptotic} convergence behavior \cite{liu2017stein_flow}. The asymptotic theory is also studied in \cite{LuLN:arxiv18}. A very recent work \cite{LiuW:NIPS18} investigated non-asymptotic properties of SVGD, limited to the region of finite particles and infinite time with restricted conditions. In \cite{SimsekliLMD:arxiv18} the convergence property of the sliced-Wasserstein flow are only considered under an infinite-particle setting.

Recently, \cite{ChenZWLC:tech18} unified SG-MCMC and SVGD by proposing a particle-optimization-sampling (POS) framework to interpret both as Wasserstein gradient flows (WGFs). Generally, a WGF is a PDE defined on the space of probability measures, describing the evolution of a density over time. \cite{ChenZWLC:tech18} defined a WGF by combining the corresponding PDEs for both SG-MCMC and SVGD, and solved it with deterministic particle approximations. However, due to its diffusion nature, deterministic-particle approximation leads to a hard-to-control error, making it challenging for theoretical analysis.

\paragraph{Our contributions}
In this paper, we generalize POS to a stochastic setting, and develop a novel analytical framework based on granular media equations \cite{Malrieu:AAP03,CattiauxGM:PTRF08} to analyze its non-asymptotic convergence properties. Our contributions are summarized as follows: $\RN{1}$) We first identify a pitfall of standard SVGD, where particles tend to collapse under certain conditions and measurement, indicating that developing non-asymptotic theory for SVGD is changing(if possible at all). $\RN{2}$) Based on the unified framework in \cite{ChenZWLC:tech18}, we propose {\em stochastic} particle-optimization sampling (SPOS) by injecting Gaussian noise in particle updates to overcome the pitfall. $\RN{3}$) For the first time, we develop nonasymptotic convergence theory for the family of SPOS algorithms, considering both convex- and nonconvex-energy targets. Different from existing theory for SG-MCMC-based algorithms \cite{TehTV:arxiv14,VollmerZT:arxiv15,ChenDC:NIPS15,RaginskyRT:COLT17,ZhangLC:COLT17,XuCZG:arxiv17}, our development relies on the theory of {\em nonlinear PDEs}, which is more involved and less explored in the literature. In particular, we adopt tools from granular media equations \cite{Malrieu:AAP03,CattiauxGM:PTRF08} to develop non-asymptotic error bounds in terms of 1-Wasserstein distance. More detailed distinctions between our work and existing work are discussed in Section~\ref{app:related} of the Supplementary Material (SM). Somewhat surprisingly, our theory indicates that utilizing more particles does not necessarily lead to better approximations, due to the numerical errors in the algorithms. This phenomenon is also observed for SVGD. $\RN{4}$) Our theory and advantages of the algorithm are verified via various experiments, including experiments on synthetic data, Bayesian deep learning and Bayesian exploration for reinforcement learning.

\section{Preliminaries}
\vspace{-0.2cm}
\paragraph{Notation}
We use { bold letters} to denote variables in {\em continuous-time diffusions and model definitions} (no numerical methods included yet), {\it e.g.}, $\thetab_{\tau}$ in \eqref{eq:itodif} below (indexed by ``time'' $\tau$). By contrast, {\em unbold letters} are used to denote parameters in {\em algorithms} (numerical solutions of continuous-time diffusions), {\it e.g.}, $\theta_{k}^{(i)}$ in \eqref{eq:svgd_update} below (indexed by ``iteration'' $k$). For conciseness, all proofs, extra experimental results and a discussion on algorithmic complexity are presented in the SM. 
\vspace{-0.1cm}
\subsection{Stochastic gradient MCMC}
\vspace{-0.1cm}
In Bayesian sampling, one aims to generate random samples from a posterior distribution $p(\thetab| \mathcal{X})\propto p(\mathcal{X}|\thetab)p(\thetab)$, where $\thetab\in\mathbb{R}^d$ represents the model parameter with a prior distribution $p(\thetab)$, and $\mathcal{X}\triangleq \{\xb_q\}_{q=1}^N$ represents the observed data with likelihood $p(\mathcal{X}|\thetab) = \prod_q p(\xb_q|\thetab)$. Define the potential energy as:
$U(\thetab) \triangleq -\log p(\mathcal{X}|\thetab)-\log p(\thetab)
= -\sum_{q=1}^N\left(\log p(\xb_q | \thetab)+\frac{1}{N}\log p(\thetab)\right) \triangleq \sum_{q=1}^NU_q(\thetab)$. 
SG-MCMC algorithms belong to diffusion-based sampling methods, where a continuous-time diffusion process is designed such that its stationary distribution matches the target posterior distribution. The diffusion process is driven by a specific SDE. For example, in stochastic gradient Langevin dynamic (SGLD) \cite{WellingT:ICML11}, the SDE endows the following form:
\begin{align}\label{eq:itodif}
\mathrm{d}\thetab_{\tau} = -\beta^{-1}F(\thetab_{\tau})\mathrm{d}\tau + \sqrt{2\beta^{-1}}\mathrm{d}\mathcal{W}_{\tau}~,
\end{align}
where $F(\thetab) \triangleq \nabla_{\thetab}U(\thetab) = \sum_{q=1}^N\nabla_{\thetab}U_q(\thetab)\triangleq \sum_{q=1}^N F_q(\thetab)$; $\tau$ is the time index, $\beta > 0$ is the temperature parameter, and $\mathcal{W}_{\tau} \in \mathbb{R}^{d}$ is a $d$-dimensional Brownian motion. More instances of SDEs corresponding to other SG-MCMC algorithms can be defined by specifying different forms of $F$ and potentially other diffusion coefficients. We focus on SGLD and \eqref{eq:itodif} in this paper, and refer interested readers to \cite{MaCF:NIPS15} for a more detailed description of general SG-MCMC algorithms. Denote the probability density function of $\thetab_{\tau}$ in \eqref{eq:itodif} as $\nu_{\tau}$, and let $\ab\cdot\bb \triangleq \ab^{\top} \bb$ for two vectors $\ab$ and $\bb$. It is known that $\nu_t$ is characterized by the following Fokker-Planck (FP) equation \cite{Risken:FPE89}:
\begin{align}\label{eq:FPE}
\partial_{\tau} \nu_{\tau} = \nabla_{\thetab}\cdot (\beta^{-1}\nu_{\tau}F(\thetab) + \beta^{-1}\nabla_{\thetab} \nu_{\tau})~.
\end{align}
According to \cite{Chiang:1987:DGO:37121.37136}, the stationary distribution $\nu_{\infty}$ equals to our target distribution $p(\thetab| \mathcal{X})$. As a result, SGLD is designed to generates samples from $p(\thetab|\mathcal{X})$ by numerically solving the SDE \eqref{eq:itodif}. For scalability, it replaces $F(\theta_k)$ in each iteration with an unbiased evaluation by randomly sampling a subset of $\mathcal{X}$, {\it i.e.}, $F(\theta_k)$ is approximated by:
$G_k \triangleq \frac{N}{B_k}\sum_{q\in\mathcal{I}_k}F_q(\theta_k)$, 
where $\mathcal{I}_k$ is a random subset of $[1, 2, \cdots, N]$ with size $B_k$ in each iteration. As a result, SGLD uses the Euler method with stepsize $h_k$ to numerically solve \eqref{eq:itodif}, resulting in the update equation:
$\theta_{k+1} = \theta_{k} - \beta^{-1}G_kh_k + \sqrt{2\beta^{-1}h_k}\xi_{k}$, 
with $\xi_{k}\sim\mathcal{N}(\mathbf{0}, \Ib)$. 

\subsection{Stein variational gradient descent}\label{sec:svgd}

Different from SG-MCMC, SVGD is a deterministic particle-optimization algorithm that is able to generate samples from a target distribution. In the algorithm, a set of particles interact with each other, driving them to high density regions in the parameter space while keeping them far away from each other with an induced {\em repulsive} force. The update equations of the particles follow the fastest descent direction of the KL-divergence between current particle distribution and the target distribution, on a RKHS (reproducing kernel Hilbert space) induced by a kernel function $\kappa(\cdot, \cdot)$ \cite{LiuW:NIPS16}. Formally, \cite{LiuW:NIPS16} derived the following update rules for the particles $\{\theta_{k}^{(i)}\}_{i=1}^M$ at the $k$-th iteration with stepsize $h_k$ and $G_k^{(i)} \triangleq \frac{N}{B_k}\sum_{q\in\mathcal{I}_k}F_q(\theta_k^{(i)})$: for $\forall i$, $\theta_{k+1}^{(i)}$
{\small\begin{align}  \label{eq:svgd_update}
	 = \theta_{k}^{(i)} +\frac{h_k}{M}\sum_{j=1}^M \left[ \kappa(\theta_{\k}^{(j)}, \theta_{k}^{(i)}) G_k^{(i)}+ \nabla_{\theta_{k}^{(j)}} \kappa(\theta_{k}^{(j)}, \theta_{k}^{(i)}) \right]
	\end{align}}
\noindent where the first term in the bracket encourages moving particles to the density modes, and the second term serves as repulsive force that pushes away different particles. In particular, the particle evolution \eqref{eq:svgd_update} are numerical solutions of the ODEs: $\mathrm{d}\thetab_{\tau}^{(i)} =\frac{1}{M}\sum_{j=1}^M \left[ \kappa(\thetab_{\tau}^{(j)}, \thetab_{\tau}^{(i)})F(\thetab_{\tau}^{(i)}) + \nabla_{\thetab_{\tau}^{(j)}} \kappa(\thetab_{\tau}^{(j)}, \thetab_{\tau}^{(i)}) \right]\mathrm{d}\tau$. 
Different from SG-MCMC, only particles at the {\em current} iteration, $\{\theta_{k}^{(i)}\}_{i=1}^M$, are used to approximate the target distribution.

\subsection{Particle-optimization based sampling}\label{POS_sub}
SG-MCMC and SVGD, though they may look closely related, behave very differently as algorithms, {\it e.g.}, stochastic and noninteractive versus deterministic and interactive particle updates. Recently, \cite{ChenZWLC:tech18} proposed a deterministic particle-optimization framework that unified SG-MCMC and SVGD. Specifically, the authors viewed both SG-MCMC and SVGD as solutions of Wasserstein gradient flows (WGFs) on the space of probabilistic measures, and derived several deterministic particle-optimization techniques for particle evolution, like SVGD. For SG-MCMC, the FP equation \eqref{eq:FPE} for SGLD is a special type of WGF. Together with an interpretation of SVGD as a special case of the Vlasov equation in the nonlinear PDE literature, \cite{ChenZWLC:tech18} proposed a general form of PDE to characterize the evolution of the density for the model parameter $\thetab$, denoted as $\nu_{\tau}$ at time $\tau$ with $\nu_{\infty}$ matching our target (posterior) distribution, {\it i.e.}, 
{\small\begin{align}\label{eq:unified}
\partial_{\tau} \nu_{\tau} = \nabla_{\thetab}\cdot \left(\nu_{\tau}\beta^{-1}F(\thetab) +\nu_{\tau}\left(\mathcal{K}*\nu_{\tau}(\thetab)\right)
+ \beta^{-1}\nabla_{\thetab}\nu_{\tau}\right)~,
\end{align}}
\hspace{-0.18cm}where $\mathcal{K}$ is a function controlling the interaction of particles in the PDE system. For example, in SVGD, \cite{ChenZWLC:tech18} showed that $\mathcal{K}$ and $\mathcal{K}*\nu_{\tau}(\thetab)$ endow the following forms:
\begin{align}\label{eq:W}
\mathcal{K}*\nu_{\tau}(\thetab)\triangleq \int \mathcal{K}(\thetab, \thetab^\prime) \nu_{\tau}(\thetab^\prime)d\thetab^\prime~,
\end{align}
where $\mathcal{K}(\thetab, \thetab^\prime) \triangleq F(\thetab^\prime)\kappa(\thetab^\prime, \thetab) - \nabla_{\thetab^\prime} \kappa(\thetab^\prime, \thetab)$ and $\kappa(\cdot, \cdot)$ is a kernel function such as the RBF kernel. In the following, we introduce a new unary function $K(\thetab)=\exp(-\frac{\| \thetab \|^2}{\eta^2})$, thus $\kappa(\thetab, \thetab^\prime)$ can be rewritten as $\kappa(\thetab,\thetab^{\prime})=K(\thetab-\thetab^{\prime})$. Hence, \eqref{eq:unified} with $\mathcal{K}$ defined in \eqref{eq:W} is equivalently written as:
\begin{align}\label{eq:unifiednew}
\partial_{\tau} \nu_{\tau} =& \nabla_{\thetab}\cdot (\nu_{\tau}\beta^{-1}F(\thetab) + \nu_{\tau}\left(E_{Y \sim\nu_{\tau}}K(\thetab-Y)F(Y)\right.\nonumber\\
&\left.-\nabla K*\nu_{\tau}(\thetab)\right)
+ \beta^{-1}\nabla_{\thetab}\nu_{\tau})~,
\end{align}
where $Y$ is a random sample from $\nu_{\tau}$ independent of $\thetab$. Note our formula here is significantly different from standard granular media equations in the literature. Section~\ref{app:related} of the SM provides more details. 
\begin{proposition}[\cite{ChenZWLC:tech18}]\label{prop:prop1}
	The stationary distribution of \eqref{eq:unifiednew} equals to our target distribution, which means $\nu_{\infty}(\thetab) = p(\thetab|\mathcal{X})$.
\end{proposition}
\cite{ChenZWLC:tech18} proposed to solve \eqref{eq:unified} numerically with deterministic particle-optimization algorithms, such as what is called the blob method. Specifically, the continuous density $\nu_{\tau}$ is approximated by a set of $M$ particles $\{\thetab_{\tau}^{(i)}\}_{i=1}^M$ that evolve over time $\tau$, {\it i.e.} $\nu_{\tau} \approx \frac{1}{M}\sum_{i=1}^M \delta_{\thetab_{\tau}^{(i)}}(\thetab)$, where $\delta_{\thetab_{\tau}^{(i)}}(\thetab)= 1$ if $\thetab = \thetab_{\tau}^{(i)}$ and 0 otherwise. 
Note $\nabla_{\thetab}\nu_{\tau}$ in \eqref{eq:unified} is no longer a valid definition when adopting particle approximation for $\nu_{\tau}$. Consequently, $\nabla_{\thetab}\nu_{\tau}$ needs nontrivial approximations, {\it e.g.}, by discrete gradient flows or blob methods proposed in \cite{ChenZWLC:tech18}. We omit the details here for simplicity.

\section{Stochastic Particle-Optimization Sampling  (SPOS)}\label{sec:spos}
\vspace{-0.1cm}
We first introduce a pitfall of SVGD, which is overcame by SPOS. In the analysis for both SVGD and SPOS, we impose the following basic assumptions.
\vspace{-0.2cm}
\begin{assumption}\label{assnew:ass1}
	Assume $F$ and $K$ satisfy the following assumptions:
	\vspace{-0.15cm}
	\begin{itemize}		
		\item[1.1] $F$ is $L_F$-Lipschitz continuous {\it i.e.}, $\|F(\thetab) - F(\thetab^\prime)\| \leq L_F\|\thetab - \thetab^\prime\|$.
		\item[1.2] $K$ is $L_K$-Lipschitz continuous; $\nabla K $ is $L_{\nabla K}$-Lipschitz continuous.
		\item[1.3] $F(\mathbf{0}) = \mathbf{0}$ and $K$ is an even function, {\it i.e.}, $K(-\thetab) = K(\thetab)$.
	\end{itemize}
\end{assumption}
A few remarks: 
$\RN{1})$ Assumptions~1.1 is widely adopted in the other theoretical works such as \cite{user-friendly,chatterji2018theory}
$\RN{2})$ $F(\mathbf{0}) = \mathbf{0}$ in Assumption~1.3 is reasonable, as $F$ in our setting corresponds to an unnormalized log-posterior, which can be shifted such that $F(\mathbf{0}) = \mathbf{0}$ for a specific problem. The assumptions of K are satisfied due to the properties of Gaussian Kernel.


\vspace{-0.1cm}
\subsection{A pitfall of SVGD}\label{sec:svgdpifall}
\vspace{-0.1cm}
We first motivate SPOS by discovering a pitfall of standard SVGD: particles in SVGD tend to collapse to a local mode under some particular conditions. Inspired by the work on analyzing granular media equations by \cite{Malrieu:AAP03,CattiauxGM:PTRF08}, we measure the collapse by calculating the expected distance between {\em exact particles} (without numerical errors), called expected particle distance (EPD) defined below.
\begin{assumption}\label{assnew:ass1new}
Assume $F$ and $K$ satisfy the following assumptions:
	\vspace{-0.3cm}
\begin{itemize}
   \item[2.1] There exists positive $m_K$ such that $\langle \nabla K(\thetab)-\nabla K(\thetab^\prime), \thetab-\thetab^\prime\rangle \leq -m_K \|\thetab - \thetab^\prime\|^2$.
   \vspace{-0.2cm}
    \item[2.2] $F$ is bounded by $H_F$ {\it i.e.}, $\|F(\thetab)\| \leq H_F$
\end{itemize}
    
\end{assumption}
\vspace{-0.1cm}
For an RBF kernel, this assumption could be satisfied by setting the bandwidth large enough and only considering the concave region for simplicity. This may seem restrictive. However, this assumption is imposed only for the analysis of the pitfall property. It is not needed in the non-asymptotic convergence analysis. Besides, we point out what will happen without this assumption in Remark \ref{remark1}.
\vspace{-0.1cm}
\begin{theorem}\label{rem:rem2_}
	Under Assumptions~\ref{assnew:ass1} and \ref{assnew:ass1new}, for the particles ${\thetab}_{\tau}^{(i)}$ defined in Section \ref{sec:svgd}, the EPD for SVGD is bounded as: 
	$\mbox{EPD} \triangleq \sqrt{\sum_{i,j}^{M} \mathbb{E}\|{\thetab}_{\tau}^{(i)}-{\thetab}_{\tau}^{(j)}\|^{2} }\leq C_0e^{-2\lambda \tau}$, 
	where $C_0 = \sqrt{\sum_{i,j}^{M}\|{\thetab}_{0}^{(i)}-{\thetab}_{0}^{(j)}\|^{2} }$, $\lambda={m_K}-H_FL_K$.
\end{theorem}
\begin{remark}\label{remark1}
	1) In the case of $\lambda \geq 0$, Theorem~\ref{rem:rem2_} indicates that particles in SVGD would collapse to a point when $\tau\rightarrow \infty$. In practice, we usually find that particles are trapped in a local mode instead of collapsing in practice. This might be due to two reasons: $\RN{1})$ Particles in SVGD are numerical solutions instead of exact solutions as used in EPD, which induces extra numerical errors; $\RN{2})$ Some particles might be out of the concave region of $K$ stated in Assumption~\ref{assnew:ass1new} in SVGD, which is required for the theory to hold. These make the empirical EPD behave not exactly the same as the true particle distance. 2) Theorem~\ref{rem:rem2_} and its proof in the SM also apply to the case of non-convex energy functions. 3) Even if the kernel is not concave, the result would still indicate that particles in the {\em concave regions} would collapse. 4) The pitfall indicates a challenge for developing non-asymptotic SVGD theory (if possible at all), motivating the development of SPOS. 
	5) This is a complement to the result of \cite{LiuZCZZC:icml19}, which proves SVGD is ill-pose under some conditions.
\end{remark}

\subsection{Stochastic particle-optimization sampling to mitigate the pitfall}\label{sec:spos_alg}
\vspace{-0.1cm}
We argue the WGF framework proposed in \cite{ChenZWLC:tech18}, if solved appropriately, is able to overcome the pitfall of SVGD. Specifically, the original solution in \cite{ChenZWLC:tech18} is based on a deterministic particle-approximation method for \eqref{eq:unified}, which introduces hard-to-control approximation errors. Instead, we propose to solve \eqref{eq:unified} {\em stochastically} to replace the $\nabla_{\thetab}\nu_{\tau}$ term in \eqref{eq:unified} with a Brownian motion.
Specifically, first note that the term $\beta^{-1}\nabla_{\thetab}\cdot\nabla_{\thetab}\nu_{\tau}$ is contributed from Brownian motion, {\it i.e.}, solving the SDE, $\mathrm{d}\thetab_{\tau} = \sqrt{2\beta^{-1}}\mathrm{d}\mathcal{W}_{\tau}$, is equivalent to solving the corresponding FP equation: $\partial \nu_{\tau} = \beta^{-1}\nabla_{\thetab}\cdot\nabla_{\thetab}\nu_{\tau}$. Consequently, we decompose the RHS of \eqref{eq:unified} into two parts: $F_1 \triangleq \nabla_{\thetab}\cdot \left(\nu_{\tau}\beta^{-1}F(\thetab_{\tau}) + (\mathcal{K}*\nu_{\tau})\nu_{\tau}\right)$ and $F_2 \triangleq \beta^{-1}\nabla_{\thetab}\cdot \nabla_{\thetab}\nu_{\tau}$. Our idea is to solve $F_1$ deterministically under a PDE setting, and solve $F_2$ stochastically based on its corresponding SDE. When adopting particle approximation for the density $\nu_{\tau}$, both solutions of $F_1$ and $F_2$ are represented in terms of particles $\{\thetab_{\tau}^{(i)}\}$. Thus we can combine the solutions from the two parts directly to approximate the original exact solution of \eqref{eq:unified}. Similar to the results of SVGD in Section~3.3 in \cite{liu2017stein_flow}, we first formally show in Theorem~\ref{theo:particle_app} that when approximating $\nu_{\tau}$ with particles, {\it i.e.}, $\nu_{\tau} \approx \frac{1}{M}\sum_{i=1}^M \delta_{\thetab_{\tau}^{(i)}}(\thetab)$, the PDE can be transformed into a system of deterministic differential equations  with interacting particles.
\begin{theorem}\label{theo:particle_app}
	When approximating $\nu_{\tau}$ in \eqref{eq:unified} with particles $\{\thetab_{\tau}^{(i)}\}$, the PDE $\partial_{\tau}\nu_{\tau} = F_1$ reduces to the following system of differential equations describing evolutions of the particles over time: $\forall i$
	\vspace{-0.2cm}
	{\small \begin{align}\label{eq:particle1}
		\mathrm{d}\thetab_{\tau}^{(i)} = -\beta^{-1}F(\thetab_{\tau}^{(i)})\mathrm{d}\tau &- \frac{1}{M}\sum_{j=1}^{M}K(\thetab_{\tau}^{(i)} - \thetab_{\tau}^{(j)})F(\thetab_{\tau}^{(j)})\mathrm{d}\tau \nonumber \\
		&+\frac{1}{M}\sum_{j=1}^{M}\nabla K(\thetab_{\tau}^{(i)} - \thetab_{\tau}^{(j)})\mathrm{d}\tau
		\end{align} }
\end{theorem}
\vspace{-0.5cm}
Consequently, by solving $\partial_{\tau}\nu_{\tau} = F_2$ stochastically from an SDE perspective, we arrive at the following differential equation system, describing evolution of the particles $\{\thetab_{\tau}^{(i)}\}$ over time $\tau$: $\forall i$
\vspace{-0.2cm}
{\small
	\begin{align}\label{eq:particle}
	\hspace{-0.3cm}\mathrm{d}\thetab_{\tau}^{(i)} = &-\beta^{-1}F(\thetab_{\tau}^{(i)} - \frac{1}{M}\sum_{j=1}^{M}K(\thetab_{\tau}^{(i)} - \thetab_{\tau}^{(j)})F(\thetab_{\tau}^{(j)}) \nonumber \\
	&+\frac{1}{M}\sum_{j=1}^{M}\nabla K(\thetab_{\tau}^{(i)} - \thetab_{\tau}^{(j)}))\mathrm{d}\tau + \sqrt{2\beta^{-1}}\mathrm{d}\mathcal{W}_{\tau}^{(i)}
	\vspace{-0.3cm}
	\end{align}}
\vspace{-0.3cm}

\begin{algorithm}[H]\label{alg1}
	\caption{Stochastic Particle-Optimization Sampling}
	{\bf Input:} Initial particles $\{\theta_0^{(i)}\}_{i=1}^{M}$ with $\theta_0^{(i)}\in{\mathbb{R}^d}$, step size $h_k$, batch size $B_k$ 
	\begin{algorithmic}[1]
		\FOR{iteration $k$= 0,1,...,T}
		\STATE Update ${\theta}_{k+1}^{(i)}$ with \eqref{eq:particle_num} for $\forall i$.
		\ENDFOR
	\end{algorithmic}
	{\bf Output:}{$\{\theta_T^{(i)}\}_{i=1}^M$}
\end{algorithm}
\vspace{-0.3cm}

Our intuition is that if the particle evolution \eqref{eq:particle} can be solved exactly, the solution of \eqref{eq:unifiednew} $\nu_{\tau}$ will be well-approximated by the particles $\{\thetab_{\tau}^{(i)}\}_{i=1}^M$. In our theory, we show this intuition is true. In practice, however, solving \eqref{eq:particle} is typically infeasible, and thus numerical methods are adopted. Furthermore, in the case of large data sets, following SG-MCMC, $F(\theta_k^{(i)})$ is typically replaced by a stochastic version $G_k^{(i)} \triangleq \frac{N}{B_k}\sum_{q\in\mathcal{I}_k}F_q(\theta_k^{(i)})$ evaluated with a minibatch of data of size $B_k$ for computational feasibility.
Based on the Euler method \cite{ChenDC:NIPS15} with a stepsize $h_k$, \eqref{eq:particle} leads to the following updates for the particles at the $k$-th iteration: let $\xi_{k}^{(i)}\sim\mathcal{N}(\mathbf{0}, \Ib)$ for $\forall i$, 
\vspace{-0.3cm}
{\begin{align}\label{eq:particle_num}
{\theta}_{k+1}^{(i)} = &{\theta}_{k}^{(i)} -h_k\beta^{-1}G_k^{(i)}  - \frac{h_k}{M}\sum_{j=1}^{M}K(\theta_{k}^{(i)} - \theta_{k}^{(j)})G_k^{(j)}\nonumber \\
& +\frac{h_k}{M}\sum_{j=1}^{M}\nabla K({\theta}_{k }^{(i)} - {\theta}_{k }^{(j)}) + \sqrt{2\beta^{-1}h_k}\xi_{k}^{(i)}
\end{align}}


\begin{figure}[t!]
	\centering
	\includegraphics[width=0.49\columnwidth]{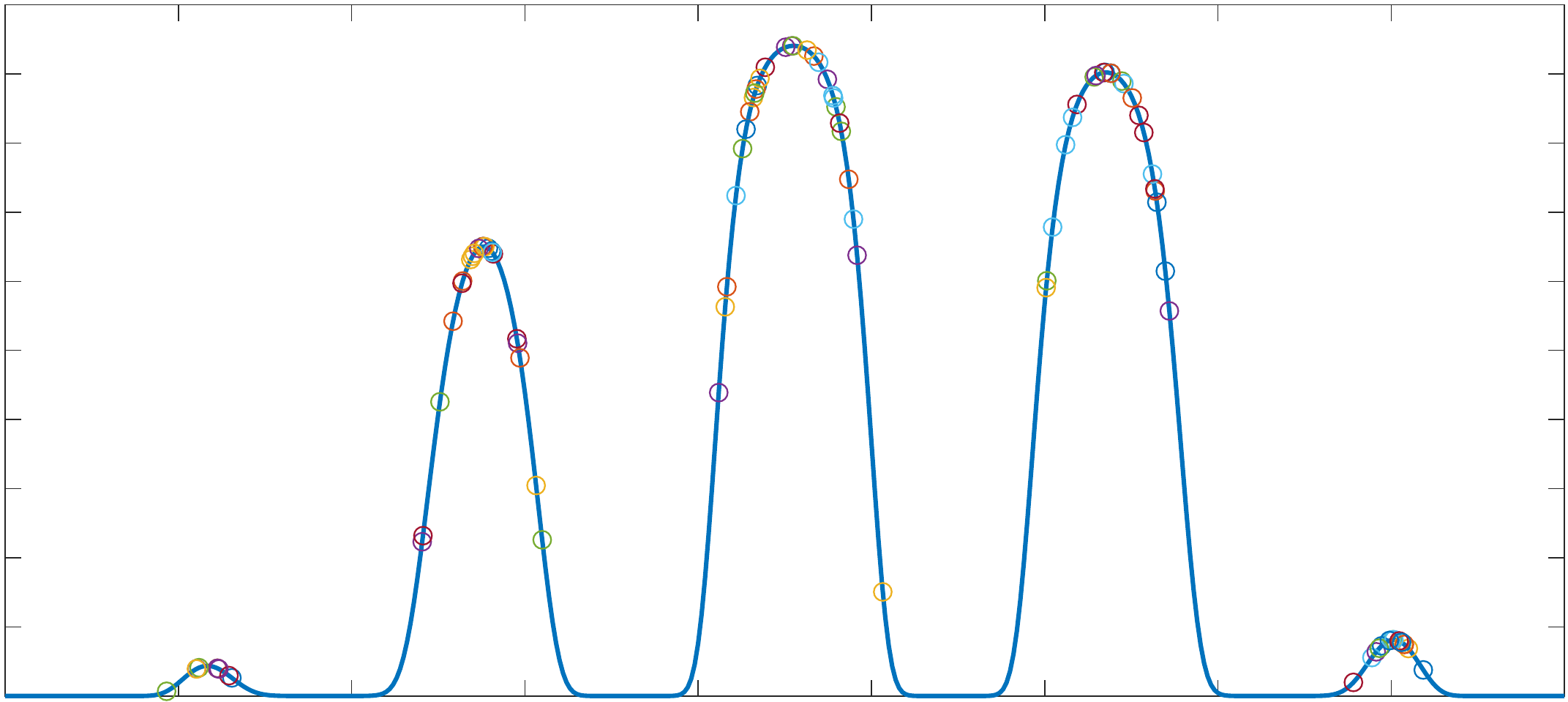}
	\includegraphics[width=0.49\columnwidth]{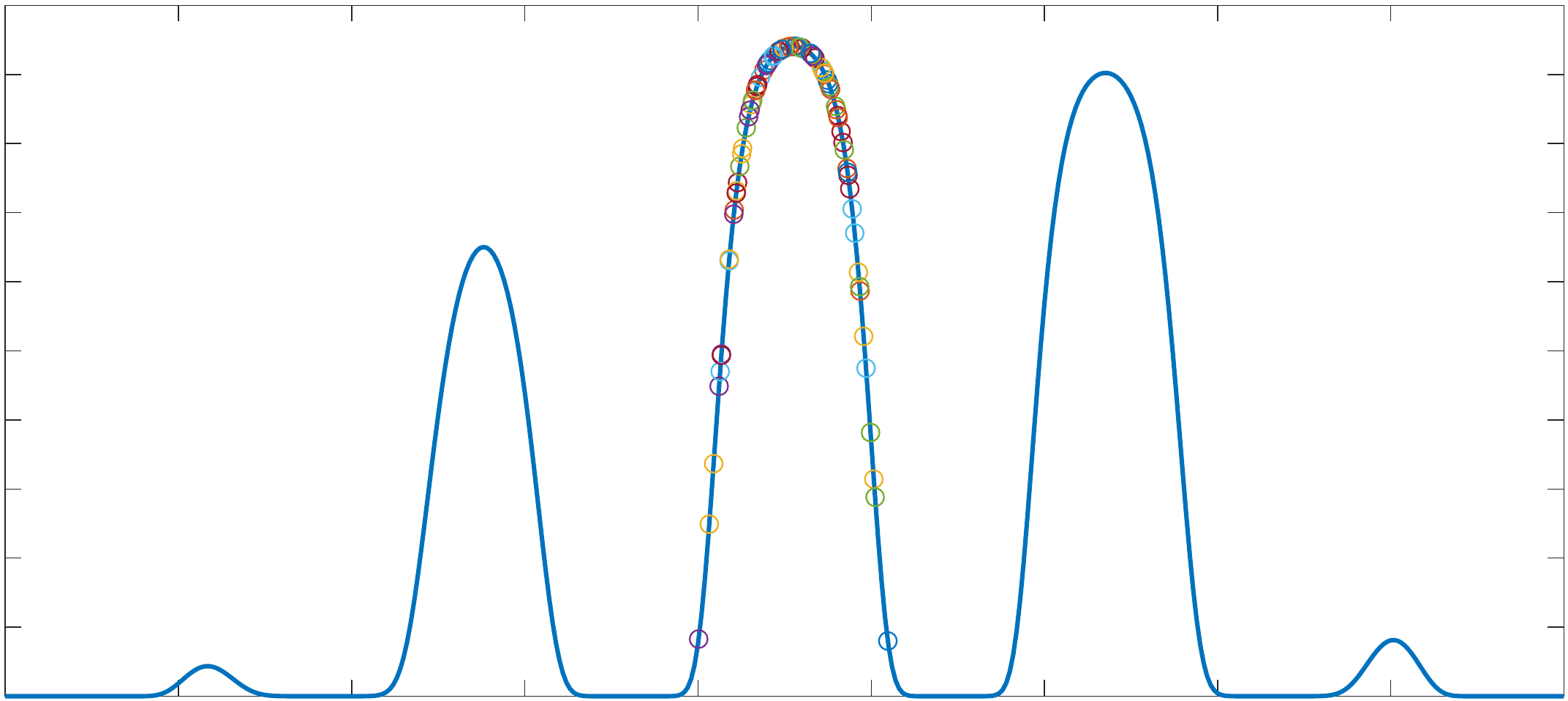}
	\vskip -0.1in
	\caption{Comparison of SPOS (left) and SVGD (right) on a multi-mode distribution. The circles with different colors are the final 100 particles, which are able to spread over all modes for SPOS.}
	\label{fig:multimode}
\end{figure}

We call the algorithm with particle update equations \eqref{eq:particle_num} stochastic particle-optimization sampling (Algorithm~1), in the sense that particles are optimized stochastically with extra random Gaussian noise. Intuitively, the added noise enhances the ability of the algorithm to jump out of local modes, leading to better exploration properties compared to standard SVGD. This serves as one of our motivations to generalize SVGD to SPOS. To illustrate the advantage of introducing the noise term, we compare SPOS and SVGD on sampling a difficult multi-mode distribution, with the density function given in Section~\ref{app:multimode} of the SM. The particles are initialized on a local mode close to zero. Note that in this example there is always positive probability to jump between modes. Figure~\ref{fig:multimode} plots the final locations of the particles along with the true density, which shows that particles in SPOS are able to reach different modes, while they are all trapped at one mode in SVGD. Theorem~\ref{theonew:chaospnew} below bounds the EPD of SPOS, in contrast with that for SVGD in Theorem~\ref{rem:rem2_}, which is intuitively obtained by taking the $\beta \rightarrow \infty$ limit.

\begin{theorem}\label{theonew:chaospnew}
	Under Assumption~\ref{assnew:ass1}, further assume that every $\{\thetab_{\tau}^{(i)}\}$ of \eqref{eq:particle} for approximating $\nu_{\tau}$ in \eqref{eq:unified} has the same initial probability law $\nu_0$ and $\Gamma \triangleq \mathbb{E}_{\thetab\sim\nu_{0}, \thetab^\prime \sim \nu_0}[\|\thetab-\thetab^\prime\|^{2}]<\infty$. Choose a $\beta$ such that $\lambda=\frac{m_F}{\beta}+m_K-H_FL_K > 0$. Then the EPD of SPOS is bounded as:
	$\mbox{EPD} \triangleq \sqrt{\sum_{i,j}^{M} \mathbb{E}\|\thetab_{\tau}^{(i)}-\thetab_{\tau}^{(j)}\|^{2} }\leq C_1e^{-2\lambda \tau}+4\sqrt{\frac{d}{\beta}}\frac{M}{\lambda}$, 
	where $C_1=M(M-1)\Gamma-4\sqrt{d\beta^{-1}}\frac{M}{\lambda}$.
\end{theorem}
\begin{remark}
	There are two interesting cases: $\RN{1})$ When $C_1 > 0$, the EPD would decrease to the bound $4\sqrt{d\beta^{-1}}M/\lambda$ with time $t$. This represents the phenomenon of an attraction force between particles; $\RN{2})$ When $C_1 < 0$, the EPD would increase to the same bound, which represents the phenomenon of a repulsive force between particles, {\it e.g.}, when particles are initialized with the same value ($\Gamma = 0$), they would be pushed away from each other until the EPD increases to the aforementioned bound. 
\end{remark}

\section{Non-Asymptotic Convergence Analysis}\label{sec:nonasym_cnovex}
In this section, we prove non-asymptotic convergence rates for the proposed SPOS algorithm under the 1-Wasserstein metric ${W}_1$, a special case of p-Wasserstein metric defined as
{\small$${W}_p(\mu,\nu)=\left(\inf_{\zeta \in \Gamma(\mu,\nu)}\int _{\mathbb{R}^d\times\mathbb{R}^d} \|X_{\mu}-X_{\nu}\|^p d\zeta(X_{\mu},X_{\nu})\right)^{1/p}
	$$} where $\Gamma(\mu,\nu)$ is the set of joint distributions on $\mathbb{R}^d\times\mathbb{R}^d$ with marginal distribution $\mu$ and $\nu$. Note that SPOS reduces to SVGD when $\beta \rightarrow \infty$, thus our theory also sheds light on the convergence behavior of SVGD, where non-asymptotic theory is currently mostly missing, despite the asymptotic theory developed recently \cite{liu2017stein_flow,LuLN:arxiv18}. 
 For analysis, we further impose the following assumptions.
\vspace{-0.2cm}
\begin{assumption}\label{assnew:ass3new}
	Assume $F$ and $\nu_0$ satisfy the following assumptions:
	\vspace{-0.15cm}
	\begin{itemize}		
		\item[3.1] There exists positive $m_F$ such that $\langle F(\thetab)-F(\thetab^\prime), \thetab-\thetab^\prime \rangle \geq m_F \|\thetab-\thetab^\prime\|^2$.
		\item[3.2] The initial probability law of each particle has a bounded and strictly positive density $\nu_0$ with respect to the Lebesgue measure on $\mathbb{R}^d$, and $\gamma_0 \triangleq \log \int_{\mathbb{R}^d}e^{\|\thetab\|^2}\nu_0(\thetab)d\thetab < \infty$
	\end{itemize}
\end{assumption}
A few remarks: $\RN{1})$ Assumption~3.1 indicates $U$ to be a convex function. Theory of non-convex $U$ is presented in Section~\ref{sec:generalU} of the SM with some extra assumptions. 
$\RN{2})$ Assumptions~3.1 is widely adopted in other theoretical works such as \cite{user-friendly,chatterji2018theory}
$\RN{3})$ Assumptions~3.2 has also been adopted in \cite{RaginskyRT:COLT17}

\vspace{-0.2cm}
\subsection{Basic setup and extra notation}\label{sec:setup}
\vspace{-0.2cm}
Due to the exchangeability of the particle system $\{\thetab_{\tau}^{(i)}\}_{i=1}^M$ in \eqref{eq:particle}, if we initialize all the particles $\thetab_{\tau}^{(i)}$ with the same distribution $\rho_0$, they would endow the same distribution for each time $\tau$. {\em We denote the distribution of each $\thetab_{\tau}^{(i)}$ as $\rho_{\tau}$}. Similar arguments hold for the particle system $\{\theta_k^{(i)}\}_{i=1}^M$ in \eqref{eq:particle_num}, and thus we denote the distribution of each $\theta_k^{(i)}$ as $\mu_k$ ($k = 1, 2, \cdots, T$). To this end, our analysis aims at bounding ${W}_1(\mu_T,\nu_{\infty})$ since $\nu_{\infty}$ is our target distribution $p(\thetab|\mathcal{X})$ according to Proposition \ref{prop:prop1}. 

In the following, for conciseness, we use a summation of stepsizes to represent the ``time index'' of some density, {\it e.g.}, $\rho_{\sum_{k=0}^{T-1} h_k}$. The high-level idea of bounding ${W}_1(\mu_T,\nu_{\infty})$ in this section is to decompose it as follows:
\begin{align}\label{eq:decomp}
{W}_1(\mu&_T,\nu_{\infty}) \leq {W}_1\left(\mu_T,\rho_{\sum\limits_{k=0}^{T-1} h_k}\right) \\
&+{W}_1\left(\rho_{\sum\limits_{k=0}^{T-1} h_k},\nu_{\sum\limits_{k=0}^{T-1} h_k}\right)+{W}_1\left(\nu_{\sum\limits_{k=0}^{T-1} h_k},\nu_{\infty}\right)~\nonumber.
\end{align}

\subsection{Bounds with stochastic particle approximation}\label{sec:spa_bound}
In this section, we bound ${W}_1(\rho_{\sum_{k=0}^{T-1} h_k},\nu_{\sum_{k=0}^{T-1} h_k})$ and ${W}_1(\nu_{\sum_{k=0}^{T-1} h_k},\nu_{\infty})$ in \eqref{eq:decomp}. The first term corresponds to a variant of granular media equation, but is much more challenging to bound.
\begin{theorem}\label{theonew:chaosp}
	Under Assumption~\ref{assnew:ass1}$\&$\ref{assnew:ass3new} and letting $\rho_0=\nu_0 $, there exist positive constants $c_{1}$ and $c_{2}$ independent of $(M,\tau)$ and satisfying $c_2 < \beta^{-1}$ such that
	\begin{align}\label{eq:bt21}
	{W}_1(\rho_{\tau},\nu_{\tau})\leq c_{1}(\beta^{-1}-c_{2})^{-1}M^{-1/2},~~~~\forall \tau.
	\end{align}
\end{theorem}

\begin{remark}\label{rem:rem1}
	According to Theorem~\ref{theonew:chaosp}, we can bound the ${W}_1(\rho_{\sum_{k=0}^{T-1} h_k},\nu_{\sum_{k=0}^{T-1} h_k})$ term as ${W}_1(\rho_{\sum_{k=0}^{T-1} h_k},\nu_{\sum_{k=0}^{T-1} h_k})\leq \frac{c_{1}}{\sqrt{M}(\beta^{-1}-c_{2})}$. Furthermore, by letting $\tau \rightarrow \infty$, we have ${W}_1(\rho_{\infty},\nu_{\infty})\leq \frac{c_{1}}{\sqrt{M}(\beta^{-1}-c_{2})}$, an important result to prove the following theorem.
\end{remark}
\begin{theorem}\label{theonew:Bt2}
	Under Assumption~\ref{assnew:ass1}$\&$\ref{assnew:ass3new}, the following holds: ${W}_1(\nu_{\tau},\nu_{\infty})\leq c_{3}e^{-2\lambda_{1}\tau}$, 
	where $\lambda_{1}=\beta^{-1}m_F-L_F-2L_K$ and $c_{3}$ is some positive constant independent of $(M,\tau)$. Furthermore, the ${W}_1(\nu_{\sum_{k=0}^T h_k},\nu_{\infty})$ term in \eqref{eq:decomp} can be bounded as:
	\begin{align}\label{eq:bt22}
	&{W}_1(\nu_{\sum_{k=0}^{T-1} h_k},\nu_{\infty})\leq c_{3}\exp\left(-2\lambda_{1}(\sum_{k=0}^{T-1} h_k)\right)~.
	\end{align}
\end{theorem}
To ensure ${W}_1(\nu_{\sum_{k=0}^{T-1} h_k},\nu_{\infty})$ decreases over time, one needs to choose $\beta$ small enough such that $\lambda_1 > 0$. This also sheds light on a failure case of SVGD (where $\beta\rightarrow\infty$) discussed in Section~\ref{sec:svgdpifall}.
\vspace{-0.2cm}
\subsection{Bounds with a numerical solution}
\vspace{-0.2cm}
\label{sec:num_bound}
To bound the ${W}_1(\mu_T,\rho_{\sum_{k=0}^{T-1} h_k})$ term in \eqref{eq:decomp}, we adopt techniques from \cite{RaginskyRT:COLT17,XuCZG:arxiv17} on analyzing the behavior of SGLD, and derive the following results for our SPOS algorithm:
\begin{theorem}\label{theonew:num_error}
	Under Assumptions~\ref{assnew:ass1}$\&$\ref{assnew:ass3new}, for a fixed step size $h_k=h$ ($\forall k$) that is small enough, the corresponding  ${W}_1(\mu_T,\rho_{Th})$ is bounded as: 
	{\small\begin{align}
		{W}_1(\mu_T,\rho_{Th})\leq c_4Md^{\frac{3}{2}}\beta^{-3}(c_5\beta^2B^{-1}+c_6h)^{\frac{1}{2}}T^{\frac{1}{2}}h^{\frac{1}{2}} \label{eq:Btprime}
		\end{align}}
	where $B$ is the minibatch size and $(c_4,c_5,c_6)$ are some positive constants independent of $(M,T,h)$.
\end{theorem}

Combining bounds from Theorems~\ref{theonew:chaosp} and \eqref{theonew:num_error}, given $T$, the optimal bound over $h$ can be seen to decrease at a rate of $O(M^{-1/2})$. Furthermore, the dependence of $T$ in the bound of Theorem~\ref{theonew:num_error} makes the bound relatively loose. Fortunately, the bound can be made independent of $T$ by considering a decreasing-stepsize SPOS algorithm, as stated in Theorem~\ref{theonew:adanum_error}.
\begin{theorem}\label{theonew:adanum_error}
	Under Assumptions~\ref{assnew:ass1}$\&$\ref{assnew:ass3new}, for a decreasing step size $h_{k}=h_0/(k+1)$,  and letting the minibatch size in each iteration $k$ be $B_k=B_0+[\log(k+1)]^{100/99}$ with $B_0$ the initial batch size, the corresponding ${W}_1(\mu_T,\rho_{\sum_{k=0}^{T-1} h_k})$ term is bounded, for some $\beta$ small enough, as:
	\begin{align}\label{eq:adaBtprime}
	{W}_1\left(\mu_T,\rho_{\sum\limits_{k=0}^{T-1} h_k}\right) \leq c_4 \beta^{-3}Md^{\frac{3}{2}}&\left(c_7h_0^3+c_8\beta^3h_0/B_0\right. \nonumber\\
	&\left.+c_9h_0^2\beta^2\right)^{1/2}~,
	\end{align}
	where $(c_4, c_7, c_8, c_9)$ are positive constants independent of $(M,T,h_0)$.
\end{theorem}
Note $B_k$ increases at a very low speed, {\it e.g.}, only by 15 after $10^5$ iterations, thus it does not affect algorithm efficiency. Consequently, ${W}_1(\mu_T,\rho_{\sum_{k=0}^{T-1} h_k}) $ would approach zero when $h_0^{1/2}M\rightarrow 0$. 

\paragraph{The Overall Non-Asymptotic Bounds}
By directly combining results from Theorems~\ref{theonew:chaosp}--\ref{theonew:adanum_error}, one can easily bound the target ${W}_1(\mu_T,\nu_{\infty})$, stated in Theorem~\ref{theo:fixed} and Theorem~\ref{theo:decrease}.

\begin{figure*}[htb]
	\centering
	\includegraphics[width=0.4\linewidth]{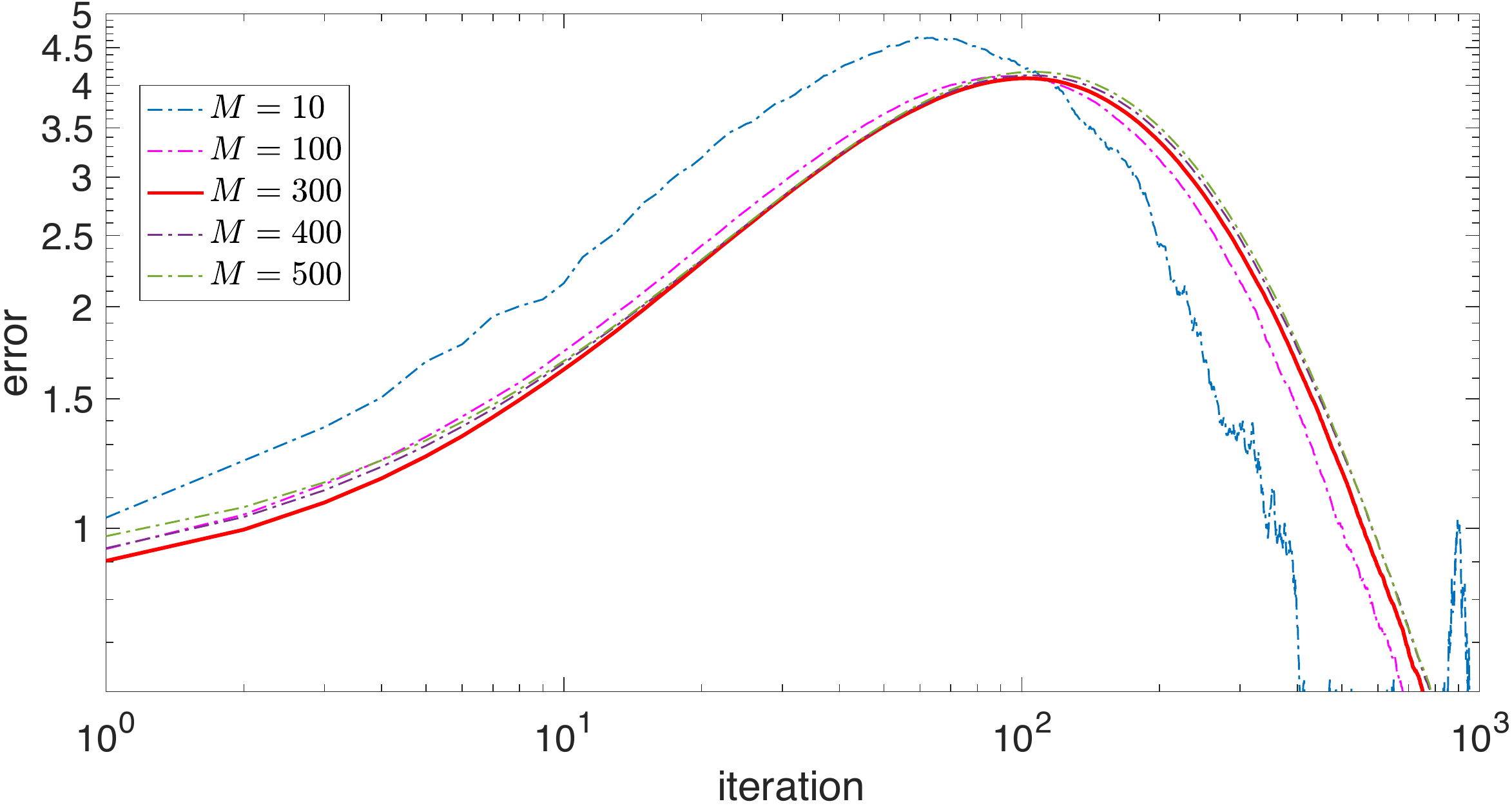}
	\hspace{0.3cm}
	\includegraphics[width=0.4\linewidth]{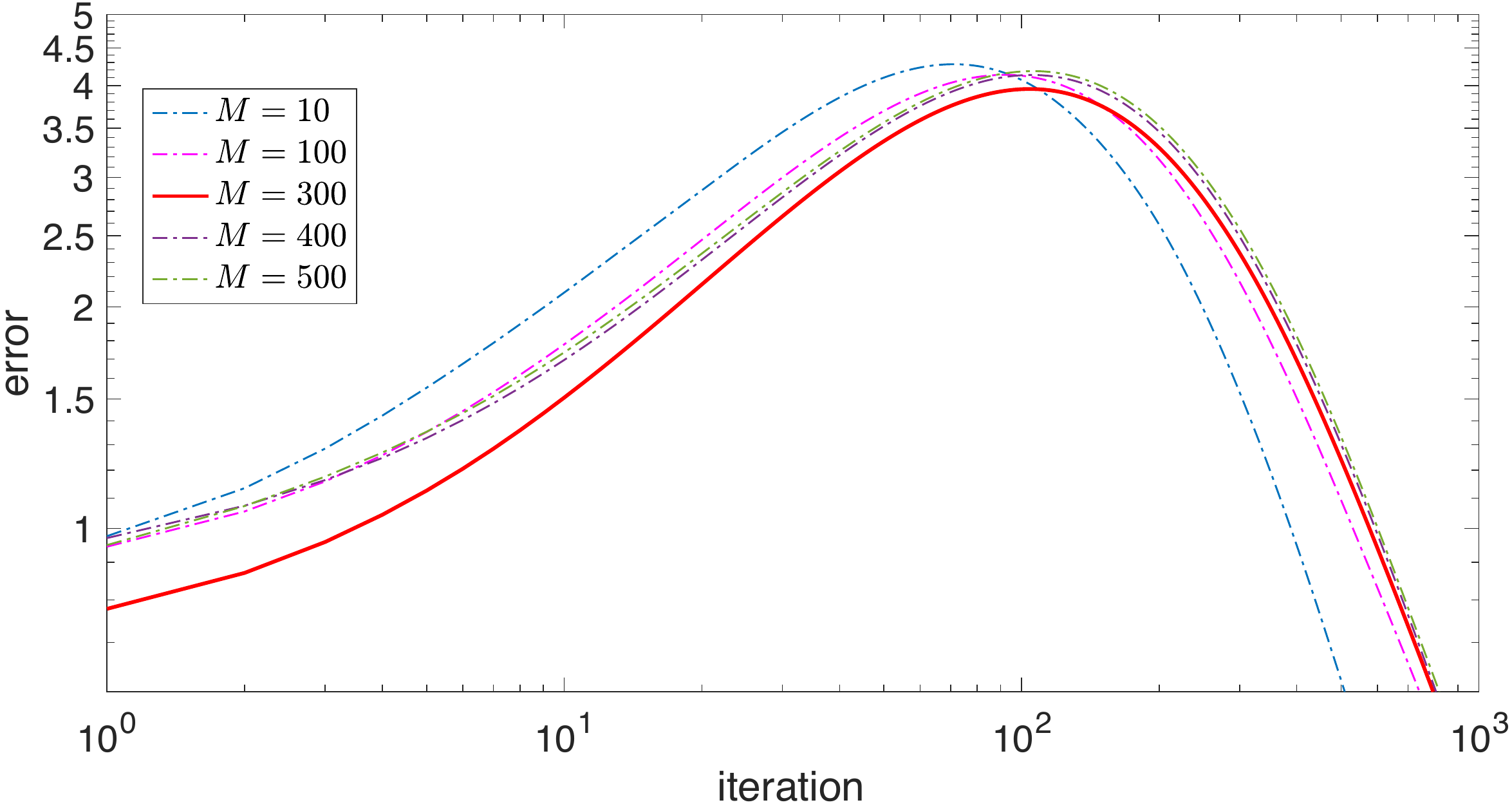}
	\vspace{-0.2cm}
	\caption{Estimation errors versus number of iterations for SPOS (left) and SVGD (right).}\label{fig:gau}
	\vspace{-0.5cm}
\end{figure*}

\begin{theorem}[Fixed Stepsize]\label{theo:fixed}
	Under Assumption~\ref{assnew:ass1}$\&$\ref{assnew:ass3new} and setting $h_k=h_0$, $B_k=B_0$, ${W}_1(\mu_T, \nu_{\infty})$ is bounded as: ${W}_1(\mu_T, \nu_{\infty})\leq $
	\begin{align}\label{eq:fixedbound}
	&\frac{c_1}{\sqrt{M}(\beta^{-1} - c_2)} + c_6Md^{\frac{3}{2}}\beta^{-3}(c_4\beta^2B^{-1}+c_5h)^{\frac{1}{2}}T^{\frac{1}{2}}h^{\frac{1}{2}} \nonumber\\
	&+ c_3\exp\left\{-2\left(\beta^{-1}m_F-L_F-2L_K\right)Th\right\},
	\end{align}
	where $(c_1, c_2, c_3, c_4, c_5, c_6, \beta)$ are positive constants such that $\frac{1}{\beta} > c_2$ and $\frac{m_F}{\beta} >L_F+2L_K$.
\end{theorem}

\begin{theorem}[Decreasing Stepsize]\label{theo:decrease}
	Denote $\tilde{h}_T \triangleq \sum_{k=0}^{T-1}h_k$. Under Assumption ~\ref{assnew:ass1}$\&$\ref{assnew:ass3new}, if we set $h_k=h_0/(k+1)$ and $B_k=B_0+[\log(k+1)]^{100/99}$, ${W}_1(\mu_T, \nu_{\infty})$ is bounded as:
	\begin{align}\label{eq:decreasebound}
	{W}_1&(\mu_T, \nu_{\infty})\leq \frac{c_1}{\sqrt{M}(\beta^{-1} - c_2)} \\
	&+ c_3\exp\{-2\left(\beta^{-1}m_F-L_F-2L_K\right)\tilde{h}_T\} \nonumber\\
	&+ c_{10} \beta^{-3}Md^{\frac{3}{2}}(c_9h_0^3+c_7\beta^3h_0/B_0+c_8h_0^2\beta^2)^{\frac{1}{2}}~. \nonumber
	\end{align}\par \vspace{-0.3cm}
	where $(c_1, c_2, c_3, c_7, c_7, c_8,c_9,c_10, \beta)$ are positive constants such that $\frac{1}{\beta} > c_2$ and $\frac{m_F}{\beta} >L_F+2L_K$.
\end{theorem}

\begin{remark}\label{remark:4}
	Four implications are highlighted from the theorems: $\RN{1})$ $M$ and $T$ play a similar role when bounding the numerical errors (the third term in the RHS of \eqref{eq:fixedbound}). The bound increases with increasing $M$ and $T$, which seems unavoidable and is consistent with the latest result for SGLD, whose bound is proved to increase w.r.t.\! $T$ \cite{RaginskyRT:COLT17}. $\RN{2})$ The increasing bound w.r.t.\! $T$ can be compromised by using decreasing stepsizes shown in Theorem~\ref{theo:decrease}. Unfortunately, this does not seem to eliminate the effect of $M$. To accommodate this, one should either use a smaller $h$ or a larger $\beta$. We believe future work is needed to improve the bound w.r.t.\! $M$. However, this is nontrivial as recent theory shows coordinate-wise SGLD scales linearly w.r.t.\! parameter dimension \cite{ShenBG:arxiv19} (corresponding to scaling linearly w.r.t.\! $M$ in our case, consistent with our theory). $\RN{3})$ When $T\times M$ (proportional to computation cost) is not too large, the error is bounded above by $O(M^{-1/2} + M)$, indicating the existence of an optimal $M$, {\it i.e.}, one should not choose arbitrary many particles as it would induce larger numerical-error bounds. This is somewhat surprising and counter-intuitive compared with the asymptotic theory \cite{liu2017stein_flow,LuLN:arxiv18}. However, we will demonstrate this is true with experiments on synthesized data, where the phenomenon is also observed in SVGD. $\RN{4})$ When $T\times M$ is large enough, the $O(M)$ term dominates, indicating an increasing error w.r.t.\! $M$. This is verified by the experiments in Section~\ref{app:extgau} (Figure~\ref{fig:err_vs_M}), although the bound might not be strictly tight.
\end{remark}

\vspace{-0.2cm}
\section{Experiments}\label{sec:exp}
\vspace{-0.2cm}
We use simple experiments with synthetic data to demonstrate the non-asymptotic convergence behaviors of SPOS indicated by our theory. In addition to the results presented below, more experiments and real applications are provided in Section~\ref{app:exp} of the SM, considering comparisons of SPOS with SVGD and SGLD on Bayesian learning of deep neural network and Bayesian exploration in deep reinforcement learning (RL).

\subsection{Sampling a Gaussian distribution}
We apply the algorithms to sample from a simple 1-D Gaussian distribution with mean 2 and variance 1. 
Since the 1-Wasserstein distance is infeasible to calculate, we follow \cite{VollmerZT:arxiv15,ChenDC:NIPS15} and measure the convergence using $\text{err} \triangleq |\mathbb{E}_{\theta \sim \mu_T}[f(\theta)] - \mathbb{E}_{\theta \sim \mathcal{N}(2, 1)}[f(\theta)]|$ with a test function
$f(\theta) \triangleq \theta^2$.
We fix $T = 1000$ and $h = 0.03$. Particles are initialized as being drawn from $\mathcal{N}(0,1)$. Figure~\ref{fig:gau} plots the estimation errors versus the number of iterations for different particles $M$. For both SPOS and SVGD, it is observed that when $T$ is not too large ($\approx 100$), the errors increase w.r.t.\! $T$, and the optimal $M$ is around 300, consistent with our theory. When $T$ is large enough, the errors decrease w.r.t.\! $T$, and larger $M$ induces larger errors. This is also consistent with our theory, because the last term in Theorem~\ref{theo:fixed} dominates when $T$ is large, leading to increasing errors with larger $M$. The only concern seems to be the tightness of the bound, which might be due to technical difficulty as current techniques for SGLD also indicate an increasing bound w.r.t.\! $T$ \cite{RaginskyRT:COLT17}. The large optimal $M$ also suggests using a relatively large $M$ should not be a problem in real applications.

\paragraph{Impact of particle number $M$}\label{app:extgau}
In addition to the above result to demonstrate the existence of an optimal $M$, we further verify that when $T\times M$ is large enough, for a fixed $T$, we observe the errors increase with increasing $M$. We use the same setting as above. Figure~\ref{fig:err_vs_M} plots the curves of errors versus number of particles. We see that errors indeed increase w.r.t.\! particle numbers, consistent with our theory. Although the rate of the bound from our theory might not match exactly with the experimental results, we believe this is still significant as the problem has never been discovered before, which is somewhat counter-intuitive. On the other hand, the results are also reasonable, as more particles would need much more updates to fit a distribution well. The results indicate that to get a smaller error, one should increase number of iterations faster than increasing the number of particles.

\begin{figure}[t!]
	\centering
	\hspace{-0.0cm}
	\includegraphics[width=1.0\linewidth]{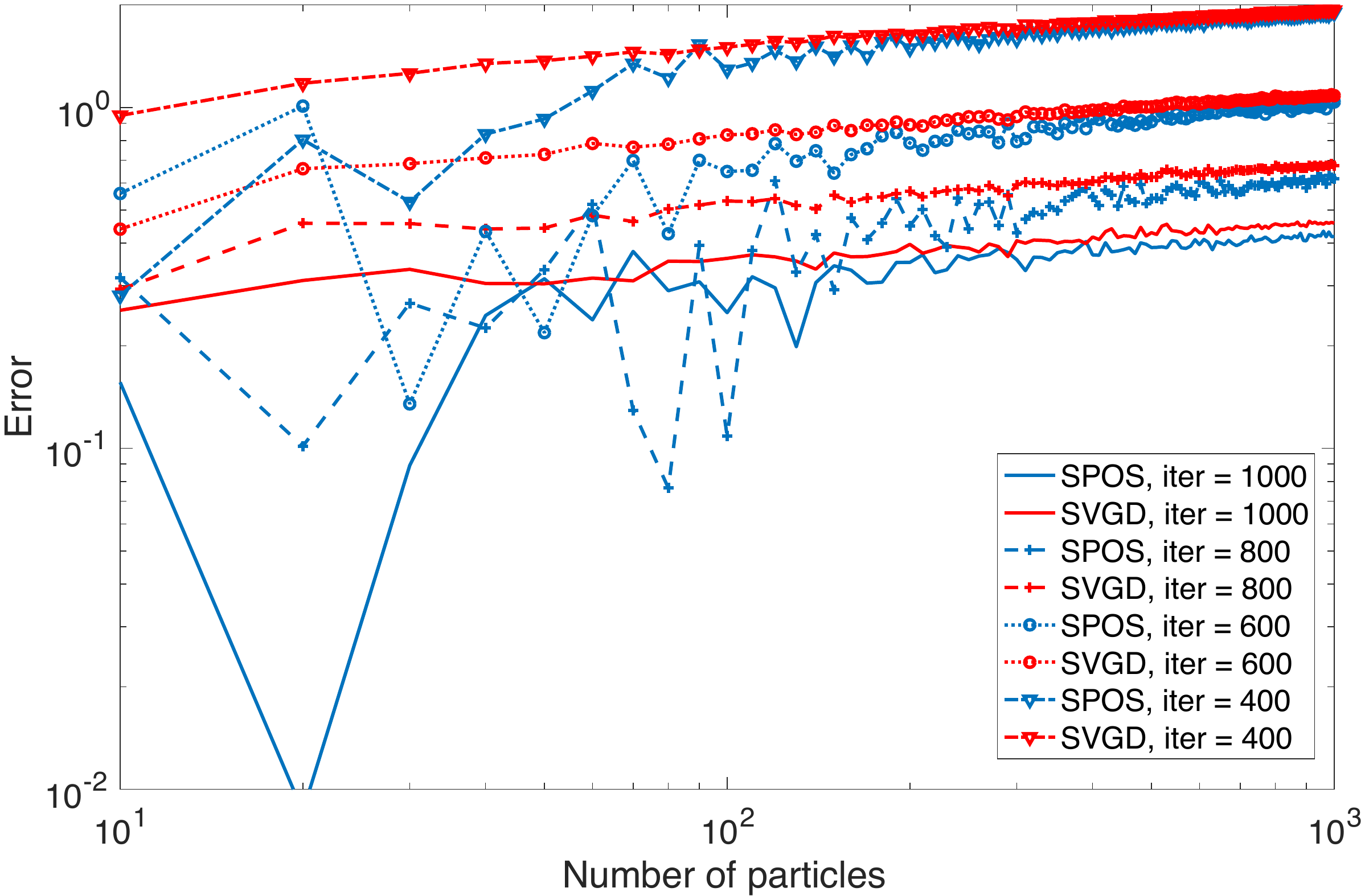}\\
	\vspace{-0.2cm}
	\caption{Errors versus Number of particles. Errors increase with increasing particle numbers.}\label{fig:err_vs_M}
	\vspace{-0.3cm}
\end{figure}

\vspace{-0.2cm}
	\begin{table}[htp]
	\vspace{-0.3cm}
	\centering
	\caption{\small Averaged RMSE with standard deviations.}
	\label{tab:reg_1}
	\vskip 0.02in
	\scalebox{0.9}{
		\begin{adjustbox}{width=\linewidth,tabular=c|ccc,center}
			\hline
			& \multicolumn{3}{c}{Test RMSE} \\
			Dataset &SGLD&SVGD&SPOS \\
			\hline
			Boston & 3.114 {\scriptsize $\pm$ 0.144} & 2.961 {\scriptsize $\pm$ 0.109} & \textbf{ 2.829} {\scriptsize $\pm$ \textbf{0.126}}  \\
			Concrete & 5.508 {\scriptsize $\pm$ 0.275} & 5.157 {\scriptsize $\pm$ 0.082} & \textbf{ 5.071} {\scriptsize $\pm$ \textbf{0.150}} \\
			Energy& 0.842 {\scriptsize $\pm$ 0.060} & 1.291 {\scriptsize $\pm$ 0.029} & \textbf{ 0.752} {\scriptsize $\pm$ \textbf{0.029}} 
			\\
			Kin8nm & 0.080 {\scriptsize $\pm$ 0.001}   & 0.090 {\scriptsize $\pm$ 0.001} & \textbf{ 0.079} {\scriptsize $\pm$ \textbf{0.001}}\\
			Naval& 0.004 {\scriptsize $\pm$ 0.000} & 0.004 {\scriptsize $\pm$ 0.000} & \textbf{ 0.004} {\scriptsize $\pm$ \textbf{0.000}} \\
			CCPP& 4.059 {\scriptsize $\pm$ 0.080} & 4.127 {\scriptsize $\pm$ 0.027}& \textbf{ 3.939} {\scriptsize $\pm$ \textbf{0.049}} \\
			Wine &0.632 {\scriptsize $\pm$ 0.022} & 0.604 {\scriptsize $\pm$ 0.007} & \textbf{0.598} {\scriptsize $\pm$ \textbf{0.014}} \\
			Yacht& 1.183 {\scriptsize $\pm$ 0.263} & 1.597 {\scriptsize $\pm$ 0.099} & \textbf{ 0.840} {\scriptsize $\pm$ \textbf{0.087}}  \\
			Protein& 4.281 {\scriptsize $\pm$ 0.011} & 4.392 {\scriptsize $\pm$ 0.015} & \textbf{ 4.254} {\scriptsize $\pm$ \textbf{0.005}} \\
			YearPredict& 8.707 {\scriptsize $\pm$ NA} & 8.684 {\scriptsize $\pm$ NA} & \textbf{ 8.681} {\scriptsize $\pm$ NA} \\
			\hline
		\end{adjustbox}
	}
	\vspace{-0.3cm}
	\end{table}

\subsection{BNNs for regression}
\vspace{-0.2cm}
We next conduct experiments for Bayesian learning of deep neural networks (DNNs) to empirically compare SGLD, SVGD and SPOS for posterior sampling of BNN weights with standard Gaussian priors. 
We use a RBF kernel with the bandwidth set to the medium of particles. Following \cite{li2015stochastic}, 10 UCI public datasets are considered, and we consider 100 hidden units for 2 large datasets (Protein and YearPredict), and 50 hidden units for the other 8 small datasets. We use the same setting as \cite{ZhangLCC:AISTATS18}. The datasets are randomly split into 90\% training and 10\% testing. For a fair comparison, we use the same split of data (train, val and test) for all methods. 
We report the root mean squared error (RMSE) in Table~\ref{tab:reg_1}. The proposed SPOS outperforms both SVGD and SGLD. More detailed settings and results are given in Section~\ref{app:exp} of the SM.

\vspace{-0.1cm}
\subsection{Bayesian exploration in deep RL}
\vspace{-0.1cm}
It is well accepted that RL performance directly measures how well the uncertainty is learned, due to the need for exploration. We apply SPOS for RL, and compare it with SVPG, a SVGD version of the policy gradient method \cite{liu2017stein}. 
Following \cite{liu2017stein,ZhangCLC:ICML}, we define policies with Bayesian DNNs. This naturally introduces uncertainty into action selection, rendering Bayesian explorations to make policy learning more effective.

We follow the same setting as in \cite{liu2017stein}, except that we use simpler policy-network architectures, as in \cite{houthooft2016vime}. We conduct experiments on three classical continuous control tasks are considered: Cartpole Swing-Up, Double Pendulum, and Cartpole. Detailed experimental settings are given in the SM. Figure~\ref{fig:bayesian_exp1} plots the cumulative rewards over time on the Cartpole environment, which clearly shows the advantage of our method over SVPG. More results are provided in the SM.

\begin{figure}[H]
    \vspace{-0.2cm}
	\centering
    \includegraphics[width=0.49\linewidth]{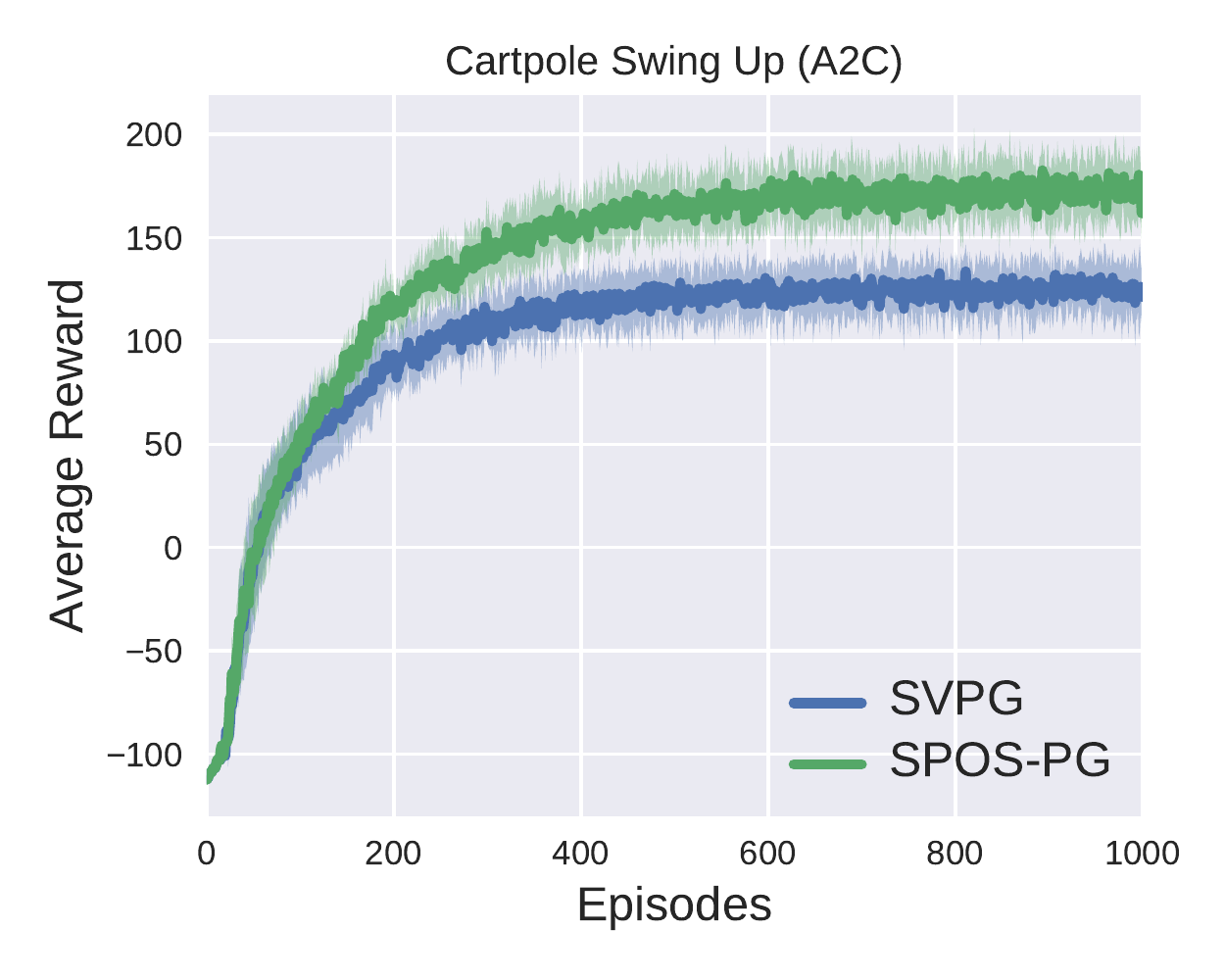}
	\includegraphics[width=0.49\linewidth]{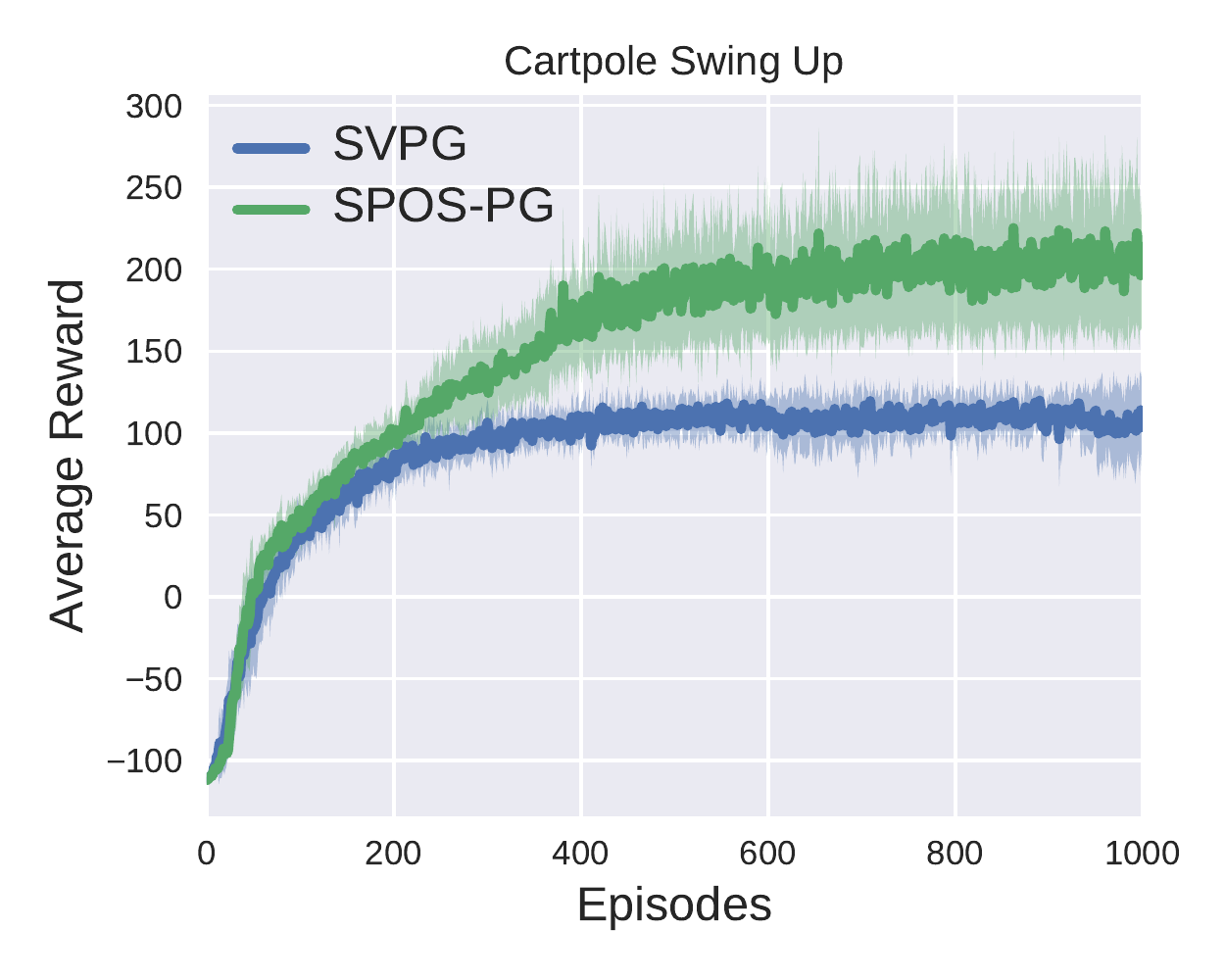}
	\vspace{-0.7cm}
	\caption{Policy learning with Bayesian exploration in policy-gradient methods with SVPG and SPOS-PG.}\label{fig:bayesian_exp1}
	\vspace{-0.2cm}
\end{figure}

\vspace{-0.3cm}
\section{Conclusion}
\vspace{-0.1cm}
We have proposed an approach for particle-optimization-based sampling that overcomes a potential pitfall of standard SVGD. Notably, for the first time, we develop non-asymptotic convergence theory for the proposed SPOS framework, a missing yet important theoretical result since the development of SVGD. Within our theoretical framework, a pitfall of SVGD, which has been studied empirically \cite{wang2017stein,ZhuoLSZCZ:arxiv17}, is formally analyzed. Our theory is practically significant as it provides nonasymptotic theoretical guarantees for the recently proposed particle-optimization-based algorithms such as the SVGD, whose advantages have also been extensively examined in real applications. Surprisingly, our theory indicates the existence of an optimal particle size, {\it i.e.}, increasing particle size does not necessarily guarantee performance improvement. This is also observed for SVGD in an experiment with synthesized data. There are a number of interesting future works. For example, one might explore more recently developed techniques such as \cite{ChengCABJ:arxiv18,LiuW:NIPS18} to improve the convergence bound; one can also adopt the SPOS framework for non-convex optimization like where SG-MCMC is used, and develop corresponding theory to study the convergence properties of the algorithm to the global optimum.

\clearpage
%

\bibliography{reference}
\bibliographystyle{apalike}

\newpage 
\onecolumn
\appendix

\section{Density Function of the Multi-Mode Distribution in Section~\ref{sec:spos}}\label{app:multimode}
The negative log-density function of the multi-mode distribution in Section~\ref{sec:spos} is defined as:
\begin{align*}
U(\thetab) \triangleq e^{\frac{3}{4}\thetab^2 - \frac{3}{2}\sum_{i=1}^{10}c_i \sin\left(\frac{1}{4}\pi i (\thetab + 4)\right)}~,
\end{align*}
where $c = (-0.47,-0.83,-0.71,-0.02,0.24,0.01,0.27,-0.37,0.87,-0.37)$ is a vector, $c_i$ is the $i$-th element of $c$.

\section{Gronwall Lemma}

The Gronwall Lemma plays an important role in parts of our proofs, which is stated in Lemma~\ref{lem:Gronwall}.

\begin{lemma}[Gronwall Lemma]\label{lem:Gronwall}
	Let $\mathcal{I}$ denotes an interval of the form $[a,+\infty)$ for some $a\in\mathbb{R}$. If $v(\tau)$, defined on $\mathcal{I}$, is differentiable in $\mathcal{I}$ and satisfies the following inequality:
	\begin{align*}
	v^{\prime}(\tau)\leq \beta(\tau)v(\tau)~,
	\end{align*}
	where $\beta(\tau)$ is a real-value continuous function defined on $\mathcal{I}$.
	Then $v(\tau)$ can be bounded as:
	\begin{align*}
	v(\tau)\leq v(a)\exp\left(\int_{a}^{\tau}\beta(s)\mathrm{d}s\right)
	\end{align*}
\end{lemma}



\section{Proof of Theorem~\ref{theo:particle_app}}

Proofs of Theorem~\ref{rem:rem2_} and \ref{theonew:chaospnew} rely on techniques in the proofs for Section~\ref{sec:nonasym_cnovex}. As a result, we defer the proofs of Theorem~\ref{rem:rem2_} and \ref{theonew:chaospnew} to the later part.

To prove Theorem~\ref{theo:particle_app}, we rely on the definition of generalized derivative in Definition~\ref{def:gd}.

\begin{definition}[Generalized Derivative]\label{def:gd}
	Let $g$ and $\phi$ be locally integrable functions on an open set $\Omega \subset \mathbb{R}^d$, that is, Lebesgue integrable on any closed bounded set $\mathcal{F} \subset \omega$. Then $\phi$ is the generalized derivative of $g$ with respect to $\thetab_j$ on $\Omega$, written as $\phi = \partial_{\thetab_j} g$, if for any infinitely-differentiable function $u$ with compact support in $\Omega$, we have
	\begin{align*}
	\int_{\Omega}g(\thetab)\partial_{\thetab_j} u(\thetab) \mathrm{d}\thetab = -\int_{\Omega}\phi(\thetab)u(\thetab)\mathrm{d}\thetab~.
	\end{align*}
\end{definition}
where $\thetab=(\thetab_1,\thetab_2,...,\thetab_{d})$ on $\Omega \subset \mathbb{R}^d$. 

\begin{proof}
	The proof relies on further expansions on the definition of generalized derivative on specific functions. Specifically, let the function $g$ in Definition~\ref{def:gd} be in a form of $g \triangleq Gf$ for the product of two functions $G$ and $f$ (specified below).
	The generalized derivative of $(Gf)$ with respect to $\thetab_j$, written as $\partial_{\thetab_j}(Gf)$, satisfies 
	\begin{align}\label{eq:generalizedD}
	\int \partial_{\thetab_j}(Gf) \ u(\thetab)\mathrm{d}\thetab = -\int Gf \ \partial_{\thetab_j} u(\thetab)\mathrm{d}\thetab
	\end{align}
	for all differentiable function $u(\cdot)$.
	
	In Theorem~\ref{theo:particle_app}, we want to prove a particle representation of the following PDE:
	\begin{align*}
	&\partial_{\tau}\nu_{\tau} = F_1 = \nabla_{\thetab}\cdot \left(\nu_{\tau}F(\thetab) + (\mathcal{K}*\nu_{\tau})\nu_{\tau}\right) \triangleq -\sum_j^d\partial_{\thetab_j}(Gf)~,
	\end{align*}
	where we set $f(\thetab) = \nu_{\tau}(\thetab)$ and $G(\thetab)\triangleq -F(\thetab) - (\mathcal{K}*\nu_{\tau})(\thetab)$. Taking integration on both sides for any continuous function $u(\thetab)$, we have
	\begin{align}\label{eq:ff1}
	&\int \partial_{\tau}\nu_{\tau} u(\thetab) \mathrm{d}\thetab = -\int \sum_j^d\partial_{\thetab_j}(Gf)\ u(\thetab) \mathrm{d}\thetab \nonumber\\
	&\Rightarrow \int \partial_{\tau}f \ u(\thetab)\mathrm{d}\thetab = -\sum_j^d \int \partial_{\thetab_j}(Gf) \ u(\thetab)\mathrm{d}\thetab
	\end{align}
	
	By applying \eqref{eq:generalizedD} in \eqref{eq:ff1}, we have
	\begin{align*}
	&\int \partial_{\tau}f \ u(\thetab)\mathrm{d}\thetab = - \sum_j^d \int \partial_{\thetab_j}(Gf) \ u(\thetab)\mathrm{d}\thetab\\
	&= \sum_j^d \int Gf  \partial_{\thetab_j}u(\thetab))\mathrm{d}\thetab~.
	\end{align*}
	Since $f = \nu_{\tau}(\thetab)$ and we can set $u(\thetab) = \thetab$, we will derive
	\begin{align}\label{eq:ff2}
	&\int \partial_{\tau}\nu_{\tau}(\thetab)u(\thetab) \mathrm{d}\thetab= \sum_j^d \int G(\thetab) \nu_{\tau}(\thetab) \ \partial_{\thetab_j}u(\thetab)\mathrm{d}\thetab \nonumber\\
	\Rightarrow& \frac{\mathrm{d}}{\mathrm{d}\tau}\int \nu_{\tau}(\thetab)\thetab  \mathrm{d}\thetab = \int G (\thetab)\ \nu_{\tau}(\thetab)\mathrm{d}\thetab  \nonumber\\
	\Rightarrow& \frac{\mathrm{d}}{\mathrm{d}\tau}\mathbb{E}_{\nu_{\tau}}[\thetab] = \mathbb{E}_{\nu_{\tau}}[G(\thetab)]~.
	\end{align}
	In particle approximation, we have $\nu_{\tau}(\thetab) \approx \frac{1}{M}\sum_{i=1}^M\delta_{(\thetab_{\tau}^{(i)})}(\thetab)$. For each particle,according to the definition of $\mathcal{K}*\nu_{\tau}$ in Sec \ref{POS_sub}, \eqref{eq:ff2} reduces to the following equation:
	\begin{align*}
	\mathrm{d}\thetab_{\tau}^{(i)} &= G(\thetab_{\tau}^{(i)}) \mathrm{d}\tau \nonumber \\
	&=-\beta^{-1}F(\thetab_{\tau}^{(i)})\mathrm{d}\tau - \frac{1}{M}\sum_{j=1}^{M}K(\thetab_{\tau}^{(i)} - \thetab_{\tau}^{(j)})F(\thetab_{\tau}^{(j)})\mathrm{d}\tau +\frac{1}{M}\sum_{j=1}^{M}\nabla K(\thetab_{\tau}^{(i)} - \thetab_{\tau}^{(j)})\mathrm{d}\tau,
	\end{align*}
	which is the update equation in Theorem~\ref{theo:particle_app}. This completes the proof.
\end{proof}

\section{Proof of Theorem~\ref{theonew:chaosp}}
Note that one challenge in our analysis compared with the analysis for diffusion-based methods, such as those for SG-MCMC \cite{VollmerZT:arxiv15,ChenDC:NIPS15}, is how to bound the gap between the original nonlinear PDE \eqref{eq:unified} and the reduced SDE \eqref{eq:particle}. Following the analysis of granular media equations such as \cite{Malrieu:AAP03,CattiauxGM:PTRF08,2018arXiv180511387D}, we introduce a intermediate SDE in-between \eqref{eq:unifiednew} and \eqref{eq:particle}, defined as:
\begin{align}\label{eq:nonlinear_sde}
\begin{cases}
\mathrm{d}\bar{\thetab}_{\tau}=&-\beta^{-1}F(\bar{\thetab}_{\tau})\mathrm{d}\tau -\mathbb{E}_{Y \sim\nu_{\tau}}K(\bar{\thetab}_{\tau}-Y)F(Y)\mathrm{d}\tau +\nabla K*\nu_{\tau}(\bar{\thetab}_{\tau})\mathrm{d}\tau
+ \sqrt{2\beta^{-1}}\mathrm{d}\bar{\mathcal{W}}_{\tau}~\\
\mathcal{L}(\bar{\thetab}_{\tau}) = &\nu_{\tau}\mathrm{d}\thetab
\end{cases}~
\end{align}
where $\mathcal{L}(\bar{\thetab}_{\tau})$ denotes the probability law of $\bar{\thetab}_{\tau}$, $\bar{\mathcal{W}_{\tau}} \in \mathbb{R}^{d}$ is a $d$-dimensional Brownian motion independent of $\bar{\thetab}_{\tau}$ and $Y$ is a random variable independent of $\bar{\thetab}_{\tau}$, which is integrated out. In order to match $\bar{\thetab}_{\tau}$ with  the particles $\{\thetab_{\tau}^{(i)}\}_{i=1}^M$ in the SDE system \eqref{eq:particle}, we duplicate \eqref{eq:nonlinear_sde} $M$ times, each endowing an exact solution $\bar{\thetab}_{\tau}^{(i)}$ indexed by $i$. The distribution of each particles $\{\bar{\thetab}_{\tau}^{(i)}\}_{i=1}^M$ is denoted as $\nu_{\tau}$. Note since \eqref{eq:nonlinear_sde} is introduced for the purpose of proof convenience without any restrictions, we construct it in a way such that all the $\bar{\mathcal{W}}_{\tau}^{(i)}$ are {\em exactly the same}, but independent of each ${\mathcal{W}}_{\tau}^{(i)}$, {\it i.e.},
\begin{align}\label{eqnew:nonlinear_sde}
\begin{cases}
\mathrm{d}\bar{\thetab}_{\tau}^{(i)} =&-\beta^{-1}F(\bar{\thetab}_{\tau}^{(i)})\mathrm{d}\tau -\mathbb{E}_{Y_i \sim\nu_{\tau}}K(\bar{\thetab}_{\tau}^{(i)}-Y_i)F(Y_i)\mathrm{d}\tau +\nabla K*\nu_{\tau}(\bar{\thetab}_{\tau}^{(i)})\mathrm{d}\tau
+ \sqrt{2\beta^{-1}}\mathrm{d}\bar{\mathcal{W}}_{\tau}^{(i)}~\\
\mathcal{L}(\bar{\thetab}_{\tau}^{(i)}) = &\nu_{\tau}\mathrm{d}\thetab
\end{cases}~
\end{align}
where, similarly, $Y_i$ is a random variable independent of $\bar{\thetab}_{\tau}^{(i)}$, introduced for the convenience of the proof. Furthermore, we set all the $\bar{\thetab}_{0}^{(i)}$ exact the same but independent of each ${\thetab}_{0}^{(i)}$. Consequently, all the $\bar{\thetab}_{\tau}^{(i)}$ are also exactly the same but independent of each other. Please note these settings do not affect our algorithm, as \eqref{eqnew:nonlinear_sde} are only introduced for the purpose of proof. Now it is ready to prove Theorem~\ref{theonew:chaosp}.

\begin{proof}[Proof of Theorem~\ref{theonew:chaosp}]
	First, from the definitions, we have
	\begin{align*}
	\mathrm{d}&\left(\thetab_{\tau}^{(i)} - \bar{\thetab}_{\tau}^{(i)}\right) = -\beta^{-1}\left(F(\thetab_{\tau}^{(i)}) - F(\bar{\thetab}_{\tau}^{(i)})\right)\mathrm{d}\tau \\
	&+\frac{1}{M}\sum_j^{M}\left[\nabla K(\thetab_{\tau}^{(i)}- \thetab_{\tau}^{(j)}) - \nabla K*\nu_{\tau}(\bar{\thetab}_{\tau}^{(i)})\right]\mathrm{d}\tau \\
	&-\frac{1}{M}\sum_j^{M}\left(F(\thetab_{\tau}^{(j)})K(\thetab_{\tau}^{(i)}-\thetab_{\tau}^{(j)}) - \mathbb{E}_{Y_{j}\sim\nu_{\tau}}F(Y_{j})K(\bar{\thetab}_{\tau}^{(i)}-Y_{j})\right)\mathrm{d}\tau  \nonumber\\
	&+(\sqrt{2\beta^{-1}}\mathrm{d}\mathcal{W}_{\tau}^{(i)}-\sqrt{2\beta^{-1}}\mathrm{d}\bar{\mathcal{W}}_{\tau}^{(i)})\mathrm{d}\tau
	\end{align*}
	Hence,
	\begin{align}
	&\Rightarrow  \mathrm{d}\left(\sum_i^{M}\left\|\thetab_{\tau}^{(i)} - \bar{\thetab}_{\tau}^{(i)}\right\|^2\right) \\
	=&\frac{2}{M}\sum_{i,j}^{M} (A_{ij}(\tau) + B_{ij}(\tau) + C_{ij}(\tau)+F_{ij}(\tau) +G_{ij}(\tau)+H_{ij}(\tau)+I_{i,j}(\tau))\mathrm{d}\tau \label{eqnew:ABCDGH}~,
	\end{align}
	where
	\begin{align}
	&A_{ij}(\tau) = -\beta^{-1}\left(F(\thetab_{\tau}^{(i)}) - F(\bar{\thetab}_{\tau}^{(i)})\right)\cdot \left(\thetab_{\tau}^{(i)} - \bar{\thetab}_{\tau}^{(i)}\right)\nonumber\\
	&B_{ij}(\tau)= \left(\nabla K(\thetab_{\tau}^{(i)} - \thetab_{\tau}^{(j)}) - \nabla K(\bar{\thetab}_{\tau}^{(i)} - \bar{\thetab}_{\tau}^{(j)})\right) \cdot\left(\thetab_{\tau}^{(i)} - \bar{\thetab}_{\tau}^{(i)}\right) \nonumber\\
	&C_{ij}(\tau)= \left(\nabla K(\bar{\thetab}_{\tau}^{(i)} - \bar{\thetab}_{\tau}^{(j)}) - \nabla K*\nu_{\tau}(\bar{\thetab}_{\tau}^{(i)})\right) \cdot \left(\thetab_{\tau}^{(i)} - \bar{\thetab}_{\tau}^{(i)}\right) \nonumber\\	
	&F_{ij}(\tau) = -\left(F({\thetab}_{\tau}^{(j)}) K(\thetab_{\tau}^{(i)}-\thetab_{\tau}^{(j)}) - F(\bar{\thetab}_{\tau}^{(j)}) K(\thetab_{\tau}^{(i)}-\thetab_{\tau}^{(j)})\right)\cdot \left(\thetab_{\tau}^{(i)} - \bar{\thetab}_{\tau}^{(i)}\right)\nonumber\\
	&G_{ij}(\tau) = -\left(F(\bar{\thetab}_{\tau}^{(j)}) K(\thetab_{\tau}^{(i)}-\thetab_{\tau}^{(j)}) - F(\bar{\thetab}_{\tau}^{(j)}) K(\bar{\thetab}_{\tau}^{(i)}-\bar{\thetab}_{\tau}^{(j)})\right)\cdot \left(\thetab_{\tau}^{(i)} - \bar{\thetab}_{\tau}^{(i)}\right)\nonumber\\
	&H_{ij}(\tau) = -\left(F(\bar{\thetab}_{\tau}^{(j)}) K(\bar{\thetab}_{\tau}^{(i)}-\bar{\thetab}_{\tau}^{(j)}) -\mathbb{E}_{Y_{j}\sim\nu_{\tau}}F(Y_{j})K(\bar{\thetab}_{\tau}^{(i)}-Y_{j}) \right)\cdot \left(\thetab_{\tau}^{(i)} - \bar{\thetab}_{\tau}^{(i)}\right) \nonumber\\
	&I_{i,j}(\tau)=(\sqrt{2\beta^{-1}}\mathrm{d}\mathcal{W}_{\tau}^{(i)}-\sqrt{2\beta^{-1}}\mathrm{d}\bar{\mathcal{W}}_{\tau}^{(i)})\cdot \left(\thetab_{\tau}^{(i)} - \bar{\thetab}_{\tau}^{(i)}\right)\nonumber
	\end{align}
	
	Next, we bound these terms in the following. For the $A_{ij}(\tau)$ term, according to bullet $\RN{1})$ in Assumption~\ref{assnew:ass3new} for $F$, we have
	\begin{align*}
	\mathbb{E}\sum_{ij}A_{ij}(\tau)&= -\mathbb{E}\sum_{ij}\beta^{-1}\left(F(\thetab_{\tau}^{(i)}) - F(\bar{\thetab}_{\tau}^{(i)})\right)\cdot \left(\thetab_{\tau}^{(i)} - \bar{\thetab}_{\tau}^{(i)}\right) \\
	&\leq -\beta^{-1}m_FM\sum_i\mathbb{E}\left\|\thetab_{\tau}^{(i)} - \bar{\thetab}_{\tau}^{(i)}\right\|^2 ~
	\end{align*}
	
	For the $B_{ij}(\tau)$ term, applying the oddness of $\nabla K$ in Assumption~\ref{assnew:ass1}, we have
	\begin{align*}
	\mathbb{E}&\sum_{ij}B_{ij}(\tau)\\
	=\mathbb{E}& \sum_{ij}\left(\nabla K(\thetab_{\tau}^{(i)} - \thetab_{\tau}^{(j)}) - \nabla K(\bar{\thetab}_{\tau}^{(i)} - \bar{\thetab}_{\tau}^{(j)})\right)\cdot \left(\thetab_{\tau}^{(i)} - \bar{\thetab}_{\tau}^{(i)}\right) \\
	=\mathbb{E}&\frac{1}{2}\sum_{ij}\left(\nabla K(\thetab_{\tau}^{(i)} - \thetab_{\tau}^{(j)}) - \nabla K(\bar{\thetab}_{\tau}^{(i)} - \bar{\thetab}_{\tau}^{(j)})\right) \left(\thetab_{\tau}^{(i)} - \bar{\thetab}_{\tau}^{(i)}-(\thetab_{\tau}^{(j)}-\bar{\thetab}_{\tau}^{(j)})\right)\\
	\leq \mathbb{E}&\frac{1}{2} L_{\nabla K}\sum_{ij}\left\|\thetab_{\tau}^{(i)} - \bar{\thetab}_{\tau}^{(i)}-(\thetab_{\tau}^{(j)} - \bar{\thetab}_{\tau}^{(j)})\right\|^2 \leq 2L_{\nabla K}M\mathbb{E}\sum_i\left\|\thetab_{\tau}^{(i)} - \bar{\thetab}_{\tau}^{(i)}\right\|^2~
	\end{align*}
	
	For the $C_{ij}(\tau)$ term, we have
	\begin{align*}
	&\mathbb{E}\sum_j C_{ij}(\tau) \stackrel{(1)}{\leq} \left(\mathbb{E}\left\|\thetab_{\tau}^{(i)} - \bar{\thetab}_{\tau}^{(i)}\right\|^2\right)^{1/2} \left(\mathbb{E}\left\|\sum_j \left(\nabla K(\bar{\thetab}_{\tau}^{(i)} - \bar{\thetab}_{\tau}^{(j)}) - \nabla K*\nu_{\tau}(\bar{\thetab}_{\tau}^{(i)})\right)\right\|^2\right)^{1/2} \\
	&\stackrel{(2)}{=} \left(\mathbb{E}\left\|\thetab_{\tau}^{(i)} - \bar{\thetab}_{\tau}^{(i)}\right\|^2\right)^{1/2} \left(\sum_{j}\mathbb{E}\left(\nabla K(\bar{\thetab}_{\tau}^{(i)} - \bar{\thetab}_{\tau}^{(j)}) - \nabla K*\nu_{\tau}(\bar{\thetab}_{\tau}^{(i)})\right)^2\right)^{1/2} \\
	&\stackrel{(3)}{\leq} 2H_{\nabla K} \sqrt{M} \left(\mathbb{E}\left\|\thetab_{\tau}^{(i)} - \bar{\thetab}_{\tau}^{(i)}\right\|^2\right)^{1/2}~
	\end{align*}
	where (1) is obtained by applying the Cauchy-Schwarz inequality, and (2) by the fact that $\mathbb{E}\left(K(\bar{\thetab}_{\tau}^{(i)} - \bar{\thetab}_{\tau}^{(j)}) - K*\nu_{\tau}(\bar{\thetab}_{\tau}^{(i)})\right) = 0$. Furthermore, due to the fact that $\nabla K=\exp(-\frac{\|\thetab\|^2}{\eta^2})\frac{2}{\eta^2}\thetab$, we can bound the $\nabla K(\thetab)$ with  $\|\nabla K\|\leq \exp(-\frac{\|\thetab\|^2}{\eta^2})\frac{2}{\eta^2}\|\thetab\|$. Hence there exists some positive constant $H_{\nabla K}$ such that $\|\nabla K(\thetab)\| \leq H_{\nabla K}$.\\
	
	Similarly, since $K \leq 1$, we have the following result for the $H_{ij}(\tau)$ term,
	\begin{align*}
	&\mathbb{E}\sum_j H_{ij}(\tau) \\
	{\leq}& \left(\mathbb{E}\left\|\thetab_{\tau}^{(i)} - \bar{\thetab}_{\tau}^{(i)}\right\|^2\right)^{1/2}\left(\mathbb{E}\left\|\sum_j \left(F(\bar{\thetab}_{\tau}^{(j)}) K(\bar{\thetab}_{\tau}^{(i)}-\bar{\thetab}_{\tau}^{(j)}) -\mathbb{E}_{Y_{j}\sim\nu_{\tau}}F(Y_{j})K(\bar{\thetab}_{\tau}^{(i)}-Y_{j}) \right)\right\|^2\right)^{1/2} \\
	{=}& \left(\mathbb{E}\left\|\thetab_{\tau}^{(i)} - \bar{\thetab}_{\tau}^{(i)}\right\|^2\right)^{1/2}\left(\sum_{j}\mathbb{E}\left(F(\bar{\thetab}_{\tau}^{(j)}) K(\bar{\thetab}_{\tau}^{(i)}-\bar{\thetab}_{\tau}^{(j)}) -\mathbb{E}_{Y_{j}\sim\nu_{\tau}}F(Y_{j})K(\bar{\thetab}_{\tau}^{(i)}-Y_{j}) \right)^2\right)^{1/2} \\
	{\leq}& 2H_{\theta} \sqrt{M} \left(\mathbb{E}\left\|\thetab_{\tau}^{(i)} - \bar{\thetab}_{\tau}^{(i)}\right\|^2\right)^{1/2}~
	\end{align*}
	The last inequality follows from the fact that $\sqrt{E[F(\bar{\thetab}_{\tau}^{(j)})^2} \leq L_F \sqrt{E\|\bar{\thetab}_{\tau}^{(j)}\|^2}\leq L_F\sqrt{\gamma_0+\frac{d}{m^\prime\beta}}$, which is derived in Theorem \ref{thm:continousbound2}. We denote $L_F\sqrt{\gamma_0+\frac{d}{m^\prime\beta}}$ as $H_{\theta}$.
	
	For the $F_{ij}(\tau)$ and $G_{ij}(\tau)$ terms, we have:
	\begin{align*}
	\mathbb{E}\sum_{i j} F_{ij}(\tau) &= -\mathbb{E}\sum_{ij}\left(F({\thetab}_{\tau}^{(j)}) K(\thetab_{\tau}^{(i)}-\thetab_{\tau}^{(j)}) - F(\bar{\thetab}_{\tau}^{(j)}) K(\thetab_{\tau}^{(i)}-\thetab_{\tau}^{(j)})\right)\cdot \left(\thetab_{\tau}^{(i)} - \bar{\thetab}_{\tau}^{(i)}\right)\nonumber \\
	&\leq \mathbb{E}\sum_{i j}L_F \left\|\thetab_{\tau}^{(j)} - \bar{\thetab}_{\tau}^{(j)}\right\| \left \|\thetab_{\tau}^{(i)} - \bar{\thetab}_{\tau}^{(i)}\right\|\\
	&\leq L_F M\mathbb{E}\sum_{i} \left \|\thetab_{\tau}^{(i)} - \bar{\thetab}_{\tau}^{(i)}\right\|^2~.
	\end{align*}
	\begin{align*}
	\mathbb{E}\sum_{ij} G_{ij}(\tau) = &\mathbb{E}\sum_{ij}\left(F(\bar{\thetab}_{\tau}^{(j)}) K(\thetab_{\tau}^{(i)}-\thetab_{\tau}^{(j)}) - F(\bar{\thetab}_{\tau}^{(j)}) K(\bar{\thetab}_{\tau}^{(i)}-\bar{\thetab}_{\tau}^{(j)})\right)\cdot \left(\thetab_{\tau}^{(i)} - \bar{\thetab}_{\tau}^{(i)}\right) \\
	&\leq L_FL_K\sum_{i,j}\mathbb{E}\left(\left\|\bar{\thetab}_{\tau}^{(j)}\right\|\left\|\thetab_{\tau}^{(i)} - \bar{\thetab}_{\tau}^{(i)}-(\thetab_{\tau}^{(j)} - \bar{\thetab}_{\tau}^{(j)})\right\|\left\|\thetab_{\tau}^{(i)} - \bar{\thetab}_{\tau}^{(i)}\right\|\right)\\
	& \leq L_FL_K\sum_{i,j} \sqrt{\mathbb{E}\left\|\thetab_{\tau}^{(i)} - \bar{\thetab}_{\tau}^{(i)}\right\|^2}\sqrt{\mathbb{E}\left(\left\|\thetab_{\tau}^{(i)} - \bar{\thetab}_{\tau}^{(i)}-(\thetab_{\tau}^{(j)} - \bar{\thetab}_{\tau}^{(j)})\right\|^2\left\|\bar{\thetab}_{\tau}^{(j)}\right\|^2\right)}\\
	&\leq L_FL_K\sum_{i,j}\sqrt{\mathbb{E}\left\|\thetab_{\tau}^{(i)} - \bar{\thetab}_{\tau}^{(i)}\right\|^2}\sqrt{\mathbb{E}\left\|\thetab_{\tau}^{(i)} - \bar{\thetab}_{\tau}^{(i)}-(\thetab_{\tau}^{(j)} - \bar{\thetab}_{\tau}^{(j)})\right\|^2\mathbb{E}\left\|\bar{\thetab}_{\tau}^{(j)}\right\|^2} \\
	&\leq H_{\theta} L_K \sum_{i,j}\sqrt{\mathbb{E}\left\|\thetab_{\tau}^{(i)} - \bar{\thetab}_{\tau}^{(i)}-(\thetab_{\tau}^{(j)} - \bar{\thetab}_{\tau}^{(j)})\right\|^2\mathbb{E}\left\|\thetab_{\tau}^{(i)} - \bar{\thetab}_{\tau}^{(i)}\right\|^2} \\
	& \leq H_{\theta} L_K \sum_{i,j}\sqrt{\left(2\mathbb{E}\left\|\thetab_{\tau}^{(i)} - \bar{\thetab}_{\tau}^{(i)}\right\|^2+2\mathbb{E}\left\|(\thetab_{\tau}^{(j)} - \bar{\thetab}_{\tau}^{(j)})\right\|^2\right)\mathbb{E}\left\|\thetab_{\tau}^{(i)} - \bar{\thetab}_{\tau}^{(i)}\right\|^2} \\
	& \leq \frac{1}{2}H_{\theta} L_K \sum_{i,j}\left( 2\mathbb{E}\left\|\thetab_{\tau}^{(i)} - \bar{\thetab}_{\tau}^{(i)}\right\|^2+3\mathbb{E}\left\|(\thetab_{\tau}^{(j)} - \bar{\thetab}_{\tau}^{(j)})\right\|^2\right)\\
	&\leq \frac{5}{2}H_{\theta} L_K M\mathbb{E}\sum_{i} \left \|\thetab_{\tau}^{(i)} - \bar{\thetab}_{\tau}^{(i)}\right\|^2
	\end{align*}
	The above result is derived with Cauchy-Schwarz inequality and the independence between 
	$\thetab_{\tau}^{(i)} - \bar{\thetab}_{\tau}^{(i)}-(\thetab_{\tau}^{(j)} - \bar{\thetab}_{\tau}^{(j)})$ and $\bar{\thetab}_{\tau}^{(j)}$. The independency come from the following argument: 
	According to our constructions of all the $\bar{\thetab}_{\tau}^{(i)}$ in \eqref{eqnew:nonlinear_sde}, we conclude that all the $\bar{\thetab}_{\tau}^{(i)}$ are identical. Hence, we have $\thetab_{\tau}^{(i)} - \bar{\thetab}_{\tau}^{(i)}-(\thetab_{\tau}^{(j)} - \bar{\thetab}_{\tau}^{(j)})=\thetab_{\tau}^{(i)}-\thetab_{\tau}^{(j)}$, which is obvious independent of $\bar{\thetab}_{\tau}^{(j)}$. 
	
	For the $I_{i,j}(\tau)$ term, following the analysis in \cite{Malrieu:AAP03,CattiauxGM:PTRF08,2018arXiv180511387D} and applying the independency between $\mathcal{W}_{\tau}^{(i)}-\bar{\mathcal{W}}_{\tau}^{(i)}$ and $\thetab_{\tau}^{(i)} - \bar{\thetab}_{\tau}^{(i)}$, we have 
	\begin{align*}
	    \mathbb{E}\sum_{ij} I_{ij}(\tau) =0.
	\end{align*}
	Denote $\gamma_i(\tau) \triangleq \mathbb{E}\left\|\thetab_{\tau}^{(i)} - \bar{\thetab}_{\tau}^{(i)}\right\|^2$. Due to the exchangeability of the particles, $\gamma_i(\tau)$ have the same value for all the particles, denoted as $\gamma(\tau)$. According to \eqref{eqnew:ABCDGH} and the bounds derived above, we have
	\begin{align*}
	&\gamma^\prime(\tau) \leq -2\lambda_1\gamma(\tau)+ \frac{2H_{\nabla K} +2H_{\theta}} {\sqrt{M}}\sqrt{\gamma(\tau)}~.
	\end{align*}
	where $\lambda_1=\beta^{-1}m_F-\frac{5}{2}H_{\theta} L_K-L_F-2L_{\nabla K}$. After some algebra, the above inequality can be transformed to
	\begin{align*}
	&(\sqrt{\gamma(\tau)}-\frac{2(H_{\nabla K}+H_{\theta})}{\sqrt{M}(\beta^{-1}m_F-\frac{5}{2}H_{\theta} L_K-L_F-2L_{\nabla K})})^\prime \\
	\leq& -\lambda_1(\sqrt{\gamma(\tau)}-\frac{2(H_{\nabla K}+H_{\theta})}{\sqrt{M}(\beta^{-1}m_F-\frac{5}{2}H_{\theta} L_K-L_F-2L_{\nabla K})})
	\end{align*}
	
	Note that $\thetab_{\tau}^{(i)}$ and $\bar{\thetab}_{\tau}^{(i)}$ are initialized with the same initial distribution $\mu_0=\nu_0$ but independent of each other. 
	From the proof of Theorems~\ref{app:use1} and \ref{thm:continousbound2}, we can have $\gamma(0)\leq 4\gamma_0$ for some constant $\gamma_0$. When we set $\gamma_0$ small enough, we can have the following results according to the Gronwall Lemma.
	
	\begin{align*}
	\sqrt{\gamma(\tau) }\leq \frac{2(H_{\nabla K}+H_{\theta})}{\sqrt{M}(\beta^{-1}m_F-\frac{5}{2}H_{\theta} L_K-L_F-2L_{\nabla K})}~
	\end{align*}

	Hence, there exist some positive constant $(c_1,c_2)$ such that:
	\begin{align}\label{leq:import}
	&{W}_1(\rho_{\tau},\nu_{\tau}) \stackrel{(1)}{\leq} {W}_2(\rho_{\tau},\nu_{\tau}) \nonumber\\
	\stackrel{(2)}{\leq}& \sqrt{\mathbb{E}\left\|\thetab_{\tau}^{(i)} - \bar{\thetab}_{\tau}^{(i)}\right\|^2} \stackrel{(3)}{\leq} ~\frac{c_{1}}{\sqrt{M}(\beta^{-1}-c_{2})},
	\end{align}
	where $(1)$ holds due to the relationship between ${W}_1$ and ${W}_2$ metric \cite{GivensS:MMJ84}, $(2)$ due to the definition of ${W}_2$, and $(3)$ due to the result from the previous proof.

\end{proof}

\section{Proof of Theorem~\ref{theonew:Bt2}}

\begin{proof}[Proof of Theorem~\ref{theonew:Bt2}]
	First, note our goal is to bound ${W}_1(\nu_{\tau},\nu_{\infty})\leq c_{3}\exp{(-2\lambda_{1}\tau)}$. According to the relationship between ${W}_1$ and ${W}_2$ metric that $W_1 \leq W_2$ \cite{GivensS:MMJ84}, once we bound ${W}_2(\nu_{\tau},\nu_{\infty})$ as ${W}_2(\nu_{\tau},\nu_{\infty})\leq c_{3}\exp{(-2\lambda_{1}\tau)}$, the bound for $W_1$ will automatically hold.
	
	In the following, we will bound $W_2$. We first note the following cases based on equation \eqref{eq:particle}:
	\begin{itemize}
		\item We set the initial distribution of each particle to be $\nu_{0}$, which means $\rho_0=\mathcal{L}({\thetab}^{(i)}_0) = \nu_{0}$. In this case, the $M$ evolved particles are denoted as $\{{\thetab}^{(i)}_{\tau,1}\}_{i=1}^M$. We denote the distribution of each ${\thetab}^{(i)}_{\tau,1}$ at $\tau$ as $\rho_{\tau,1}$.
		\item We set the initial distribution of each particle to be $\nu_{\infty}$, which means $\rho_0=\mathcal{L}({\thetab}^{(i)}_0) = \nu_{\infty}$. In this case, the $M$ evolved particles are denoted as $\{{\thetab}^{(i)}_{\tau,2}\}_{i=1}^M$. We denote the distribution of each ${\thetab}^{(i)}_{\tau,2}$ at $\tau$ as $\rho_{\tau,2}$.\\
	\end{itemize}
	
	To bound ${W}_2(\nu_{\tau},\nu_{\infty})$, we decompose it as:
	\begin{align}\label{eq:3terms}
	W_2(\nu_{\tau}, \nu_{\infty}) &\leq W_2(\nu_{\tau}, \rho_{\tau,1}) + W_2(\rho_{\tau,1},\rho_{\tau,2}) + W_2(\rho_{\tau,2}, \nu_{\infty})~.
	\end{align}
	Note that $\rho_{0,1}=\nu_0$ and $\rho_{0,2}=\nu_{\infty}$. According to \eqref{leq:import}, we have 
	\begin{align*}
	& W_2(\nu_{\tau}, \rho_{\tau,1}) \leq \frac{c_{1}}{\sqrt{M}(\beta^{-1}-c_{2})} \\
	& W_2(\rho_{\tau,2}, \nu_{\infty}) \leq \frac{c_{1}}{\sqrt{M}(\beta^{-1}-c_{2})}
	\end{align*}
	It remains to bound the term $W_2(\rho_{\tau,1},\rho_{\tau,2})$. It is worth mentioning that the reason of introducing $\{{\thetab}^{(i)}_{\tau,1}\}_{i=1}^M$ and $\{{\thetab}^{(i)}_{\tau,2}\}_{i=1}^M$ is to bound the term ${W}_2(\nu_{\tau},\nu_{\infty})$, which consequently is to bound $W_2(\rho_{\tau,1},\rho_{\tau,2})$. For some special settings of $\{{\thetab}^{(i)}_{\tau,1}\}_{i=1}^M$ and $\{{\thetab}^{(i)}_{\tau,2}\}_{i=1}^M$, it will allow us to bound $W_2(\rho_{\tau,1},\rho_{\tau,2})$ easier. To this end, we set all the $\{{\thetab}^{(i)}_{0,1}\}_{i=1}^M$ and the corresponding $\mathcal{W}_{\tau,1}^{(i)}$ to be exactly the same. Consequently, all the $\{{\thetab}^{(i)}_{\tau,1}\}_{i=1}^M$ will be identical. In this setting, the bound proved above for $W_2(\nu_{\tau}, \rho_{\tau,1})$ still holds since this is just a specific case for Theorem~\ref{theonew:chaosp}. The same argument goes for $\{{\thetab}^{(i)}_{\tau,2}\}_{i=1}^M$. And we are left to prove the bound for  $W_2(\rho_{\tau,1},\rho_{\tau,2})$. 

	Since $W_2(\rho_{\tau,1},\rho_{\tau,2}) \leq \mathbb{E}\left( \left\|{\thetab}^{(i)}_{\tau,1} - {\thetab}^{(i)}_{\tau,2}\right\|^2 \right)\triangleq r(\tau)$, we will derive a bound for $\mathbb{E}\left( \left\|{\thetab}^{(i)}_{\tau,1} - {\thetab}^{(i)}_{\tau,2}\right\|^2 \right)$ in the following:
	\begin{align*}
	&\mathrm{d}\left({\thetab}^{(i)}_{\tau,1} - {\thetab}^{(i)}_{\tau,2}\right) = -\beta^{-1}\left(F({\thetab}^{(i)}_{\tau,1}) - F({\thetab}^{(i)}_{\tau,2})\right)\mathrm{d}\tau\\
	&+\frac{1}{M}\sum_j^{M}\left[\nabla K({\thetab}^{(i)}_{\tau,1}- {\thetab}^{(j)}_{\tau,1}) - \nabla K({\thetab}^{(i)}_{\tau,2}- {\thetab}^{(j)}_{\tau,2})\right]\mathrm{d}\tau \\
	&-\frac{1}{M}\sum_j^{M}\left(F(\thetab_{\tau,1}^{(j)})K({\thetab}^{(i)}_{\tau,1}- {\thetab}^{(j)}_{\tau,1}) - F(\thetab_{\tau,2}^{(j)})K({\thetab}^{(i)}_{\tau,2}- {\thetab}^{(j)}_{\tau,2})\right)\mathrm{d}\tau \\
	& +(\sqrt{2\beta^{-1}}\mathrm{d}\mathcal{W}_{\tau,1}^{(i)}-\sqrt{2\beta^{-1}}\mathrm{d}{\mathcal{W}}_{\tau,2}^{(i)})\mathrm{d}\tau
	\end{align*}
	As a result, we have
	\begin{align*}
	\mathrm{d}\left(\sum_i^{M}\left\|{\thetab}^{(i)}_{\tau,1} - {\thetab}^{(i)}_{\tau,2}\right\|^2\right) =&\frac{2}{M}\sum_{i,j}^{M} (\xi_{ij}^{1}(\tau) + \xi_{ij}^{2}(\tau)  + \xi_{ij}^{3}(\tau)  +\xi_{ij}^{4}(\tau)  +\xi_{ij}^{5}(\tau))\mathrm{d}\tau \\
	\end{align*}
	where
	\begin{align*}
	\xi_{ij}^{1}(\tau)  &= -\beta^{-1}\left(F({\thetab}^{(i)}_{\tau,1}) - F({\thetab}^{(i)}_{\tau,2})\right)\cdot \left({\thetab}^{(i)}_{\tau,1} - {\thetab}^{(i)}_{\tau,2}\right)\nonumber\\
	\xi_{ij}^{2}(\tau) &= \left(\nabla K({\thetab}^{(i)}_{\tau,1}- {\thetab}^{(j)}_{\tau,1}) - \nabla K({\thetab}^{(i)}_{\tau,2}- {\thetab}^{(j)}_{\tau,2})\right)\cdot \left({\thetab}^{(i)}_{\tau,1} - {\thetab}^{(i)}_{\tau,2}\right)\nonumber\\
	\xi_{ij}^{3}(\tau)  &= -\left(F(\thetab_{\tau,1}^{(j)})K({\thetab}^{(i)}_{\tau,1}- {\thetab}^{(j)}_{\tau,1}) - F(\thetab_{\tau,2}^{(j)}) K({\thetab}^{(i)}_{\tau,1}- {\thetab}^{(j)}_{\tau,1})\right)\cdot \left({\thetab}^{(i)}_{\tau,1} - {\thetab}^{(i)}_{\tau,2}\right)\nonumber\\
	\xi_{ij}^{4}(\tau)  &= -\left(F(\thetab_{\tau,2}^{(j)}) K({\thetab}^{(i)}_{\tau,1}- {\thetab}^{(j)}_{\tau,1})- F(\thetab_{\tau,2}^{(j)}) K({\thetab}^{(i)}_{\tau,2}- {\thetab}^{(j)}_{\tau,2})\right)\cdot \left({\thetab}^{(i)}_{\tau,1} - {\thetab}^{(i)}_{\tau,2}\right)\nonumber\\
	\xi_{ij}^{5}(\tau)  &=(\sqrt{2\beta^{-1}}\mathrm{d}\mathcal{W}_{\tau,1}^{(i)}-\sqrt{2\beta^{-1}}\mathrm{d}{\mathcal{W}}_{\tau,2}^{(i)}))\cdot \left({\thetab}^{(i)}_{\tau,1} - {\thetab}^{(i)}_{\tau,2}\right)\nonumber\\
	\end{align*}
	
	For the $\xi_{ij}^{1}(\tau)$ terms, according to Assumption~\ref{assnew:ass3new} for $F$, we have
	\begin{align*}
	\mathbb{E}\sum_{ij}\xi_{ij}^{1}(\tau)&= -\mathbb{E}\sum_{ij}\beta^{-1}\left(F({\thetab}^{(i)}_{\tau,1}) - F({\thetab}^{(i)}_{\tau,2})\right)\cdot \left({\thetab}^{(i)}_{\tau,1} - {\thetab}^{(i)}_{\tau,2}\right)\\
	&\leq -\beta^{-1}m_FM\mathbb{E}\sum_i\left\|{\thetab}^{(i)}_{\tau,1} - {\thetab}^{(i)}_{\tau,2}\right\|^2 ~.
	\end{align*}
	
	For the $\xi_{ij}^{2}(\tau)$ term, applying the concave condition for $K$ and the oddness of $\nabla K$ in Assumption~\ref{assnew:ass1}, we have
	\begin{align*}
	&\mathbb{E}\sum_{ij}\xi_{ij}^{2}(\tau)= \mathbb{E}\sum_{ij}^{M}\left(\nabla K({\thetab}^{(i)}_{\tau,1}- {\thetab}^{(j)}_{\tau,1}) - \nabla K({\thetab}^{(i)}_{\tau,2}- {\thetab}^{(j)}_{\tau,2})\right) \cdot \left({\thetab}^{(i)}_{\tau,1} - {\thetab}^{(i)}_{\tau,2}\right) \\
	&=\frac{1}{2}\sum_{ij}^{M}\mathbb{E}\left(\nabla K({\thetab}^{(i)}_{\tau,1}- {\thetab}^{(j)}_{\tau,1}) - \nabla K({\thetab}^{(i)}_{\tau,2}- {\thetab}^{(j)}_{\tau,2})\right) \cdot \left({\thetab}^{(i)}_{\tau,1} - {\thetab}^{(i)}_{\tau,2}-({\thetab}^{(j)}_{\tau,1} - {\thetab}^{(j)}_{\tau,2})\right)\\
	&\leq \frac{1}{2} L_K\mathbb{E}\sum_{ij}^{M}\left\|{\thetab}^{(i)}_{\tau,1} - {\thetab}^{(i)}_{\tau,2}-({\thetab}^{(j)}_{\tau,1} - {\thetab}^{(j)}_{\tau,2})\right\|^2  \leq 2L_KM\mathbb{E}\sum_i\left\|{\thetab}^{(i)}_{\tau,1} - {\thetab}^{(i)}_{\tau,2}\right\|^2 ~.
	\end{align*}
	
	For the $\xi_{ij}^{3}(\tau)$ terms, applying the $L_F$-Lipschitz property for $F$ and using $K \leq 1$, we have
	\begin{align*}
	&\mathbb{E}\sum_{ij}\xi_{ij}^{3}(\tau) \\
	=& \mathbb{E}\sum_{ij}-\left(F(\thetab_{\tau,1}^{(j)})K({\thetab}^{(i)}_{\tau,1}- {\thetab}^{(j)}_{\tau,1}) - F(\thetab_{\tau,2}^{(j)}) K({\thetab}^{(i)}_{\tau,1}- {\thetab}^{(j)}_{\tau,1})\right) \cdot \left({\thetab}^{(i)}_{\tau,1} - {\thetab}^{(i)}_{\tau,2}\right)\nonumber\\
	\leq& \mathbb{E}\sum_{ij} L_F \left\|{\thetab}^{(j)}_{\tau,1} - {\thetab}^{(j)}_{\tau,2}\right\|\left\|{\thetab}^{(i)}_{\tau,1} - {\thetab}^{(i)}_{\tau,2}\right\|\\
	\leq& L_F M\mathbb{E}\sum_{i} \left \|{\thetab}^{(i)}_{\tau,1} - {\thetab}^{(i)}_{\tau,2}\right\|^2~.
	\end{align*}
	
	For the $\xi_{ij}^{4}(\tau)$ terms, recall that all the ${\thetab}^{(i)}_{\tau,1}$ are identical (and all the ${\thetab}^{(i)}_{\tau,2}$ are identical), we have
	\begin{align*}
	&\mathbb{E}\sum_{ij} \xi_{ij}^{4}(\tau) \\
	&= -\mathbb{E}\sum_{ij}\left(F(\thetab_{\tau,2}^{(j)}) K({\thetab}^{(i)}_{\tau,1}- {\thetab}^{(j)}_{\tau,1}) - F(\thetab_{\tau,2}^{(j)}) K({\thetab}^{(i)}_{\tau,2}- {\thetab}^{(j)}_{\tau,2})\right) \cdot \left({\thetab}^{(i)}_{\tau,1} - {\thetab}^{(i)}_{\tau,2}\right)=0~.
	\end{align*}
	Similar to the proof of the Theorem~\ref{theonew:chaosp}, we have
	\begin{align*}
	    	\mathbb{E}\sum_{i,j}\xi_{ij}^{5}(\tau)  &=\mathbb{E}\sum_{i,j}(\sqrt{2\beta^{-1}}\mathrm{d}\mathcal{W}_{\tau,1}^{(i)}-\sqrt{2\beta^{-1}}\mathrm{d}{\mathcal{W}}_{\tau,2}^{(i)}))\cdot \left({\thetab}^{(i)}_{\tau,1} - {\thetab}^{(i)}_{\tau,2}\right)\\
	    	&=0
	\end{align*}
	Combining these bounds, we have
	\begin{align*}
	r^\prime(\tau) \leq -2(\beta^{-1}m_F-L_F-2L_K)r(\tau)~.
	\end{align*}
	
	According to the Gronwall lemma, we have
	\begin{align*}
	r(\tau)\leq r(0) e^{-2\lambda_{1}\tau},
	\end{align*}
	where $\lambda_{1}=\beta^{-1}m_F-L_F-2L_K$.\\

	Consequently, there exists some positive constant $c_3$ such that 
	\begin{align*}
	W_2(\rho_{\tau,1},\rho_{\tau,2}) \leq c_3e^{-2\lambda_{1}\tau}
	\end{align*}
	Combing all bounds for \eqref{eq:3terms}, we have
	\begin{align*}
	W_2(\nu_{\tau}, \nu_{\infty}) \leq &c_3e^{-2\lambda_{1}\tau}+\frac{c_{1}}{\sqrt{M}(\beta^{-1}-c_{2})} +\frac{c_{1}}{\sqrt{M}(\beta^{-1}-c_{2})}
	\end{align*}
	We can further tighten the above bound by noting that $\nu_{\tau}$ is the solution of \eqref{eq:unifiednew}, which has nothing to do with the number of particles $M$. As a result, we can set $M\rightarrow \infty$, resulting in 
	$$W_2(\nu_{\tau}, \nu_{\infty}) \leq c_3e^{-2\lambda_{1}\tau}~,$$ 
	which completes the proof.
	
\end{proof}

\section{Proof of Theorem~\ref{theonew:num_error}}\label{sec:protheonewnum_error}

To bound the ${W}_1(\mu_T,\rho_{\sum_{k=0}^{T-1} h_k})$ term, note the original SDE driving the particles $\{{\thetab}_{\tau}^{(i)}\}$ in \eqref{eq:particle} corresponds to is a nonlinear PDE, which is hard to deal with. Fortunately, \eqref{eq:particle} can be turned into a diffusion-based SDE by concatenating the particles at each time into a single vector representation, {\it i.e.}, by defining the new parameter at time $\tau$ as $\Thetab_{\tau}\triangleq [\thetab_{\tau}^{(1)}, \cdots, \thetab_{\tau}^{(M)}] \in \mathbb{R}^{Md}$. Consequently, $\Thetab_{\tau}$ is driven by the following SDE:\vspace{-0.1cm}
\begin{align}\label{eqapp:extendP}
\mathrm{d}\Thetab_{\tau} = -F_{\Thetab}(\Thetab_{\tau})\mathrm{d}\tau + \sqrt{2\beta^{-1}}\mathrm{d}\mathcal{W}_{\tau}^{(Md)}~,
\end{align}
where 
\begin{align*}
F_{\Thetab}(\Thetab_{\tau}) \triangleq [&\beta^{-1}F(\thetab_{\tau}^{(1)})-\frac{1}{M}\sum_{j=1}^M\nabla K(\thetab_{\tau}^{(1)}-\thetab_{\tau}^{(j)})+\frac{1}{M}\sum_{j=1}^{M}K(\thetab_{\tau}^{(1)} - \thetab_{\tau}^{(j)})F(\thetab_{\tau}^{(j)}), \cdots, \\
&\beta^{-1}F(\thetab_{\tau}^{(M)})-\frac{1}{M}\sum_{j=1}^M\nabla K(\thetab_{\tau}^{(M)}-\thetab_{\tau}^{(j)})+\frac{1}{M}\sum_{j=1}^{M}K(\thetab_{\tau}^{(M)} - \thetab_{\tau}^{(j)})F(\thetab_{\tau}^{(j)})]
\end{align*}
is a vector function $\mathbb{R}^{Md}\rightarrow \mathbb{R}^{Md}$, and $\mathcal{W}{\tau}^{(Md)}$ is Brownian motion of dimension $Md$.\\

Now we define $F_{(q)\Thetab}(\Thetab_{\tau}) \triangleq [\beta^{-1}F_q(\thetab_{\tau}^{(1)})-\frac{1}{MN}\sum_{j=1}^M\nabla K(\thetab_{\tau}^{(1)}-\thetab_{\tau}^{(j)})+\frac{1}{M}\sum_{j=1}^{M}K(\thetab_{\tau}^{(1)} - \thetab_{\tau}^{(j)})F_q(\thetab_{\tau}^{(j)}), \cdots, \beta^{-1}F_q(\thetab_{\tau}^{(M)})-\frac{1}{MN}\sum_{j=1}^M\nabla K(\thetab_{\tau}^{(M)}-\thetab_{\tau}^{(j)})+\frac{1}{M}\sum_{j=1}^{M}K(\thetab_{\tau}^{(M)} - \thetab_{\tau}^{(j)})F_q(\thetab_{\tau}^{(j)})]$. We can verify that $F_{\Thetab}(\Thetab_{\tau})=\sum_{q=1}^NF_{(q)\Thetab}(\Thetab_{\tau})$.\\

Define ${\Theta}_k \triangleq [{\theta}_{k}^{(1)}, \cdots, {\theta}_{k}^{(M)}]$ and $G_{\mathcal{I}_k}^{\Theta} \triangleq \frac{N}{B_k}\sum_{q\in\mathcal{I}_k}F_{(q)\Thetab}(\Theta_k)$. It is seen that the following result holds:
\begin{align}\label{eqapp:numerror1}
\Theta_{k+1} = \Theta_{k} - G_{\mathcal{I}_k}^{\Theta}h_k + \sqrt{2\beta^{-1}h_k}\Xi_{k}~,
\end{align}
where $\Xi_{k}\sim \mathcal{N}(\mathbf{0}, \Ib_{Md \times Md})$. As a result, we have that ${\Theta}_k$ of \eqref{eqapp:numerror1} is accutually  the numerical solution of the SDE \eqref{eqapp:extendP} via stochastic gradients.\\

Denote the distribution of $\Theta_{k}$ as $\mu_k^{\Theta}$, and the distribution of $\Thetab_{\tau}$ as $\rho_{\tau}^{\Thetab}$. Before proceeding to our theoretical results, we first present the following Lemmas, which is very important in our proof.
\begin{lemma}\label{lem:W}
	${W}_1(\mu_k, \rho_{\tau})\leq \frac{1}{\sqrt{M}}{W}_1(\mu_k^{\Theta},\rho_{\tau}^{\Thetab})$
\end{lemma}

\begin{proof}[Proof of Lemma~\ref{lem:W}]
	Let us recall the definition of $W_1$ metric and its Kantorovich-Rubinstein duality \cite{Villani:08}, {\it i.e.} ${W}_1(\mu,\nu) \triangleq \sup_{\|g\|_{lip}\leq 1}\left|\mathbb{E}_{\thetab \sim \mu}[g(\thetab)] - \mathbb{E}_{\thetab \sim \nu}[g(\thetab)]\right|$. We can prove the fact that if $g(\thetab):\mathbb{R}^{d}\rightarrow \mathbb{R}$ is a $L_g$-Lipschitz function in $\mathbb{R}^{d}$, the $g_{\Thetab}(\Thetab)$, defined as $g_{\Thetab}(\Thetab) =\frac{1}{\sqrt{M}}\sum_{i}^{M}g(\thetab^{(i)})$, is a $L_g$-Lipschitz function in $\mathbb{R}^{Md}$, where $\Thetab \triangleq [{\thetab}^{(1)}, \cdots, {\thetab}^{(M)}]$. The proof is as follows:
	\begin{align*}
	&\|g_{\Thetab}(\Thetab_1)-g_{\Thetab}(\Thetab_2)\|\leq \frac{1}{\sqrt{M}}\sum_{i=1}^{M}\|g(\thetab_1^{(i)})-g(\thetab_2^{(i)})\| \\
	&\leq \frac{L_g}{\sqrt{M}} \sum_{i=1}^{M} \|\thetab_1^{(i)}-\thetab_2^{(i)}\|\\
	&\leq \frac{L_g}{\sqrt{M}} \sqrt{M} \sqrt{ \sum_{i=1}^{M} \|\thetab_1^{(i)}-\thetab_2^{(i)}\|^2}= L_g\|\Thetab_1-\Thetab_2\|
	\end{align*}
	As a result, we have:
	\begin{align*}
	& \frac{1}{M}\sum_{i=1}^{M}\left|\mathbb{E}_{\theta_k^{(i)} \sim \mu_k}[g(\theta_k^{(i)})] - \mathbb{E}_{\thetab_{\tau}^{(i)} \sim \rho_{\tau}}[g(\thetab_{\tau}^{(i)})]\right| \\
	\stackrel{(1)}{=}&\frac{1}{\sqrt{M}}\left|\frac{1}{\sqrt{M}}\sum_{i=1}^{M}(\mathbb{E}_{\theta_k^{(i)} \sim \mu_k}[g(\theta_k^{(i)})] - \mathbb{E}_{\thetab_{\tau}^{(i)} \sim \rho_{\tau}}[g(\thetab_{\tau}^{(i)})])\right|\\
	{=}&\frac{1}{\sqrt{M}}\left|\mathbb{E}_{\Theta_k \sim \mu_k}[g_{\Thetab}(\Theta_k)] - \mathbb{E}_{\Thetab_{\tau} \sim \rho_{\tau}}[g_\Thetab(\Thetab_{\tau})]\right|~,
	\end{align*}
	where $(1)$ holds because $\mathbb{E}_{\theta_k^{(1)} \sim \mu_k}[g(\theta_k^{(1)})]=\cdot \cdot \cdot= \mathbb{E}_{\theta_k^{(M)} \sim \mu_k}[g(\theta_k^{(M)})]$ for all the particles $\theta_k^{(i)}$, and $\mathbb{E}_{\thetab_{\tau}^{(1)} \sim \rho_{\tau}}[g(\thetab_{\tau}^{(1)})]=\cdot \cdot \cdot=\mathbb{E}_{\thetab_{\tau}^{(M)} \sim \rho_{\tau}}[g(\thetab_{\tau}^{(M)})] $
	for all the particles $\thetab_{\tau}^{(i)}$.
	According to the definition of ${W}_1$ metric, we derive that 
	\begin{align*}
	&{W}_1(\mu_k, \rho_{\tau}) \\
	=& \sup_{\|g\|_{lip}\leq 1}\frac{1}{M}\sum_{i=1}^{M}\left|\mathbb{E}_{\theta_k^{(i)} \sim \mu_k}[g(\theta_k^{(i)})] - \mathbb{E}_{\thetab_{\tau}^{(i)} \sim \rho_{\tau}}[g(\thetab_{\tau}^{(i)})]\right| \\
	=&\frac{1}{\sqrt{M}}\sup_{\|g\|_{lip}\leq 1}\left|\mathbb{E}_{\Theta_k \sim \mu_k}[g_{\Thetab}(\Theta_k)] - \mathbb{E}_{\Thetab_{\tau} \sim \rho_{\tau}}[g_\Thetab(\Thetab_{\tau})]\right| \\
	=& \frac{1}{\sqrt{M}}\sup_{\|g_{\Thetab}\|_{lip}\leq 1}\left|\mathbb{E}_{\Theta_k \sim \mu_k}[g_{\Thetab}(\Theta_k)] - \mathbb{E}_{\Thetab_{\tau} \sim \rho_{\tau}}[g_\Thetab(\Thetab_{\tau})]\right| \\
	& \leq \frac{1}{\sqrt{M}}{W}_1(\mu_k^{\Theta},\rho_{\tau}^{\Thetab})~,
	\end{align*}
	which completes the proof.
\end{proof}

\begin{lemma}\label{lem:lip}
	Assuming $F(\mathbf{0})=\mathbf{0}$. If $F$ in (\ref{eq:particle_num}) 
	is Lipschitz with constant $L_F$, and 
	satisfies the dissipative property that $\langle F(\thetab), \thetab\rangle \geq m_F\left\|\thetab\right\|^2 - b$. Then $F_{\Thetab}$ in (\ref{eqapp:extendP})  
	satisfies $\langle F_{\Thetab}(\Thetab), \Thetab\rangle \geq (\beta^{-1}m_F-m^{\prime})\left\|\Thetab\right\|^2 - \beta^{-1}Mb$, where $l^{\prime}$ and $m^{\prime}$ are some positive constants. Besides we have  $\mathbb{E}\|F_{\Thetab}(\Thetab_1)-F_{\Thetab}(\Thetab_2)\|^2\leq(\sqrt{2}\beta^{-1}L_F+l^{\prime})\mathbb{E}\|\Thetab_1-\Thetab_2\|^2$ if the same settings as in the proof of Theorem~\ref{theonew:Bt2} is adopted.
\end{lemma}

\begin{proof}[Proof of Lemma~\ref{lem:lip}]
	
We will bound $\langle F_{\Thetab}(\Thetab), \Thetab\rangle$ by noting that:
	\begin{align*}
	&\langle F_{\Thetab}(\Thetab), \Thetab\rangle \\
	=&\sum_{i}^{M}\left(\beta^{-1}F(\thetab^{(i)})\thetab^{(i)}+\frac{1}{M}\sum_{j}^{M}K(\thetab^{(i)}-\thetab^{(j)})F(\thetab^{(j)})\thetab^{(i)} -\frac{1}{M}\sum_{j}^{M}\nabla K(\thetab^{(i)}-\theta^{(j)})\thetab^{(i)}\right)
	\end{align*}
	
	Notice that:
	\begin{align*}
	\sum_{i}^{M}\beta^{-1}F(\thetab^{(i)})\thetab^{(i)}&\geq \beta^{-1} m_F\sum_{i}^{M}\|\thetab^{(i)}\|^2-\beta^{-1}Mb\\
	&= \beta^{-1}m_F\|\Thetab\|^2-\beta^{-1}Mb
	\end{align*}
	Furthermore, since it is assumed that $F(0)=\mathbf{0}$, we have:
	\begin{align*}
	&\sum_{i}^{M}\frac{1}{M}\sum_{j}^{M}K(\thetab^{(i)}-\thetab^{(j)})F(\thetab^{(j)})\thetab^{(i)} \\
	\geq &-\frac{1}{M}\sum_{i}^{M}\sum_{j}^{M}L_F\|\thetab^{(i)}\|\|\thetab^{(j)}\|\\
	\geq &-L_F\sum_{i=1}^{M}\|\thetab^{(i)}\|^2 = -L_F\|\Thetab\|^2
	\end{align*}
	In addition, since $\nabla K$ is an odd function, we have:
	\begin{align*}
	&\sum_{i}^{M}\frac{1}{M}\sum_{j}^{M}\nabla K(\thetab^{(i)}-\thetab^{(j)})\thetab^{(i)} \\
	\geq&- \sum_{i}^{M}\frac{1}{M}\sum_{j}^{M}\frac{2}{\eta^2}\|\thetab^{(i)}-\thetab^{(j)}\|\|\thetab^{(i)}\|\\
	\geq& -\frac{4}{\eta^2}\sum_{i}^{M}\|\thetab^{(i)}\|^2=-\frac{4}{\eta^2}\|\Thetab\|^2
	\end{align*}
	
	As a result, we arrive at the following result:
	\begin{align*}
	&\langle F_{\Thetab}(\Thetab), \Thetab\rangle
	\geq (\beta^{-1}m-L_F-\frac{4}{\eta^2})\|\Thetab\|^2-\beta^{-1}Mb~.
	\end{align*}
	
	Furthermore, for the other conclusion, we have:
	\begin{align*}
	&\mathbb{E}\|F_{\Thetab}(\Thetab_1)-F_{\Thetab}(\Thetab_2)\|^2=\mathbb{E}{\sum_i^{M}\left \|\omega^{1}_i+\omega^{2}_i+\omega^{3}_i+\omega^{4}_i\right \|^2}\\
	&\leq \mathbb{E}\sum_i^{M}\left(\|\omega^{1}_i\|+\|\omega^{2}_i\|+\|\omega^{3}_i\|+\|\omega^{4}_i\|\right)^2\\
	&\leq 4\mathbb{E}\sum_i^{M}\left(\|\omega^{1}_i\|^2+\|\omega^{2}_i\|^2+\|\omega^{3}_i\|^2+\|\omega^{4}_i\|^2\right)
	\end{align*}
	where
	\begin{align*}
	\sum_i^{M}\mathbb{E}\|\omega^{1}_i\|^2&=\sum_i^{M}\mathbb{E}\|\beta^{-1}F(\thetab_1^{(i)})-\beta^{-1}F(\thetab_2^{(i)})\|^2
	\leq \sum_i^{M}\beta^{-2}L_F^2\mathbb{E}\|\thetab_1^{(i)}-\thetab_2^{(i)}\|^2\\
	\sum_i^{M}\mathbb{E}\|\omega^{2}_i\|^2&=\frac{1}{M^2}\mathbb{E}\|\sum_{j}^{M}K(\thetab_1^{(i)}-\thetab_1^{(j)})F(\thetab_1^{(j)}) -\sum_{j}^{M} K(\thetab_2^{(i)}-\thetab_2^{(j)})F(\thetab_1^{(j)})\|^2\\
	&= \frac{1}{M}\sum_{j}^{M}\mathbb{E}\|F(\thetab_1^{(j)})(K(\thetab_1^{(i)}-\thetab_1^{(j)})-K(\thetab_2^{(i)}-\thetab_2^{(j)}))\|^2=0\\
	\sum_i^{M}\mathbb{E}\|\omega^{3}_i\|^2&=\sum_i^{M}\mathbb{E}\|\frac{1}{M}(\sum_{j}^{M}K(\thetab_2^{(i)}-\thetab_2^{(j)})F(\thetab_1^{(j)}) -\sum_{j}^{M} K(\thetab_2^{(i)}-\thetab_2^{(j)})F(\thetab_2^{(j)}))\|^2\\
	&\leq \sum_i^{M}L_F^2\mathbb{E}\|\thetab_1^{(i)}-\thetab_2^{(i)}\|^2\\
	\sum_i^{M}\mathbb{E}\|\omega^{4}_i\|^2&=\sum_i^{M}\mathbb{E}\|-\frac{1}{M}(\sum_{j}^{M}\nabla K(\thetab_1^{(i)}-\thetab_1^{(j)})-\sum_{j}^{M}\nabla K(\thetab_2^{(i)}-\thetab_2^{(j)}))\|^2\\
	&\leq \sum_i^{M}\mathbb{E}\frac{L_{\nabla K}^2}{M}(\sum_{j}^{M}\|\thetab_1^{(i)}-\thetab_2^{(i)}-(\thetab_1^{(j)}-\thetab_2^{(j)})\|^2)\\
	&\leq \sum_i^{M}\mathbb{E} L_{\nabla K}^2(2\|\thetab_1^{(i)}-\thetab_2^{(i)}\|^2+\frac{2}{M}\sum_{j}^{M}\|\thetab_1^{(j)}-\thetab_2^{(j)}\|^2)\\
	& \leq 4L_{\nabla K}^2\sum_i^{M}\mathbb{E}\|\thetab_1^{(i)}-\thetab_2^{(i)}\|^2
	\end{align*}
	Hence, we have 
	\begin{align*}
	    \mathbb{E}\|F_{\Thetab}(\Thetab_1)-F_{\Thetab}(\Thetab_2)\|^2\leq(\beta^{-2}L_F^2+L_F^2+4L_{\nabla K}^2)\mathbb{E}\|\Thetab_1-\Thetab_2\|^2
	\end{align*}


\end{proof}

Now it is ready to prove Theorem~\ref{theonew:num_error}. It is worth noting that with the assumption of $F(0)=\mathbf{0}$, the first bullet in Assumption~\ref{assnew:ass1} recovers the dissipative assumption as $\langle F(\thetab), \thetab\rangle \geq m_F\left\|\thetab\right\|^2$.
\begin{proof}
	We use Lemma C.5 in \cite{XuCZG:arxiv17} to verify that $F_{\Thetab}$ satisfies the assumptions in \cite{RaginskyRT:COLT17} by setting $\delta=\frac{a{\prime}}{B}$ with $a^{\prime}$ a positive constant and $B$ the size of the random set $\mathcal{I}$. 
	
	Let $\mu^{\Theta}_{k}:=\mathcal{L}(\Theta_{k})$ and $\rho^{\Thetab}_{\tau}:=\mathcal{L}({\Thetab}_{\tau})$.  We make some modifications to the proof of Lemma 3.6 in \cite{RaginskyRT:COLT17} and derive the following results.
	The relative entropy $D_{KL}(\mu^{\Theta}_{k}\|\rho^{\Thetab}_{kh})$ satisfies:
	\begin{align*}
	D_{KL}(\mu^{\Theta}_{k}\|\rho^{\Thetab}_{kh}) \leq (A_0\beta\frac{a{\prime}}{B}+A_1h)kh
	\end{align*}
	with 
	\begin{align*}
	A_0=&\left(2(\beta^{-2}L_F^2+L_F^2+4L_{\nabla K}^2)\left(a_2+2(1\vee \frac{1}{\beta^{-1}m_F-m^{\prime}}) \cdot(2a_1^2+\frac{Md}{\beta})\right)+a_1^2\right)\\
	A_1=&6(\beta^{-2}L_F^2+L_F^2+4L_{\nabla K}^2)(\beta A_0+Md)
	\end{align*}
	and $a_1,a_2$ are some positive constants.
	When $\beta$ is small enough such that the subtraction terms in the above bounds are positive, there exist some positive constants $a_3,a_4$ such that 
	\begin{align*}
	A_0\leq a_3\frac{Md}{\beta^3},  \text{ and } A_1\leq a_4\frac{Md}{\beta^4}	
	\end{align*}
	Similar to the proof of Lemma~\ref{lem:lip}, it is easy to verify that  there exists some positive constant $a_5$ such that $\langle F_{\Thetab}(\Thetab_1)-F_{\Thetab}(\Thetab_2), \Thetab_1-\Thetab_2\rangle \geq (\beta^{-1}m_F-a_5)\|\Thetab_1-\Thetab_2\|^2$. Note that when $\beta$ is small enough, \eqref{eqapp:extendP} satisfies the conditions of Proposition 4.2 in \cite{CattiauxGM:PTRF08}. Hence, there exits some positive constant $\mathcal{C}$ such that ${W}_1(\mu_k^{\Theta},\rho_{kh}^{\Thetab}) \leq \mathcal{C}\sqrt{D_{KL}(\mu^{\Theta}_{k}\|\rho^{\Thetab}_{kh})}~.$
	
	According to Corollary 4 and Lemma 8 in \cite{BolleyV:Ann05}, we can derive an explicit expression for $\mathcal{C}$ :
	\begin{align*}
	\mathcal{C}\leq a_6\beta^{-1}Md~,
	\end{align*}
	when $\beta$ is a small enough constant and $a_6$ is some positive constant.
	
	Applying Lemma~\ref{lem:W}, we have 
	\begin{align*}
	&{W}_1(\mu_k, \rho_{kh})\leq \frac{1}{\sqrt{M}}{W}_1(\mu_k^{\Theta},\rho_{kh}^{\Thetab}) \\
	\leq& a_6Md^{\frac{3}{2}}\beta^{-3}(a_3a^{\prime}\beta^2B^{-1}+a_4h)^{\frac{1}{2}}k^{\frac{1}{2}}h^{\frac{1}{2}}
	\end{align*}
	Setting $k=T$ completes the poof.
\end{proof}

\section{Proof of Theorem~\ref{theonew:adanum_error}}
\begin{proof}
	Our proof is based on the techniques in the proof of Lemma 3.6 in \cite{RaginskyRT:COLT17}. Firstly, adopting the same notation as in Section \ref{sec:protheonewnum_error}, we have the following update:
	\begin{align}\label{eqapp:numerror}
	\Theta_{k+1} = \Theta_{k} - \beta^{-1}G_{\mathcal{I}_k}^{\Theta}h_k + \sqrt{2\beta^{-1}h_k}\Xi_{k}~,
	\end{align}
	where $\Xi_{k}\sim \mathcal{N}(\mathbf{0}, \Ib_{Md \times Md})$ and $h_k=\frac{h_0}{k+1}$.
	We note the unbiasness of $G_{\mathcal{I}_k}^{\Theta}$, {\it i.e.}, $\mathbb{E}(G_{\mathcal{I}_k}^{\Theta})={F}_{\Thetab}(\Theta_k),~~\forall \Thetab \in \mathbb{R}^{Md}$, due to the way we choose the minibatch $\mathcal{I}_k$.
	We need to define $q(\tau)$, which will be used in the following proof:
	\begin{align*}
	q(\tau)=\{k\in \mathbb{R}|\sum_{i=0}^{k-1}h_i\leq \tau <\sum_{i=0}^{k}h_i\}~.
	\end{align*}
	Furthermore, define $\sum_{i=0}^{-1}h_i \triangleq 0$ and  $\sum_{i=0}^{0}h_i \triangleq h_0$ for the convenience of statement in the following.

	Now we focus on the following continuous-time interpolation of ${\Theta}_{k}$:
	\begin{align*}
	\underline{{\Thetab}}(\tau)=&{\Thetab}_{0}-\int_0^{\tau}\tilde{G}_{\mathcal{I}(s)}^{\Thetab}\left(\underline{{\Thetab}}(\sum_{i=0}^{q(s)-1}h_i)\right)\mathrm{d}s+\sqrt{\frac{2}{\beta}}\int_{0}^{\tau}\mathcal{W}_{s}^{(Md)},
	\end{align*}
	where $\mathcal{I}(s)\equiv \mathcal{I}_{k}$ for $\tau\in\left[\sum_{i=0}^{k-1}h_{i},\sum_{i=0}^{k}h_{i}\right)$, $\tilde{G}_{\mathcal{I}(s)}^{\Thetab}(\Thetab)\triangleq \frac{N}{B(s)}\sum_{q\in\mathcal{I}(s)}F_{(q)\Thetab}(\Thetab)$ and $B(s)$ is the size of the minibatch $\mathcal{I}(s)$. It is easily seen that for each $k$, $\underline{{\Thetab}}(\sum_{i=0}^{k-1}h_i)$ and ${\Theta}_{k}$ have the same probability law $\rho^{\Theta}_{k}$. Besides we need some similar settings in the proof of Theorem~\ref{theonew:Bt2} for $\mathcal{W}_{s}^{(Md)}$.
	Since $\underline{{\Thetab}}(\tau)$ is not a Markov process, we define the following Itô process which has the same one-time marginals as $\underline{{\Thetab}}(\tau)$
	\begin{align*}
	\Lambda(\tau)={\Theta}_{0}-\int_0^{\tau}\underline{G}_{s}\left(\Lambda(s)\right)\mathrm{d}s+\sqrt{\frac{2}{\beta}}\int_{0}^{\tau}\mathcal{W}_{s}^{(Md)}
	\end{align*}
	\begin{align*}
	\mbox{where } \underline{G}_{\tau}(x):=\mathbb{E}\left[\tilde{G}_{\mathcal{I}(\tau)}^{\Thetab}\left(\underline{{\Thetab}}(\sum_{i=0}^{q(\tau)-1}h_i)\right)|\underline{{\Thetab}}(\tau)=x\right]~.
	\end{align*}
	Let the probability laws $\Pb_{\Lambda}^{\tau}:=\mathcal{L}\left(\Lambda(s):0\leq s\leq \tau\right)$ and $\Pb_{\Thetab}^{\tau}:=\mathcal{L}\left(\Thetab(s):0\leq s\leq \tau\right)$. According to the proof of lemma 3.6 in \cite{RaginskyRT:COLT17}, we can derive a similar result for the relative entropy of $\Pb_{\Lambda}^{\tau}$ and $\Pb_{\Thetab}^{\tau}$:
	\begin{align*}
	&D_{KL}(\Pb_{\Lambda}^{\tau}\|\Pb_{\Thetab}^{\tau})=-\int \mathrm{d}\Pb_{\Lambda}^{\tau} \text{log}\frac{\mathrm{d}\Pb_{\Lambda}^{\tau}}{\mathrm{d}\Pb_{\Thetab}^{\tau} }	\nonumber\\
	=&\frac{\beta}{4}\int_{0}^{\tau}\mathbb{E}\|{F}_{\Thetab}(\Lambda(s))-\underline{G}_{s}\left(\Lambda(s)\right)\|^2\mathrm{d}s\\
	=&\frac{\beta}{4}\int_{0}^{\tau}\mathbb{E}\|{F}_{\Thetab}(\underline{{\Thetab}}(s))-\underline{G}_{s}\left(\underline{{\Thetab}}(s)\right)\|^2\mathrm{d}s~,
	\end{align*}
	where the last line follows because $\mathcal{L}(\underline{{\Thetab}}(s))=\mathcal{L}(\Lambda(s)),~\forall s$.

	In the following proof, we let $\tau=\sum_{i=0}^{k-1}h_i$ for some $k \in \mathbb{R}$.
	Now we can use the martingale property (conditional independence) of Itô integral to derive:
	{\begin{align}
		&D_{KL}(\Pb_{\Lambda}^{\sum_{i=0}^{k-1}h_{i}}\|\Pb_{\Thetab}^{\sum_{i=0}^{k-1}h_{i}}) \nonumber\\
		&=\frac{\beta}{4}\sum_{j=0}^{k-1}\int_{\sum_{i=0}^{j-1}h_{i}}^{\sum_{i=0}^{j}h_{i}}\mathbb{E}\|{F}_{\Thetab}(\underline{{\Thetab}}(s))-\underline{G}_{s}\left(\underline{{\Thetab}}(s)\right)\|^2\mathrm{d}s \nonumber\\
		&\leq \frac{\beta}{2}\sum_{j=0}^{k-1}\int_{\sum_{i=0}^{j-1}h_{i}}^{\sum_{i=0}^{j}h_{i}}\mathbb{E}\|{F}_{\Thetab}(\underline{{\Thetab}}(s))-{F}_{\Thetab}(\underline{{\Thetab}}(\sum_{i=0}^{q(s)-1}h_{i}))\|^2\mathrm{d}s \nonumber\\ 
		&~~~~+\frac{\beta}{2}\sum_{j=0}^{k-1}\int_{\sum_{i=0}^{j-1}h_{i}}^{\sum_{i=0}^{j}h_{i}}\mathbb{E}\left\|{F}_{\Thetab}(\underline{{\Thetab}}(\sum_{i=0}^{q(s)-1}h_{i})) -\tilde{G}_{\mathcal{I}(s)}^{\Thetab}\left(\underline{{\Thetab}}(\sum_{i=0}^{q(s)-1}h_i)\right)\right\|^2\mathrm{d}s \nonumber\\
		&\leq \frac{\beta L_{F_{\Thetab}}^2}{2}\sum_{j=0}^{k-1}\int_{\sum_{i=0}^{j-1}h_{i}}^{\sum_{i=0}^{j}h_{i}}\mathbb{E}\|\underline{{\Thetab}}(s)-\underline{{\Thetab}}(\sum_{i=0}^{q(s)-1}h_{i})\|^2\mathrm{d}s \label{eq:1st}\\ 
		&+\frac{\beta}{2}\sum_{j=0}^{k-1}\int_{\sum_{i=0}^{j-1}h_{i}}^{\sum_{i=0}^{j}h_{i}}\mathbb{E}\left\|{F}_{\Thetab}(\underline{{\Thetab}}(\sum_{i=0}^{q(s)-1}h_{i})) -\tilde{G}_{\mathcal{I}(s)}^{\Thetab}\left(\underline{{\Thetab}}(\sum_{i=0}^{q(s)-1}h_i)\right)\right\|^2\mathrm{d}s~, \label{eq:2nd}
		\end{align}}
        where $L_{F_{\Thetab}}^2\triangleq \beta^{-2}L_F^2+L_F^2+4L_{\nabla K}^2$.
	
	For the first part of \eqref{eq:1st}, consider some $s\in [{\sum_{i=0}^{j-1}h_{i}},{\sum_{i=0}^{j}h_{i}} )$. From the definitions, the following equation holds:
	\begin{align*}
	&\underline{{\Thetab}}(s)-\underline{{\Thetab}}({\sum_{i=0}^{j-1}h_{i}}) \\
	=&-(s-\sum_{i=0}^{j-1}h_{i})G_{\mathcal{I}_j}^{\Theta}+\sqrt{2/\beta}(\mathcal{W}_{s}^{(Md)}-\mathcal{W}_{\sum_{i=0}^{j-1}h_{i}}^{(Md)})\\
	=&-(s-\sum_{i=0}^{j-1}h_{i})G_{\mathcal{I}_j}^{\Theta}+
	(s-\sum_{i=0}^{j-1}h_{i})({F}_{\Thetab}({\Theta}_j)-G_{\mathcal{I}_j}^{\Theta})+\sqrt{2/\beta}(\mathcal{W}_{s}^{(Md)}-\mathcal{W}_{\sum_{i=0}^{j-1}h_{i}}^{(Md)})
	\end{align*}

	Applying results from Lemma 3.1 and 3.2 in \cite{RaginskyRT:COLT17}, and Lemma C.5 in \cite{XuCZG:arxiv17}, we have:
	\begin{align*}
	&\mathbb{E}\|\underline{{\Thetab}}(s)-\underline{{\Thetab}}({\sum_{i=0}^{j-1}h_{i}})\|^2 \\
	\leq& 3\frac{{h_0}^2}{(j+1)^2}\mathbb{E}\|G_{\mathcal{I}_j}^{\Theta}\|^2+3\frac{{h_0}^2}{(j+1)^2}\mathbb{E}\|{F}_{\Thetab}({\Theta}_j)-G_{\mathcal{I}_j}^{\Theta}\|^2 +\frac{6h_0Md}{\beta(j+1)}\\
	\leq& 12\frac{{h_0}^2}{(j+1)^2}\max_{0\leq j\leq k-1}(L_{F_\Thetab}^2\mathbb{E}\|{\Theta}_j\|^2+b_1)+\frac{6h_0Md}{\beta(j+1)}
	\end{align*}
	where $b_1$ is some positive constant.
	
	Consequently, the first part of \eqref{eq:1st} can be bounded as:
	\begin{align*}
	&\frac{\beta L_{F_{\Thetab}}^2}{2}\sum_{j=0}^{k-1}\int_{\sum_{i=0}^{j-1}h_{i}}^{\sum_{i=0}^{j}h_{i}}\mathbb{E}\|\underline{{\Thetab}}(s)-\underline{{\Thetab}}(\sum_{i=0}^{q(s)-1}h_{i})\|^2\mathrm{d}s  \\
	\leq& \frac{\beta L_{F_{\Thetab}}^2}{2}\sum_{j=0}^{k-1}\left[12\frac{{h_0}^3}{(j+1)^3}\max_{0\leq j\leq K-1}(L_{F_\Thetab}^2\mathbb{E}\|{\Theta}_j\|^2+b_1) +\frac{6{h_0}^2Md}{\beta(j+1)^2}\right]\\
	\leq& \pi^2\beta L_{F_{\Thetab}}^2{h_0}^3\max_{0\leq j\leq K-1}(L_{F_\Thetab}^2\mathbb{E}\|{\Theta}_j\|^2+b_1) +\frac{\pi^2L_{F_{\Thetab}}^2{h_0}^2Md}{2}~,
	\end{align*}
	where the last inequality follows from the fact that
	\begin{align*}
	\sum_{j=0}^{k-1}\frac{1}{(j+1)^3}\leq\sum_{j=0}^{k-1}\frac{1}{(j+1)^2}\leq\sum_{j=0}^{\infty}\frac{1}{(j+1)^2}=\frac{\pi^2}{6}~.
	\end{align*}
	
	Now we bound the second part \eqref{eq:2nd}. According to Lemma C.5 in \cite{XuCZG:arxiv17}, we have:
	\begin{align*}
	&\frac{\beta}{2}\sum_{j=0}^{k-1}\int_{\sum_{i=0}^{j-1}h_{i}}^{\sum_{i=0}^{j}h_{i}}\mathbb{E}\left\|{F}_{\Thetab}(\underline{{\Thetab}}(\sum_{i=0}^{q(s)-1}h_{i})) -\tilde{G}_{\mathcal{I}(s)}^{\Thetab}\left(\underline{{\Thetab}}(\sum_{i=0}^{q(s)-1}h_i)\right)\right\|^2\mathrm{d}s\\
	=&\sum_{j=0}^{k-1}\frac{\beta h_0}{2(j+1)}\mathbb{E}\|{F}_{\Thetab}({\Theta}_j)-G_{\mathcal{I}_j}^{\Theta}\|^2\\
	\leq& \beta h_0\max_{0\leq j\leq k-1}(L_{F_\Thetab}^2\mathbb{E}\|{\Theta}_j\|^2+b_1) \cdot \left(\frac{4}{B_0}+\sum_{j=1}^{k-1}\frac{4}{(j+1)(B_0+\log^{\frac{100}{99}}(j+1))}\right)\\
	\leq& \beta h_0\max_{0\leq j\leq k-1}(L_{F_\Thetab}^2\mathbb{E}\|{\Theta}_j\|^2+b_1) \cdot \left(\frac{4}{B_0}+\sum_{j=1}^{k-1}\frac{4}{(j+1)\log^{\frac{100}{99}}(j+1)}\right)\\
	\leq& (b_2+\frac{4}{B_0})\beta h_0\max_{0\leq j\leq k-1}(L_{F_\Thetab}^2\mathbb{E}\|{\Theta}_j\|^2+b_1)~,
	\end{align*}
	where the last inequality follows from the fact that when $r>1$,  
	\begin{align*}
	&\sum_{j=1}^{k-1}\frac{4}{(j+1)\log^r(j+1)} \\
	\leq& \sum_{j=1}^{\infty}\frac{4}{(j+1)\log^r(j+1)}\leq \frac{4\log^{1-r}2}{r-1}~.
	\end{align*}

	Denote $\mu^{\Theta}_{k}:=\mathcal{L}({\Theta}_{k})$ and $\rho^{\Thetab}_{\tau}:=\mathcal{L}({\Thetab}_{\tau})$. Due to the data-processing inequality for the relative entropy, we have
	{\small\begin{align*}
		&D_{KL}(\mu^{\Theta}_{k}\|\rho^{\Thetab}_{\sum_{i=0}^{k-1}h_{i}})\leq D_{KL}(\Pb_{\Lambda}^{\sum_{i=0}^{k-1}h_{i}}\|\Pb_{\Thetab}^{\sum_{i=0}^{k-1}h_{i}})\\
		&\leq\pi^2\beta L_{F_{\Thetab}}^2{h_0}^3\max_{0\leq j\leq k-1}(L_{F_\Thetab}^2\mathbb{E}\|{\Theta}_j\|^2+b_1)+\frac{\pi^2L_{F_{\Thetab}}^2{h_0}^2Md}{2} +(b_2+\frac{4}{B_0})\beta h_0\max_{0\leq j\leq k-1}(L_{F_\Thetab}^2\mathbb{E}\|\tilde{\Thetab}_j\|^2+b_1)\\
		&\leq (\pi^2\beta L_{F_{\Thetab}}^2{h_0}^3+b_2\beta h_0+\frac{4}{B_0}\beta h_0)\max_{0\leq j\leq k-1}(L_{F_\Thetab}^2\mathbb{E}\|{\Theta}_j\|^2+b_1) +\frac{\pi^2L_{F_{\Thetab}}^2{h_0}^2Md}{2}~.
		\end{align*}}
	 Theorem~\ref{app:use3} has provided a uniform bound to $\max_{0\leq j\leq k-1}(L_{F_\Thetab}^2\mathbb{E}\|{\Theta}_j\|^2+b_1)$. Hence it can be concluded that $D_{KL}(\Pb_{\Lambda}^{\sum_{i=0}^{k-1}h_{i}}\|\Pb_{\Thetab}^{\sum_{i=0}^{k-1}h_{i}})$ would not increase w.r.t.\! $k$. This is a nice property that the fixed-step-size SPOS does not endow. Since $L_{F_{\Thetab}}^2\triangleq \beta^{-2}L_F^2+L_F^2+4L_{\nabla K}^2$, it is easy to verify that when $\beta$ is small enough, there exists some positive constants $b_3,b_4,b_5 $ and $b_6$ such that:
	\begin{align*}
	&D_{KL}(\mu^{\Theta}_{k}\|\rho^{\Thetab}_{\sum_{i=0}^{k-1}h_{i}}) \\
	\leq &(\pi^2\beta L_{F_{\Thetab}}^2{h_0}^3+b_2\beta h_0+\frac{4\beta h_0}{B_0}) \times\max_{0\leq j\leq K-1}(L_{F_\Thetab}^2\mathbb{E}\|\tilde{\Thetab}_j\|^2+b_1)+\frac{\pi^2L_{F_{\Thetab}}^2{h_0}^2Md}{2}\\
	\leq &(b_3h_0^3+\frac{b_4\beta^3h_0}{B_0}+b_5h_0^2\beta^2)\frac{Md}{\beta^4}~.
	\end{align*}
	
	Similar to the proof of Theorem~\ref{theonew:num_error}, we can bound the $W_1(\mu^{\Theta}_{k}\|\rho^{\Thetab}_{\sum_{i=0}^{k-1}h_{i}})$ term with Corollary 4, Lemma 8 in \cite{BolleyV:Ann05} and Proposition 4.2 in \cite{CattiauxGM:PTRF08}. Specifically, when $\beta$ is small enough, there exist some positive constant $a_6$ such that:
	\begin{align*}
	&W_1(\mu^{\Theta}_{k}\|\rho^{\Thetab}_{\sum_{i=0}^{k-1}h_{i}}) \leq a_6 (\frac{Md}{\beta})\sqrt{D_{KL}(\mu^{\Theta}_{k}\|\rho^{\Thetab}_{\sum_{i=0}^{k-1}h_{i}})}\\
	\leq& a_6 \beta^{-3}M^{\frac{3}{2}}d^{\frac{3}{2}}(b_3h_0^3+\frac{b_4\beta^3h_0}{B_0}+b_5h_0^2\beta^2)^{\frac{1}{2}}~.
	\end{align*}
	According to Lemma~\ref{lem:W}, we have
	\begin{align*}
	&{W}_1(\mu_k, \rho_{kh})\leq \frac{1}{\sqrt{M}}W_1(\mu^{\Theta}_{k}\|\rho^{\Thetab}_{\sum_{i=0}^{k-1}h_{i}}) \\
	=&a_6 \beta^{-3}Md^{\frac{3}{2}}(b_3h_0^3+\frac{b_4\beta^3h_0}{B_0}+b_5h_0^2\beta^2)^{\frac{1}{2}}
	\end{align*}
	Setting $k=T$ finishes the proof.
\end{proof}

\section{Proof of Theorems~\ref{rem:rem2_} and \ref{theonew:chaospnew}}
\begin{proof}[Proof for Theorem~\ref{rem:rem2_}]
    The proof is by direct calculation:
	\begin{align*}
	&\mathrm{d}\left(\thetab_{\tau}^{(i)} - {\thetab}_{\tau}^{(j)}\right) =\\ 
	&\frac{1}{M}\sum_q^{M}\left[\nabla K(\thetab_{\tau}^{(i)}- \thetab_{\tau}^{(q)}) - \nabla K(\thetab_{\tau}^{(j)}- \thetab_{\tau}^{(q)})\right]\mathrm{d}\tau \\ &-\frac{1}{M}\sum_q^{M}\left(F(\thetab_{\tau}^{(q)})K(\thetab_{\tau}^{(i)}-\thetab_{\tau}^{(q)}) - F(\thetab_{\tau}^{(q)})K(\thetab_{\tau}^{(j)}-\thetab_{\tau}^{(q)}))\right)\mathrm{d}\tau  \nonumber
	\end{align*}
	\begin{align*}
	\Rightarrow  &\mathrm{d}\left(\mathbb{E}\sum_{i j}^{M}\left\|\thetab_{\tau}^{(i)} - {\thetab}_{\tau}^{(j)}\right\|^2\right) =\\
	&\mathbb{E}\sum_{i j}^{M} \frac{2}{M}\sum_q^{M}\left[\nabla K(\thetab_{\tau}^{(i)}- \thetab_{\tau}^{(q)}) - \nabla K(\thetab_{\tau}^{(j)}- \thetab_{\tau}^{(q)})\right] \times \left(\thetab_{\tau}^{(i)} - {\thetab}_{\tau}^{(j)}\right)\mathrm{d}\tau \nonumber\\
	&-\mathbb{E}\sum_{i j}^{M} \frac{2}{M}\sum_q^{M}\left(F(\thetab_{\tau}^{(q)})K(\thetab_{\tau}^{(i)}-\thetab_{\tau}^{(q)}) - F(\thetab_{\tau}^{(q)})K(\thetab_{\tau}^{(j)}-\thetab_{\tau}^{(q)}))\right)\left(\thetab_{\tau}^{(i)} - {\thetab}_{\tau}^{(j)}\right)\mathrm{d}\tau \nonumber\\
	&\leq -2m_K \mathbb{E}\sum_{i j}^{M}\left\|\thetab_{\tau}^{(i)} - {\thetab}_{\tau}^{(j)}\right\|^2 \mathrm{d}\tau +2H_F L_K\mathbb{E}\sum_{i j}^{M}\left\|\thetab_{\tau}^{(i)} - {\thetab}_{\tau}^{(j)}\right\|^2\mathrm{d}\tau~, \nonumber\\
	\end{align*}
	where $H_F$ is the maximum value of $\|F(\theta)\|$ on the bounded space.
	Denote $z(\tau)= \mathbb{E}\sum_{i j}^{M}\left\|\thetab_{\tau}^{(i)} - {\thetab}_{\tau}^{(j)}\right\|^2$. We have
	\begin{align}\label{eq:zzsvgd}
	z(\tau)^{\prime}\leq -(2m_K-2H_FL_K)z(\tau)	
	\end{align}
	Applying Gronwall Lemma on \eqref{eq:zzsvgd} finishes the proof.
\end{proof}

\begin{proof}[Proof of Theorem~\ref{theonew:chaospnew}]


For the SPOS, we have
	\begin{align}
	&\mathrm{d}\left(\thetab_{\tau}^{(i)} - {\thetab}_{\tau}^{(j)}\right) = -\beta^{-1}\left(F(\thetab_{\tau}^{(i)}) - F({\thetab}_{\tau}^{(j)})\right)\mathrm{d}\tau\nonumber\\
	&+\frac{1}{M}\sum_q^{M}\left[\nabla K(\thetab_{\tau}^{(i)}- \thetab_{\tau}^{(q)}) - \nabla K(\thetab_{\tau}^{(j)}- \thetab_{\tau}^{(q)})\right]\mathrm{d}{\tau} \nonumber\\
	&-\frac{1}{M}\sum_q^{M}\left(F(\thetab_{\tau}^{(q)})K(\thetab_{\tau}^{(i)}-\thetab_{\tau}^{(q)}) - F(\thetab_{\tau}^{(q)})K(\thetab_{\tau}^{(j)}-\thetab_{\tau}^{(q)}))\right)\mathrm{d}{\tau}  \nonumber\\
	&+\sqrt{\frac{2}{\beta}}(\mathrm{d}\mathcal{W}_{\tau}^{(i)}-\mathrm{d}\mathcal{W}_{\tau}^{(j)})\nonumber
	\end{align}
	Hence we have 
	\begin{align}
	&\Rightarrow  \mathrm{d}\left(\mathbb{E}\sum_{i j}^{M}\left\|\thetab_{\tau}^{(i)} - {\thetab}_{\tau}^{(j)}\right\|^2\right) \\ =&-\mathbb{E}2 \sum_{i j}^{M}\beta^{-1}\left(F(\thetab_{\tau}^{(i)}) - F({\thetab}_{\tau}^{(j)})\right)\left(\thetab_{\tau}^{(i)} - {\thetab}_{\tau}^{(j)}\right)\mathrm{d}{\tau}     \nonumber \\
	&+\mathbb{E}\sum_{i j}^{M} \frac{2}{M}\sum_q^{M}\left[\nabla K(\thetab_{\tau}^{(i)}- \thetab_{\tau}^{(q)}) - \nabla K(\thetab_{\tau}^{(j)}- \thetab_{\tau}^{(q)})\right] \left(\thetab_{\tau}^{(i)} - {\thetab}_{\tau}^{(j)}\right)\mathrm{d}{\tau}\nonumber\\
	&-\mathbb{E}\sum_{i j}^{M} \frac{2}{M}\sum_q^{M}\left(F(\thetab_{\tau}^{(q)})K(\thetab_{\tau}^{(i)}-\thetab_{\tau}^{(q)}) - F(\thetab_{\tau}^{(q)})K(\thetab_{\tau}^{(j)}-\thetab_{\tau}^{(q)}))\right)\left(\thetab_{\tau}^{(i)} - {\thetab}_{\tau}^{(j)}\right)\mathrm{d}{\tau} \nonumber\\
	&+\mathbb{E}2\sum_{i j}^{M} \sqrt{\frac{2}{\beta}}(\mathrm{d}\mathcal{W}_{\tau}^{(i)}-\mathrm{d}\mathcal{W}_{\tau}^{(j)})\left(\thetab_{\tau}^{(i)} - {\thetab}_{\tau}^{(j)}\right) \nonumber\\
	&\leq -2\beta^{-1}m_F\mathbb{E}\sum_{i j}^{M}\left\|\thetab_{\tau}^{(i)} - {\thetab}_{\tau}^{(j)}\right\|^2\mathrm{d}{\tau} \nonumber\\
	&-2m_K \mathbb{E}\sum_{i j}^{M}\left\|\thetab_{\tau}^{(i)} - {\thetab}_{\tau}^{(j)}\right\|^2 \mathrm{d}{\tau} +2H_FL_K\mathbb{E}\sum_{i j}^{M}\left\|\thetab_{\tau}^{(i)} - {\thetab}_{\tau}^{(j)}\right\|^2\mathrm{d}{\tau} \nonumber\\
	& +2\sqrt{\frac{2}{\beta}} \left(\mathbb{E}\sum_{i j}^{M}(\mathrm{d}\mathcal{W}_{\tau}^{(i)}-\mathrm{d}\mathcal{W}_{\tau}^{(j)})^2\right)^{1/2}\left(\mathbb{E}\sum_{i j}^{M}\left\|\thetab_{\tau}^{(i)} - {\thetab}_{\tau}^{(j)}\right\|^2\right)^{1/2}~. \nonumber
	\end{align}
	Denote $z(\tau)= \mathbb{E}\sum_{i j}^{M}\left\|\thetab_{\tau}^{(i)} - {\thetab}_{\tau}^{(j)}\right\|^2$. We have
	\begin{align}\label{eq:zz}
	z(\tau)^{\prime}\leq -(2\beta^{-1}m_F+2m_K-2H_F L_K)z(\tau)+4M\sqrt{\frac{d}{\beta}z(\tau)}
	\end{align}
	Applying Gronwall Lemma on \eqref{eq:zz} finished the proof.
\end{proof}
Based on the bound, we can see that the particles in SPOS will not converge to one point, overcoming the pitfall of SVGD.

\section{The uniform bounds on SPOS}

Following the formulations in Section \ref{sec:protheonewnum_error}, we will derive the following theorem about the uniform bound on each particle in SPOS.
\begin{theorem} \label{app:use3}
    For the $\Theta_k$ in \eqref{eqapp:numerror1}, we have 
    \begin{align*}
    E\|\Theta_{k}\|^2\leq M\gamma_0+(1\vee\frac{1}{m^\prime})\frac{2Md}{\beta}
    \end{align*}
\end{theorem}

First we need to prove the following lemma.
\begin{lemma}\label{big_gradient_bound}
For the $F_{\Thetab}$ in the \eqref{eqapp:extendP}and $\Theta_k$ in the \eqref{eqapp:numerror1}, we have the following result
    \begin{align}
        \|F_{\Thetab}(\Theta_k)\|^2 \leq (3\beta^{-2}L_F^2+3L_F^2+48/{\eta^4})\|\Theta_k\|^2
    \end{align}
\end{lemma}

\proof[Proof of Lemma~\ref{big_gradient_bound}]
\begin{align*}
    \|F_{\Thetab}(\Theta_k)\|^2=&\sum_i^{M}\left \|\omega^{1}_i+\omega^{2}_i+\omega^{3}_i\right \|^2\\
    &\leq\sum_i^{M}\left(\|\omega^{1}_i\|+\|\omega^{2}_i\|+\|\omega^{3}_i\|\right)^2
    \leq 3\sum_i^{M} \left( \|\omega^{1}_i\|^2+\|\omega^{2}_i\|^2+\|\omega^{3}_i\|^2\right)
\end{align*}
where
	\begin{align*}
	\|\omega^{1}_i\|&=\|\beta^{-1}F(\thetab_k^{(i)})\|
	\leq \beta^{-1}L_F\|\thetab_k^{(i)}\|\\
	\|\omega^{2}_i\|&=\|\frac{1}{M}\sum_{j}^{M}K(\thetab_k^{(i)}-\thetab_k^{(j)})F(\thetab_k^{(j)})\|\\
	&\leq L_F\frac{1}{M}\sum_{j}^{M}\|\thetab_k^{(j)}\|\\
	\|\omega^{3}_i\|&=\|-\frac{1}{M}\sum_{j}^{M}\nabla K(\thetab_k^{(i)}-\thetab_k^{(j)})\|\\
	&\leq \frac{{2}/{\eta^2}}{M}\sum_{j}^{M}\|\thetab_k^{(i)}-\thetab_k^{(j)}\|\\
	&\leq \frac{2}{\eta^2}(\|\thetab_k^{(i)}\|+\frac{1}{M}\sum_{j}^{M}\|\thetab_k^{(j)}\|)\\
	\end{align*}
Substituting the above bounds into $F_{\Thetab}(\Theta_k)$, it is easy to verify that  
\begin{align*}
\|F_{\Thetab}(\Theta_k)\|^2 \leq &3\sum_i^{M}\left(  \beta^{-2}L_F^2 \|\thetab_k^{(i)}\|^2 +\frac{L_F^2}{M}\sum_j^M \|\thetab_k^{(j)}\|^2+ 2({2}/{\eta^2})^2\|\thetab_k^{(i)}\|^2+\frac{2({2}/{\eta^2})^2}{M}\sum_j^M \|\thetab_k^{(j)}\|^2\right)\\
&\leq (3\beta^{-2}L_F^2+3L_F^2+48/{\eta^4})\|\Theta_k\|^2
\end{align*}

With the Lemma \ref{lem:lip} and \ref{big_gradient_bound}, we can now derive the the uniform bound on each particle in SPOS. Our proof is based on the proof of Lemma 3.2 in \cite{RaginskyRT:COLT17}

\begin{proof}[Proof of Theorem~\ref{app:use3}]
From \eqref{eqapp:numerror1}, it follows that 
    \begin{align*}
          E\|\Theta_{k+1}\|^2=&E\|\Theta_k-G_{\mathcal{I}_k}^{\Theta}h_k\|^2+\sqrt{\frac{8h_k}{\beta}}E\langle\Theta_k-G_{\mathcal{I}_k}^{\Theta}h_k,\Xi_{k}\rangle+\frac{2h_k}{\beta}E\|\Xi_{k}\|^2 \\
          =&E\|\Theta_k-G_{\mathcal{I}_k}^{\Theta}h_k\|^2+\frac{2h_kMd}{\beta}
    \end{align*}
    where the second step uses independence of $\Theta_k-G_{\mathcal{I}_k}^{\Theta}h_k$ and $\Xi_{k}$,the unbiasedness property that $E[G_{\mathcal{I}_k}^{\Theta}]=F_{\Thetab}(\Theta_k)$ and $E[\Xi_{k}]=0$
    \begin{align}
        E\|\Theta_k-G_{\mathcal{I}_k}^{\Theta}h_k\|^2&=
        E\|\Theta_k-F_{\Thetab}(\Theta_k)h_k\|^2
        +2h_kE\langle \Theta_k-F_{\Thetab}(\Theta_k)h_k,F_{\Thetab}(\Theta_k)-G_{\mathcal{I}_k}^{\Theta} \rangle +h_k^2E\|F_{\Thetab}(\Theta_k)-G_{\mathcal{I}_k}^{\Theta} \|^2 \nonumber \\ 
        &=E\|\Theta_k-F_{\Thetab}(\Theta_k)h_k\|^2+h_k^2E\|F_{\Thetab}(\Theta_k)-G_{\mathcal{I}_k}^{\Theta} \|^2 \label{decom:lem2}
    \end{align}

The first term in \eqref{decom:lem2} can estimated as 
\begin{align*}
    E\|\Theta_k-F_{\Thetab}(\Theta_k)h_k\|^2&=E\|\Theta_k\|^2-2h_kE\langle \Theta_k,F_{\Thetab}(\Theta_k)\rangle+h_k^2E\|F_{\Thetab}(\Theta_k)\|^2 \\
    &\leq E\|\Theta_k\|^2 + 2h_k(-(\beta^{-1}m-L_F-\frac{4}{\eta^2})E\|\Theta_k\|^2)+h_k^2(3\beta^{-2}L_F^2+3L_F^2+48/{\eta^4})E\|\Theta_k\|^2\\
    &\leq (1-2h_km_F^\prime+h_k^2L^\prime)E\|\Theta_k\|^2
\end{align*}
where $m^\prime\triangleq\beta^{-1}m_F-L_F-\frac{4}{\eta^2}$ and $L^\prime\triangleq 3\beta^{-2}L_F^2+3L_F^2+48/{\eta^4}$.

Following the Lemma C.5 from \cite{XuCZG:arxiv17} and some modifications (the settings are a bit different,but the results are the same), we could estimate the the second term in \eqref{decom:lem2} as 
\begin{align}
    E\|F_{\Thetab}(\Theta_k)-G_{\mathcal{I}_k}^{\Theta} \|^2 \leq \frac{2(N-B)}{B(N-1)}L^\prime E\|\Theta_k\|^2 \leq 2L^\prime E\|\Theta_k\|^2
\end{align}

Now we can derive that 
\begin{align*}
    E\|\Theta_{k+1}\|^2 \leq (1-2h_km^\prime+3h_k^2L^\prime)E\|\Theta_{k}\|^2+\frac{2h_kMd}{\beta}
\end{align*}
Fix some $0< h_0 \leq 1 \wedge \frac{m^\prime}{3L^\prime}$, we will show that $\forall k$ 
\begin{align}\label{eq:discrbound}
    E\|\Theta_{k}\|^2\leq E\|\Theta_{0}\|^2+(1\vee\frac{1}{m^\prime})\frac{2Md}{\beta}=M\gamma_0+(1\vee\frac{1}{m^\prime})\frac{2Md}{\beta}
\end{align}
First, it is easy to see that $(1-2h_km^\prime+3h_k^2L^\prime)$ increases with the decrease of $h_k$. Suppose $k^\star$ is the last k that satisfies $(1-2h_km^\prime+3h_k^2L^\prime)\leq 0$, and $\forall k \leq k^\star$, $E\|\Theta_{k}\|^2$ satisfies \eqref{eq:discrbound}.

Then we will see that if $E\|\Theta_{k-1}\|^2 \leq S (k> k^\star)$ and $S > \frac{2Md}{\beta} $, then $E\|\Theta_{k}\|^2 \leq S$.
\begin{align*}
    E\|\Theta_{k}\|^2 \leq (1-2h_km^\prime+3h_k^2L^\prime)S+\frac{2h_kMd}{\beta}\leq S- S(2h_km^\prime-3h_k^2L^\prime)+\frac{2Md}{\beta}<S
\end{align*}
Since $M\gamma_0+(1\vee\frac{1}{m^\prime})\frac{2Md}{\beta}>\frac{2Md}{\beta}$, it is easy to verify that \eqref{eq:discrbound} holds.
\end{proof}

We next prove the following theorem.
\begin{theorem} \label{app:use1}
For the $\Thetab_{\tau}$ in \eqref{eqapp:extendP}, we have
\begin{align}
    E\|\Thetab_{\tau}\|^2 \leq M\gamma_0+\frac{Md}{m^\prime\beta}
\end{align}
\end{theorem}

\begin{proof}
    Let $\mathcal{Y}(\tau)\triangleq \|\Thetab_{\tau}\|^2$. The Itô lemma gives 
    \begin{align*}
        d\mathcal{Y}(\tau)=-2\langle \Thetab_{\tau}, F_{\Thetab}(\Thetab_{\tau}\rangle d\tau +\frac{2Md}{\beta} d\tau+\sqrt{\frac{8}{\beta}}\Thetab_{\tau}^\star d\mathcal{W}_{\tau},
    \end{align*}
    where $\Thetab_{\tau}^\star d\mathcal{W}_{\tau}\triangleq \sum_{i=1}^{Md} \Thetab_{i,{\tau}}d\mathcal{W}_{i,\tau}$ and the $\Thetab_{i,{\tau}},d\mathcal{W}_{i,\tau}$ are the $i$-th components of $\Thetab_{\tau}$ and $\mathcal{W}_{\tau}$. Now this can be rewritten as 
    \begin{align}
        &2m^\prime e^{2m^\prime \tau}\mathcal{Y}(\tau) d\tau+e^{2m^\prime\tau}d\mathcal{Y}(\tau)=\nonumber\\
        &=-2e^{2m^\prime\tau}\langle \Thetab_{\tau}, F_{\Thetab}(\Thetab_{\tau}\rangle d\tau+2m^\prime e^{2m^\prime\tau}\mathcal{Y}(\tau)d\tau+\frac{2Md}{\beta}e^{2m^\prime \tau} d\tau+\sqrt{\frac{8}{\beta}}e^{2m^\prime \tau}\Thetab_{\tau}^\star d\mathcal{W}_{\tau}
    \end{align}
    Since $2m^\prime e^{2m^\prime \tau}\mathcal{Y}(\tau) d\tau+e^{2m^\prime \tau}d\mathcal{Y}(\tau)$ is the total Itô derivative of  $e^{2m^\prime \tau}\mathcal{Y}(\tau)$, we arrive at 
    \begin{align}
        d\left(e^{2m^\prime \tau}\mathcal{Y}(\tau)\right)=-2e^{2m^\prime \tau}\langle \Thetab_{\tau}, F_{\Thetab}(\Thetab_{\tau}\rangle d\tau+2m^\prime e^{2m^\prime \tau}\mathcal{Y}(\tau)d\tau+\frac{2Md}{\beta}e^{2m^\prime \tau} d\tau+\sqrt{\frac{8}{\beta}}e^{2m^\prime \tau}\Thetab_{\tau}^\star d\mathcal{W}_{\tau}
    \end{align}
    With integrating and rearranging, the above equation turns into 
    \begin{align} \label{eq:41}
        \mathcal{Y}(\tau)=&e^{-2m^\prime \tau}\mathcal{Y}(0)-2\int_0^\tau e^{2m^\prime (s-\tau)}\langle \Thetab_{\tau}, F_{\Thetab}(\Thetab_{\tau}\rangle d\tau \nonumber\\
        &+2m^\prime \int_0^\tau e^{2m^\prime (s-\tau)} \mathcal{Y}(s)ds+\frac{Md}{m^\prime \beta}(1-e^{-2m^\prime \tau})+\sqrt{\frac{8}{\beta}}\int_0^\tau e^{2m^\prime (s-\tau)}\Thetab_{s}^\star d\mathcal{W}_{s}ds
    \end{align}
    Now with lemma \ref{lem:lip}, we can write
    \begin{align*}
         -2\int_0^\tau e^{2m^\prime (s-\tau)} \langle \Thetab_{\tau}, F_{\Thetab}(\Thetab_{\tau}\rangle d\tau &\leq  -2\int_0^\tau e^{2m^\prime (s-\tau)} (m^\prime \mathcal{Y}(s))d\tau 
    \end{align*}
   Then, with \eqref{eq:41} we arrive at
   \begin{align*}
       \|\Thetab_{\tau}\|^2 \leq e^{-2m^\prime \tau}\|\Thetab_{0}\|^2+\frac{Md}{\beta m^\prime}(1-e^{-2m^\prime\tau})+\sqrt{\frac{8}{\beta}}\int_0^\tau e^{2m^\prime (s-\tau)}\Thetab_{s}^\star d\mathcal{W}_{s}ds
   \end{align*}
   Taking expectations and using the martingale property of the Itô integral, we can derive the following result according to the independence of the the initial particles ${\thetab_0^{(i)}}$:
   \begin{align*}
       E\|\Thetab_{\tau}\|^2 &\leq e^{-2m^\prime \tau}E \|\Thetab_{0}\|^2+\frac{Md}{\beta m^\prime}(1-e^{-2m^\prime\tau})\\
       &\leq e^{-2m^\prime \tau}M \gamma_0+\frac{Md}{\beta m^\prime}(1-e^{-2m^\prime\tau})
   \end{align*}
   This finishes the proof.
\end{proof}

It is easy to get the following corollary with the exchangeability of the particles

\begin{corollary}\label{corollary:1}
For the particles ${\thetab_\tau^{(i)}}$ in \eqref{theo:particle_app}, we have 
\begin{align}
    E\|\thetab_{\tau}\|^2 \leq \gamma_0+\frac{d}{m^\prime\beta}
\end{align}
\end{corollary}



Similarly, now we can provide a uniform bound for the $\bar{\thetab}_{\tau}$ in \eqref{eq:nonlinear_sde}.

\begin{theorem} \label{thm:continousbound2}
    For the $\bar{\thetab}_{\tau}^{(i)}$ in \eqref{eqnew:nonlinear_sde}, we have 
    \begin{align}
        E\|\bar{\thetab}_{\tau}^{(i)}\|^2 \leq \gamma_0+\frac{d}{m^\prime\beta}
    \end{align}
\end{theorem}

\begin{proof}
Let $\mathcal{Y}(\tau)\triangleq \|\bar{\thetab}_{\tau}\|^2$, where the $\bar{\thetab}_{\tau}$ is from \eqref{eq:nonlinear_sde}. The Itô lemma gives 
    \begin{align*}
        d\mathcal{Y}(\tau)=2\langle \bar{\thetab}_{\tau}, -\beta^{-1}F(\bar{\thetab}_{\tau}) -E_{Y \sim\nu_{\tau}}K(\bar{\thetab}_{\tau}-Y)F(Y) +\nabla K*\nu_{\tau}(\bar{\thetab}_{\tau})\rangle d\tau +\frac{2d}{\beta} d\tau+\sqrt{\frac{8}{\beta}}\bar{\thetab}_{\tau}^\star d\mathcal{W}_{\tau},
    \end{align*}
    where $\bar{\thetab}_{\tau}^\star d\mathcal{W}_{\tau}\triangleq \sum_{i=1}^{d} \bar{\thetab}_{i,{\tau}}d\mathcal{W}_{i,\tau}$ and the $\bar{\thetab}_{i,{\tau}},d\mathcal{W}_{i,\tau}$ are the $i$-th components of $\bar{\thetab}_{i,{\tau}}$ and $\mathcal{W}_{\tau}$. This can be rewritten as 
    \begin{align}
        2m^\prime e^{2m^\prime \tau}\mathcal{Y}(\tau) d\tau+e^{2m^\prime\tau}d\mathcal{Y}(\tau)=
        &2e^{2m^\prime\tau}\langle \bar{\thetab}_{\tau}, -\beta^{-1}F(\bar{\thetab}_{\tau}) -E_{Y \sim\nu_{\tau}}K(\bar{\thetab}_{\tau}-Y)F(Y) +\nabla K*\nu_{\tau}(\bar{\thetab}_{\tau})\rangle d\tau \nonumber\\
        &+2m^\prime e^{2m^\prime\tau}\mathcal{Y}(\tau)d\tau+\frac{2d}{\beta}e^{2m^\prime \tau} d\tau+e^{2m^\prime \tau}\sqrt{\frac{8}{\beta}}\bar{\thetab}_{\tau}^\star d\mathcal{W}_{\tau}
    \end{align}
    Since $2m^\prime e^{2m^\prime \tau}\mathcal{Y}(\tau) d\tau+e^{2m^\prime \tau}d\mathcal{Y}(\tau)$ is the total Itô derivative of  $e^{2m^\prime \tau}\mathcal{Y}(\tau)$, we arrive at 
    \begin{align}
        d\left(e^{2m^\prime \tau}\mathcal{Y}(\tau)\right)&=2e^{2m^\prime\tau}\langle \bar{\thetab}_{\tau}, -\beta^{-1}F(\bar{\thetab}_{\tau}) -E_{Y \sim\nu_{\tau}}K(\bar{\thetab}_{\tau}-Y)F(Y) +\nabla K*\nu_{\tau}(\bar{\thetab}_{\tau})\rangle d\tau \nonumber\\
        &+2m^\prime e^{2m^\prime \tau}\mathcal{Y}(\tau)d\tau+\frac{2d}{\beta}e^{2m^\prime \tau} d\tau+e^{2m^\prime \tau}\sqrt{\frac{8}{\beta}}\bar{\thetab}_{\tau}^\star d\mathcal{W}_{\tau}
    \end{align}
    With integrating and rearranging, the above equation turns into 
    \begin{align} \label{eq:41}
        \mathcal{Y}(\tau)=&e^{-2m^\prime \tau}\mathcal{Y}(0)+2\int_0^\tau e^{2m^\prime (s-\tau)}\langle \bar{\thetab}_{s}, -\beta^{-1}F(\bar{\thetab}_{s}) -E_{Y \sim\nu_{\tau}}K(\bar{\thetab}_{s}-Y)F(Y) +\nabla K*\nu_{s}(\bar{\thetab}_{s})\rangle ds \nonumber\\
        &+2m^\prime \int_0^\tau e^{2m^\prime (s-\tau)} \mathcal{Y}(s)ds+\frac{d}{m^\prime \beta}(1-e^{-2m^\prime \tau})+\int_{0}^\tau e^{2m^\prime (s-\tau)}\sqrt{\frac{8}{\beta}}\bar{\thetab}_{s}^\star d\mathcal{W}_{s}ds
    \end{align}
    With lemma \ref{lem:lip}, we can write
    \begin{align*}
         2\int_0^\tau e^{2m^\prime (s-\tau)} \langle \bar{\thetab}_{s}, -\beta^{-1}F(\bar{\thetab}_{s}) \rangle ds &\leq  -2\int_0^\tau e^{2m^\prime (s-\tau)} (\beta^{-1}m_F \mathcal{Y}(s))ds
    \end{align*}
    \begin{align*}
        2\int_0^\tau e^{2m^\prime (s-\tau)} \langle \bar{\thetab}_{s}, E_{Y \sim\nu_{s}}K(\bar{\thetab}_{s}-Y)F(Y) \rangle ds &\leq  2\int_0^\tau e^{2m^\prime (s-\tau)} \left(L_F (E_{Y \sim\nu_{s}}\|Y\|)\|\bar{\thetab}_{s}\|\right)ds
    \end{align*}
    \begin{align*}
        2\int_0^\tau e^{2m^\prime (s-\tau)} \langle \bar{\thetab}_{s}, \nabla K*\nu_{s}(\bar{\thetab}_{s}) \rangle ds &\leq  2\int_0^\tau e^{2m^\prime (s-\tau)} \left( \frac{2}{\eta^2}(E_{Y \sim\nu_{s}}\|Y\|)\|\bar{\thetab}_{s}\|+\frac{2}{\eta^2}\mathcal{Y}(s)\right)ds
    \end{align*}
   Then, with \eqref{eq:41} we arrive at
   \begin{align*}
       \mathcal{Y}(\tau) &\leq e^{-2m^\prime t}\mathcal{Y}(0)+\frac{d}{\beta m^\prime}(1-e^{-2m^\prime\tau})-2\int_0^\tau e^{2m^\prime (s-\tau)} (\beta^{-1}m_F \mathcal{Y}(s))ds \\
       &+2\int_0^\tau e^{2m^\prime (s-\tau)} \left(L_F (E_{Y \sim\nu_{s}}\|Y\|)\|\bar{\thetab}_{s}\|\right)ds+2\int_0^\tau e^{2m^\prime (s-\tau)} \left( \frac{2}{\eta^2}(E_{Y \sim\nu_{s}}\|Y\|)\|\bar{\thetab}_{s}\|+\frac{2}{\eta^2}\mathcal{Y}(s)\right)ds\\
       &+2m^\prime \int_0^\tau e^{2m^\prime (s-\tau)} \mathcal{Y}(s)ds+\int_{0}^\tau e^{2m^\prime (s-\tau)}\sqrt{\frac{8}{\beta}}\bar{\thetab}_{s}^\star d\mathcal{W}_{s}ds
   \end{align*}
   Taking expectations and using the martingale property of the Itô integral, we can derive the following result:
   \begin{align*}
       E\|\bar{\thetab}_{\tau}\|^2 &\leq e^{-2m^\prime t}E \|\bar{\thetab}_{0}\|^2+\frac{d}{\beta m^\prime}(1-e^{-2m^\prime\tau})-2\int_0^\tau e^{2m^\prime (s-\tau)} (\beta^{-1}m_F E\|\bar{\thetab}_{s}\|^2)ds\\
       &+2\int_0^\tau e^{2m^\prime (s-\tau)} \left(L_F (E_{Y \sim\nu_{s}}\|Y\|)E\|\bar{\thetab}_{s}\|\right)ds+2\int_0^\tau e^{2m^\prime (s-\tau)} \left( \frac{2}{\eta^2}(E_{Y \sim\nu_{s}}\|Y\|)E\|\bar{\thetab}_{s}\|+\frac{2}{\eta^2}E\|\bar{\thetab}_{s}\|^2\right)ds\\
       &+2m^\prime \int_0^\tau e^{2m^\prime (s-\tau)} E\|\bar{\thetab}_{s}\|^2ds+\int_{0}^\tau e^{2m^\prime (s-\tau)}\sqrt{\frac{8}{\beta}}\bar{\thetab}_{s}^\star d\mathcal{W}_{s}ds
   \end{align*}
   With $\mathcal{L}(\bar{\thetab}_{\tau}) =\nu_{\tau}\mathrm{d}\thetab$ and $m^\prime\triangleq\beta^{-1}m_F-L_F-\frac{4}{\eta^2}$, we can derive the following result with Cauchy–Schwarz. inequality,
   \begin{align*}
       E\|\bar{\thetab}_{\tau}\|^2 &\leq 
       e^{-2m^\prime t}E \|\bar{\thetab}_{0}\|^2+\frac{d}{\beta m^\prime}(1-e^{-2m^\prime\tau})-2\int_0^\tau e^{2m^\prime (s-\tau)} (\beta^{-1}m_F E\|\bar{\thetab}_{s}\|^2)ds\\
       &+2\int_0^\tau e^{2m^\prime (s-\tau)} \left(L_F E\|\bar{\thetab}_{s}\|^2\right)ds+2\int_0^\tau e^{2m^\prime (s-\tau)} \left( \frac{2}{\eta^2}E\|\bar{\thetab}_{s}\|^2+\frac{2}{\eta^2}E\|\bar{\thetab}_{s}\|^2\right)ds\\
       &+2m^\prime \int_0^\tau e^{2m^\prime (s-\tau)} E\|\bar{\thetab}_{s}\|^2ds \\
       &\leq e^{-2m^\prime t}E \|\bar{\thetab}_{0}\|^2+\frac{d}{\beta m^\prime}(1-e^{-2m^\prime\tau})\\
       & \leq \gamma_0+\frac{d}{\beta m^\prime}
   \end{align*}
   This completes the proof.
\end{proof}


\section{Non-Asymptotic Convergence Analysis: the Nonconvex Case}\label{sec:generalU}
Since the non-convex case is much more complicated than the convex case, we reply on different assumptions and adopt another distance metric, denoted as $\tilde{\mathcal{B}}$, to characterize the convergence behavior of SPOS under the non-convex case. Note in this section, we give the preliminary convergence results of SPOS under the non-convex setting. A more complete version will be interesting future work. 

Specifically, define $\tilde{\mathcal{B}}(\mu,\nu)$ as $\tilde{\mathcal{B}}(\mu,\nu) \triangleq \left|\mathbb{E}_{\thetab \sim \mu}[f(\thetab)] - \mathbb{E}_{\thetab \sim \nu}[f(\thetab)]\right|$ for a known $L_f$-continuous function $f$ satisfying Assumption~\ref{assnew:ass2} below. Note such metric has also been adopted in \cite{VollmerZT:arxiv15,ChenDC:NIPS15}. Our analysis considers ($T,M,h_k$) as variables in $\tilde{\mathcal{B}}$. In addition, we use $\{\hat{\theta}_k^{(i)}\}_{i=1}^M$ to denote the particles when full gradients are adopted in \eqref{eq:particle_num}. The distribution of the particles is denoted as $\hat{\mu}_k$. 

Our high-level idea of bounding $\tilde{\mathcal{B}}(\mu_T,\nu_{\infty})$ is to decompose it as follows:
{\begin{align}\label{decom:decom1}
	\tilde{\mathcal{B}}(\mu_T,\nu_{\infty}) \leq \tilde{\mathcal{B}}(\mu_T,\hat{\mu}_T)+\tilde{\mathcal{B}}(\hat{\mu}_T,\hat{\mu}_{\infty})+\tilde{\mathcal{B}}(\hat{\mu}_{\infty},\rho_{\infty})+\tilde{\mathcal{B}}(\rho_{\infty},\nu_{\infty})
	\end{align}}
Similarly, our idea is to concatenate the particles at each time into a single vector representation, {\it i.e.} defining the new parameter at time $\tau$ as $\Thetab_{\tau}\triangleq [\thetab_{\tau}^{(1)}, \cdots, \thetab_{\tau}^{(M)}] \in \mathbb{R}^{Md}$. Consequently, the nonlinear PDE system \eqref{eq:particle} can be turned into an SDE ,which means $\Thetab_{\tau}$ is driven by the following SDE:\vspace{-0.0cm}
\begin{align}\label{eqnew:extendP}
\mathrm{d}\Thetab_{\tau} = -F_{\Thetab}(\Thetab_{\tau})\mathrm{d}\tau + \sqrt{2\beta^{-1}}\mathrm{d}\mathcal{W}_{\tau}^{(Md)}~,
\end{align}
where $F_{\Thetab}(\Thetab_{\tau}) \triangleq [\beta^{-1}F(\thetab_{\tau}^{(1)})-\frac{1}{M}\sum_{j=1}^M\nabla K(\thetab_{\tau}^{(1)}-\thetab_{\tau}^{(j)})+\frac{1}{M}\sum_{j=1}^{M}K(\thetab_{\tau}^{(1)} - \thetab_{\tau}^{(j)})F(\thetab_{\tau}^{(j)}), \cdots, \beta^{-1}F(\thetab_{\tau}^{(M)})-\frac{1}{M}\sum_{j=1}^M\nabla K(\thetab_{\tau}^{(M)}-\thetab_{\tau}^{(j)})+\frac{1}{M}\sum_{j=1}^{M}K(\thetab_{\tau}^{(M)} - \thetab_{\tau}^{(j)})F(\thetab_{\tau}^{(j)})]$ is a vector function $\mathbb{R}^{Md}\rightarrow \mathbb{R}^{Md}$, and $\mathcal{W}_{\tau}^{(Md)}$ is Brownian motion of dimension $M\times d$. Similarly, we can define $\hat{\Theta}_k\triangleq [\hat{\theta}_{k}^{(1)}, \cdots, \hat{\theta}_{k}^{(M)}] \in \mathbb{R}^{Md}$ for the full-gradient case. Hence, it can be  seen that through such a decomposition in \eqref{decom:decom1}, the bound related to a nonlinear PDE system \eqref{eq:particle} reduces to that of an SDE. The second term $\tilde{\mathcal{B}}(\hat{\mu}_T,\hat{\mu}_{\infty})$ reflexes the geometric ergodicity of a dynamic system with a numerical method. It is known that even if a dynamic system has an exponential convergence rate to its equilibrium, its corresponding numerical method might not. Our bound for $\tilde{\mathcal{B}}(\hat{\mu}_T,\hat{\mu}_{\infty})$ is essentially a specification of the result of \cite{MattinglySH:SPA02}, which has also been applied by \cite{XuCZG:arxiv17}. The third term $\tilde{\mathcal{B}}(\hat{\mu}_{\infty},\rho_{\infty})$ reflects the numerical error of an SDE, which has been studied in related literature such as \cite{ChenDC:NIPS15}. 
To this end, we adopt standard assumptions used in the analysis of SDEs \cite{VollmerZT:arxiv15,ChenDC:NIPS15}, rephrased in Assumption~\ref{assnew:ass2}.
\begin{assumption}\label{assnew:ass2}
	For the SDE \eqref{eqnew:extendP} and a Lipschitz function $f$, let $\psi$ be the solution functional of the Poisson equation: $\mathcal{G}\psi(\hat{\Theta}_k) = \frac{1}{M}\sum_{i=1}^M f(\hat{\theta}_k^{(i)}) - \mathbb{E}_{\thetab \sim p(\thetab|\mathcal{D})}[f(\thetab)]$, where $\mathcal{G}$ denotes the infinite generator of the SDE \eqref{eqnew:extendP}. Assume $\psi$ and its up to 4th-order derivatives, $\mathcal{D}^k \psi$, are bounded by a function $\mathcal{V}$, {\it i.e.}, 
	$\|\mathcal{D}^k \psi\| \leq H_k\mathcal{V}^{p_k}$ for $k=(0, 1, 2, 3, 4)$, $H_k, p_k > 0$. Furthermore, 
	the expectation of $\mathcal{V}$ on $\{\Thetab_{\tau}\}$ is bounded: $\sup_l \mathbb{E}\mathcal{V}^p(\Thetab_{\tau}) < \infty$, and $\mathcal{V}$ is smooth such that 
	{\small$ \sup_{s \in (0, 1)} \mathcal{V}^p\left(s\Thetab + \left(1-s\right)\Thetab^\prime\right) \leq H\left(\mathcal{V}^p\left(\Thetab\right) + \mathcal{V}^p\left(\Thetab^\prime\right)\right)$, $\forall \Thetab, \Thetab^\prime, p \leq \max\{2p_k\}$} for $H > 0$.
\end{assumption}
\begin{assumption}\label{assnew:ass3}
	$\RN{1})$ $F$, $K$ and $\nabla K$ are $L_F$, $L_K$ and $L_{\nabla k}$ Lipschitz; $\RN{2})$ $F$ satisfies the dissipative property, {\it i.e.}, $\langle F(\thetab), \thetab\rangle \geq m\left\|\thetab\right\|^2 - b$ for some $m,b>0$; $\RN{3})$ Remark~\ref{rem:rem1} applies to the nonconvex setting, {\it i.e.} {\small$\sup_{\|f\|_{\tiny Lip}\leq 1}\left|\mathbb{E}_{\thetab \sim \mu_{\infty}}[f(\thetab)] - \mathbb{E}_{\thetab \sim \nu_{\infty}}[f(\thetab)]\right| =\mathcal{W}_1(\rho_{\infty},\nu_{\infty})=O(M^{-1/2})$}.
\end{assumption}
\begin{remark}
	Assumption~\ref{assnew:ass2} is necessary to control the gap between a numerical solution and the exact solution of an SDE. Specifically, it is used to bound the $\tilde{\mathcal{B}}(\hat{\mu}_{\infty},\rho_{\infty})$ term and the $\tilde{\mathcal{B}}(\mu_T,\hat{\mu}_T)$ term above. Purely relying on the dissipative assumption in Assumption~\ref{assnew:ass3} as in non-convex optimization with SG-MCMC \cite{RaginskyRT:COLT17,XuCZG:arxiv17} would induce a bound increasing linearly w.r.t.\! time $\tau$. Thus it is not suitable for our goal. Finally, $\RN{3})$ in Assumption~\ref{assnew:ass3} is a mild condition and reasonable because we expect particles to be able to approximate all distributions equally well in the asymptotic limit of $t\rightarrow \infty$ by ergodicity due to the injected noise. How to remove/replace this assumption is an interesting future work.
\end{remark}
\vspace{-0.1cm}
Based on the assumptions above, the bounds for $\tilde{\mathcal{B}}(\hat{\mu}_T,\hat{\mu}_{\infty})$ and $\tilde{\mathcal{B}}(\hat{\mu}_{\infty},\rho_{\infty})$ are summarized below.
\vspace{-0.2cm}
\begin{theorem}\label{theonew:b21_b22}
	Under Assumption~\ref{assnew:ass2}--\ref{assnew:ass3}, if we set the stepsize $h_k=h$, we can  have the following results:
	\begin{align}
	&\tilde{\mathcal{B}}(\hat{\mu}_T,\hat{\mu}_{\infty}) \leq C_2\varsigma\sigma^{-Md/2}(1+\varsigma e^{m_{\Thetab}h})\exp\left(-2m_{\Thetab}Th \sigma^{Md}/\log(\varsigma)\right), \nonumber\\ ~~
	\text{and } &\tilde{\mathcal{B}}(\hat{\mu}_{\infty},\rho_{\infty})\leq C_3h/\beta,\label{eq:Bt12}
	\end{align}
	where $\varsigma = 2L_{\Thetab}(Mb\beta + m_{\Thetab}\beta + Md)/m_{\Thetab}$, $L_{\Thetab}=\sqrt{2} \beta^{-1}L_F+l^{\prime}$, $m_{\Thetab}=\beta^{-1}m-m^{\prime}$, and $(\sigma,C_2,C_3,l^{\prime},m^{\prime})$ are some positive constants independent of (T, M, h) and $\sigma\in(0,1)$ 
\end{theorem}
\begin{remark}
	In order to make the $\tilde{\mathcal{B}}(\hat{\mu}_T,\hat{\mu}_{\infty})$ term asymptotically decrease to zero, the number of running iteration $T$ should increase at a rate faster enough to compensate the effect of increasing $M$. We believe there is room for improving this bound, which is an interesting future work.
\end{remark}

Next we bound the $\tilde{\mathcal{B}}(\mu_T,\hat{\mu}_T)$ term related to stochastic gradients. By adapting results from analysis of diffusion processes \cite{XuCZG:arxiv17}, $\tilde{\mathcal{B}}(\mu_T,\hat{\mu}_T)$ can be bounded with Theorem~\ref{theonew:sg_bound}.

\begin{theorem}\label{theonew:sg_bound}
	Under Assumptions~\ref{assnew:ass2}--\ref{assnew:ass3}, if we set $B_k=B$ and $h_k=h$, $\tilde{\mathcal{B}}(\mu_T,\hat{\mu}_T)$ is bounded as
	\begin{align*}
	\tilde{\mathcal{B}}(\mu_T,\hat{\mu}_T) &\leq C_5Th(L_{\Thetab}\Gamma^{\prime}+MC_4)\sqrt{ (6+2\Gamma^{\prime})\beta/{(BM)}},
	\end{align*}
	where $\Gamma^{\prime}=2(1+1/m_{\Theta})(Mb+2M^2C_4^2+Md/\beta)$ and , $(C_4,C_5)$ is some positive constant independent of (T, M, h) 
\end{theorem}
Finally, by combining the results from Theorem~\ref{theonew:b21_b22}, \ref{theonew:sg_bound} and $\RN{3})$ in Assumption~\ref{assnew:ass3}, we arrive at a bound for our target $\tilde{\mathcal{B}}(\mu_T,\nu_{\infty})$, summarized in Theorem~\ref{theonew:main}.

\begin{theorem}\label{theonew:main}
	Under Assumptions~\ref{assnew:ass2}--\ref{assnew:ass3}, there exist some positive constants $(C_2, C_3, C_4,C_5,C_6)$ such that:
	\begin{align*}
	\tilde{\mathcal{B}}(\mu_T,\nu_{\infty}) \leq& C_2\varsigma\sigma^{-Md/2}(1+\varsigma e^{m_{\Thetab}h})
	\times \exp\left(-2m_{\Thetab}Th \sigma^{Md}/\log(\varsigma)\right) +C_3h/\beta \\
	&+ C_5Th(L_{\Thetab}\Gamma^{\prime}+MC_4)\left( (6+2\Gamma^{\prime})\beta/{(BM)} \right)^{1/2}+C_6 / \sqrt{M},
	\end{align*}
	where $\sigma$, $\varsigma$ and $\Gamma^{\prime}$ are the same as those in Theorem~\ref{theonew:b21_b22}--\ref{theonew:sg_bound}.
\end{theorem}

\section{Proof of Theorem~\ref{theonew:b21_b22}}
\begin{proof}[Proof of Theorem~\ref{theonew:b21_b22}]
	Our conclusion for $\tilde{\mathcal{B}}(\hat{\mu}_T,\hat{\mu}_{\infty})$ is essentially a specification of the result in \cite{MattinglySH:SPA02}, which has also been applied in \cite{XuCZG:arxiv17}.
	
	Specifically, we rely on the following lemma, which is essentially Theorem~7.3 in \cite{MattinglySH:SPA02} and Lemma~C.3 in \cite{XuCZG:arxiv17}. Consider the following SDE (eq.\ref{eqnew:extendP}):
	\begin{align*}
	\mathrm{d}\Thetab_{\tau} = -F_{\Thetab}(\Thetab_{\tau})\mathrm{d}\tau + \sqrt{2\beta^{-1}}\mathrm{d}\mathcal{W}_{\tau}^{(Md)}~ 
	\end{align*}
	As mentioned in Section 5, we denote the distribution of $\Thetab_{\tau}$ as $\rho^{\Thetab}_{\tau}$, and define $\hat{\Theta}_k\triangleq [\hat{\theta}_{k}^{(1)}, \cdots, \hat{\theta}_{k}^{(M)}] \in \mathbb{R}^{Md}$, which is actually the numerical solution of \eqref{eqnew:extendP} using full gradient with Euler method. Denote the distribution of $\hat{\Theta}_k$ as $\hat{\mu}^{\Theta}_k$.
	\begin{lemma}\label{lem:mattingly}
		Let $F_{\Thetab}$ be Lipschitz-continuous with constant $L_{\Thetab}$, and satisfy the dissipative property that $\langle F_{\Thetab}(\Thetab), \Thetab\rangle \geq m_{\Thetab}\left\|\Thetab\right\|^2 - b_{\Thetab}$. Define $V_{\Thetab}(\Thetab) = C_0 + L_{\Thetab}/2\|\Thetab\|^2$. The Euler method for \eqref{eqnew:extendP}  has a unique invariant measure $\hat{\mu}^{\Theta}_{\infty}$, and for all test function $f_{\Thetab}$ such that $|f_{\Thetab}| \leq V_{\Thetab}(\Thetab)$, we have
		\begin{align*}
		&\left|\mathbb{E}[f_{\Thetab}(\hat{\Theta}_k))] - \mathbb{E}_{\hat{\Theta}_{\infty} \sim \hat{\mu}^{\Theta}_{\infty} }[f(\hat{\Theta}_{\infty})]\right| \\
		\leq& C\kappa\rho^{-Md/2}(1 + \kappa e^{m_{\Thetab}h})\exp\left(-\frac{2m_{\Thetab}kh\rho^{Md}}{\log(\kappa)}\right)~,
		\end{align*}
		where $\rho\in (0, 1)$, $C > 0$ are positive constants, and $\kappa = 2L_{\Thetab}(b_{\Thetab}\beta + m_{\Thetab}\beta + Md)/m_{\Thetab}$.
	\end{lemma}

	Now we define $f_{\Thetab}:\mathbb{R}^{Md}\rightarrow \mathbb{R}$ as $f_{\Thetab}(\Thetab) =\frac{1}{M}\sum_{i}^{M}f(\thetab^{(i)})$, where $f:\mathbb{R}^{d}\rightarrow \mathbb{R}$ is a $L_f$-Lipschitz function satisfying our Assumption~\ref{assnew:ass2}, and $\Thetab \triangleq [{\thetab}^{(1)}, \cdots, {\thetab}^{(M)}]$.  Similar to the proof of Lemma~\ref{lem:Gronwall}, we can find that $f_{\Thetab}:\mathbb{R}^{Md}$ is a $L_f/\sqrt{M}$-Lipschitz function. Furthermore, according to Lemma~\ref{lem:lip}, it is easily check that $F_{\Thetab}$ is $L_{\Thetab}$-Lipschitz where $L_{\Thetab} = \sqrt{2}\beta^{-1}L_F+l^{\prime}$. Hence, when $\beta$ is small enough, we have $L_f/\sqrt{M}\leq \sqrt{2}\beta^{-1}L_F+l^{\prime}$. As a result, we can set the $C_0$ large enough to force $f_{\Thetab}$ to satisfy the condition in Lemma~\ref{lem:mattingly} that $|f_{\Thetab}| \leq V_{\Thetab}(\Thetab)$. According to the exchangeability of the particle system $\{\hat{\theta}_{k}^{(i)}\}$ and Lemma~\ref{lem:lip}, we can bound $\tilde{\mathcal{B}}(\hat{\mu}_T,\hat{\mu}_{\infty})$ as 
	\begin{align*}
	&\tilde{\mathcal{B}}(\hat{\mu}_T,\hat{\mu}_{\infty}) \leq \left|\mathbb{E}[f_{\Thetab}(\hat{\Theta}_T))] - \mathbb{E}_{\hat{\Theta}_{\infty} \sim \hat{\mu}^{\Theta}_{\infty} }[f(\hat{\Theta}_{\infty})]\right|\\
	\leq& C_2\varsigma\sigma^{-Md/2}(1+\varsigma e^{m_{\Thetab}h})\exp\left(-2m_{\Thetab}Th \sigma^{Md}/\log(\varsigma)\right) ~~
	\end{align*}
	where $\varsigma = 2L_{\Thetab}(Mb\beta + m_{\Thetab}\beta + Md)/m_{\Thetab}$, $L_{\Thetab}=\sqrt{2} \beta^{-1}L_F+l^{\prime}$, $m_{\Thetab}=\beta^{-1}m-m^{\prime}$, and $(\sigma,C_2,l^{\prime},m^{\prime})$ are some positive constants independent of (T, M, h) and $\sigma\in(0,1)$.

	%
	To prove the bound for $\tilde{\mathcal{B}}(\hat{\mu}_{\infty},\rho_{\infty})$, 
	since $\hat{\Theta}_{k}=(\hat{\theta}_{k}^{(1)}, \cdots, \hat{\theta}_{k}^{(M)})$ can be considered as a solution to the SDE \eqref{eqnew:extendP}, standard results from linear FP equation can be applied. Specifically, for the $\tilde{\mathcal{B}}(\hat{\mu}_{\infty},\rho_{\infty})$ term, we rely on the following lemma adapted from Lemma~C.4 in \cite{XuCZG:arxiv17,ChenDC:NIPS15}, which is essentially  the result of \cite{ChenDC:NIPS15} when taking $T\rightarrow \infty$.
	\begin{lemma}\label{lem:particlemu}
		Under the same assumption as in Lemma~\ref{lem:mattingly}, for the Lipschitz-continuous function $f_{\Thetab}(\Thetab) =\frac{1}{M}\sum_{i}^{M}f(\thetab^{(i)})$ mentioned above, the following bound is satisfied for some positive constant $C$:
		\begin{align*}
		\left|\frac{1}{T}\sum_{k=1}^{T-1}\mathbb{E}[f_{\Thetab}(\hat{\Theta}_{k})] - \mathbb{E}_{{\Thetab}_{\infty} \sim {\rho}^{\Thetab}_{\infty} }[f({\Thetab}_{\infty})]\right| \leq C(\frac{h}{\beta}+\frac{\beta}{Th})~.
		\end{align*}
	\end{lemma}
	The uniqueness of invariant measure of the Euler method from Lemma~\ref{lem:mattingly} implies the numerical solution $\hat{\Theta}_{k}$ to be ergodic. Then similar to the proof of Lemma 4.2 in \cite{XuCZG:arxiv17}, we consider the case where $T \rightarrow \infty$. Taking average over the $\{\hat{\Theta}_{k}\}_{k=0}^{T-1}$, we have 
	\begin{align*}
	\mathbb{E}_{\hat{\Theta}_{\infty} \sim \hat{\mu}^{\Thetab}_{\infty}}[f_{\Thetab}(\hat{\Theta}_{\infty})]=\lim_{T \rightarrow \infty}	\frac{1}{T}\sum_{k=1}^{T}\mathbb{E}[f_{\Thetab}(\hat{\Theta}_{k})]
	\end{align*}
	Now according to the exchangeability of the particle system $\{\hat{\theta}_{k}^{(i)}\}$ and $\{{\thetab}_{\tau}^{(i)}\}$, we can bound the $\tilde{\mathcal{B}}(\hat{\mu}_{\infty},\rho_{\infty})$ as :	
	\begin{align*}
	&\tilde{\mathcal{B}}(\hat{\mu}_{\infty},\rho_{\infty}) \leq \left|\mathbb{E}_{\hat{\Theta}_{\infty} \sim \hat{\mu}^{\Thetab}_{\infty} }[f_{\Thetab}(\hat{\Theta}_{\infty})] - \mathbb{E}_{{\Thetab}_{\infty} \sim {\rho}^{\Thetab}_{\infty} }[f_{\Thetab}({\Thetab}_{\infty})]\right| \\
	& \leq  C_3h/\beta~,
	\end{align*}
	where $C_3$ are some positive constant.
\end{proof}

\section{Proof of Theorem~\ref{theonew:sg_bound}}

\begin{proof}[Proof of Theorem~\ref{theonew:sg_bound}]
	Adopting the same notation used in the proof of Theorem~\ref{theonew:num_error}, we define ${\Theta}_k \triangleq [{\theta}_{k}^{(1)}, \cdots, {\theta}_{k}^{(M)}]$ and $G_{\mathcal{I}_k}^{\Theta} \triangleq \frac{N}{B_k}\sum_{q\in\mathcal{I}_k}F_{(q)\Thetab}(\Theta_k)$. We denote the distribution of ${\Theta}_k$ as $\mu^{\Theta}_k$, where
	\begin{align*}
	\Theta_{k+1} = \Theta_{k} - \beta^{-1}G_{\mathcal{I}_k}^{\Theta}h_k + \sqrt{2\beta^{-1}h_k}\Xi_{k}~.
	\end{align*}
	We firstly derive a bound for ${W}_2(\mu^{\Theta}_k,\hat{\mu}^{\Thetab}_{k})$ (the definition of $\hat{\mu}^{\Thetab}_{k}$ is given in the last section).
	According to the proof of Lemma 4.4 in  \cite{XuCZG:arxiv17}
	\begin{align*}
	{W}_2(\mu^{\Theta}_k,\hat{\mu}^{\Thetab}_{k})\leq kh(L_{\Thetab}\Gamma^{\prime}+MC_4)\left( (6+2\Gamma^{\prime})\beta/{B} \right)^{1/2}
	\end{align*}
	where $\Gamma^{\prime}=2(1+1/m_{\Theta})(Mb+2M^2C_4^2+Md/\beta)$ and $C_4$ is some positive constant independent of (T, M, h). 
	By applying the facts that ${W}_1(\mu^{\Theta}_k,\hat{\mu}^{\Thetab}_{k})\leq {W}_2(\mu^{\Theta}_k,\hat{\mu}^{\Thetab}_{k})$  and ${W}_1(\mu_k,\hat{\mu}_{k})\leq \frac{1}{\sqrt{M}} {W}_1(\mu^{\Theta}_k,\hat{\mu}^{\Thetab}_{k})$ (see the proof of Lemma~\ref{lem:W}, similar result holds here), 
	we get 
	\begin{align*}
	{W}_1(\mu_T,\hat{\mu}_{T}) \leq Th(L_{\Thetab}\Gamma^{\prime}+MC_4)\left( (6+2\Gamma^{\prime})\beta/{(BM)} \right)^{1/2}.
	\end{align*}
	Since the definitions of $\mathcal{W}_1(\mu,\nu)$ and $\tilde{\mathcal{B}}(\mu,\nu)$ are given as:
	\begin{align*}
	&{W}_1(\mu,\nu) \triangleq \sup_{\|g\|_{lip}\leq 1}\left|\mathbb{E}_{\thetab \sim \mu}[g(\thetab)] - \mathbb{E}_{\thetab \sim \nu}[g(\thetab)]\right|\\
	&\tilde{\mathcal{B}}(\mu,\nu) \triangleq \left|\mathbb{E}_{\thetab \sim \mu}[f(\thetab)] - \mathbb{E}_{\thetab \sim \nu}[f(\thetab)]\right|~,
	\end{align*}
	it is easily seen that $\tilde{\mathcal{B}}(\mu_T,\hat{\mu}_{K})\leq L_f{W}_1(\mu_T,\hat{\mu}_{T})$, which finishes the proof.
\end{proof}

\section{Discussion on the complexity of the proposed SPOS}\label{app:complexity}
The complexity of an algorithm  mainly refers to its time complexity (corresponding to the number of iterations in our method {\it i.e. T}) and space complexity (corresponding to the number of particles used in our method {\it i.e. M}). Hence the complexity of our method can be well explored with our work, since our non-asymptotic convergence theory is developed w.r.t.\! both the number of particles {\it i.e. M} and iterations {\it i.e. T}. Their relationship (tradeoff) is discussed { further} in the experiments. Moreover, by comparing \eqref{eq:particle_num} with \eqref{eq:svgd_update} , one can easily find that our space complexity is exactly the same as SVGD and our computational time in each iteration is almost the same as SVGD with an extra addition operation. However, it is worth noting that our method have much better performance in practice with no ``pitfall'' verified by both our theory and experiments.

\section{Comparison with Related Work}\label{app:related}

Firstly, our proposed framework SPOS is different from the recently proposed particle-optimization sampling framework \cite{ChenZWLC:tech18}, in the sense that we solve the nonlinear PDE \eqref{eq:unifiednew} stochastically. For example they deterministically solve the equation in \eqref{eq:unifiednew} $\partial \nu_{\tau} = \beta^{-1}\nabla_{\thetab}\cdot\nabla_{\thetab}\nu_{\tau}$ approximately using blob method adopted from \cite{CarrilloCP:arxiv17}.

Secondly, our method is also distinguishable to existing work on granular media equations such as \cite{DurmusEGZ:arxiv18}. The work about the granular media equations mainly focuses on the following PDE:
\begin{align}\label{eq:gran}
\partial_{\tau} \nu_{\tau} = \nabla_{\thetab}\cdot \left(\nu_{\tau}\beta^{-1}F(\thetab) + \nu_{\tau}\left(\nabla K*\nu_{\tau}(\thetab)\right)
+ \beta^{-1}\nabla_{\thetab}\nu_{\tau}\right)~,
\end{align}
whereas our framework focuses on the following one:
\begin{align}\label{eq:unifiednew1}
\partial_{\tau} \nu_{\tau} =& \nabla_{\thetab}\cdot \left(\nu_{\tau}\beta^{-1}F(\thetab) + \nu_{\tau}\left(E_{Y \sim\nu_{\tau}}K(\thetab-Y)F(Y)\right.\right. \nonumber\\
&\left.\left.-\nabla K*\nu_{\tau}(\thetab)\right)
+ \beta^{-1}\nabla_{\thetab}\nu_{\tau}\right)~.
\end{align}
The extra term $\nu_{\tau}\left(E_{Y \sim\nu_{\tau}}K(\thetab-Y)F(Y)\right)$ in our framework makes the analysis much more challenging. The main differences between our work and \cite{DurmusEGZ:arxiv18} including related work are summarized below:

\begin{itemize}
	\item Formulations are different. The extra term $E_{Y\sim\mu_{\tau}}K(\thetab-Y)F(Y)$ cannot be combined with the $F(\thetab)$ term in \eqref{eq:gran} in \cite{DurmusEGZ:arxiv18}. This is because function $F(\thetab)$ \textbf{itself} is a function independent of $\tau$; while $E_{Y\sim\mu_{\tau}}K(\thetab-Y)F(Y)$ depends on both $\thetab$ and $\tau$. This makes our problem much more difficult.
	\item Assumptions are different. For example, the analysis on granular media equations in \cite{CattiauxGM:PTRF08} requires that $F$ satisfies a special condition ${C(\mathbb{A}, \alpha)}$, which is a strong condition impractical to be satisfied in our case; And \cite{DurmusEGZ:arxiv18} adopts different assumptions from ours with a different goal.
	\item For the Euler integrator, \cite{DurmusEGZ:arxiv18} does not consider an Euler solution. Furthermore, our sampling method needs "stochastic gradient" {\it i.e.} $G_k^{(i)} \triangleq \frac{N}{B_k}\sum_{q\in\mathcal{I}_k}F_q(\theta_k^{(i)})$ in \eqref{eq:particle_num} for computational feasibility, which is quite different from the former work on particle-SDE such as \cite{Malrieu:AAP03,CattiauxGM:PTRF08}. Few of the former work on particle-SDE considered the stochastic gradient issue.
\end{itemize}

To sum up, the main purpose of our paper is to provide a non-asymptotic analysis of our method instead of improving the former work on a certain type of PDE. This is also the reason why we said that parts of our proof techniques are based on those for analyzing granular media equations. 

\section{Extra Experiments}\label{app:exp}

\subsection{Posterior sampling of a Gaussian model}\label{app:extpgau} 
We further follow \cite{ChenDC:NIPS15} and consider a relatively more complex Gaussian model for posterior sampling: $x_i \sim \mathcal{N}(\theta, 1), \theta \sim \mathcal{N}(0, 1)$, where 1000 
data samples $\{x_i\}$ are generated.
We adopt the same setting as above. The posterior average $\mathbb{E}_{\theta \sim p(\theta|\{x_i\})}[f(\theta)]$ endows an explicit expression. 
Figure~\ref{fig:gau_model} plots the error versus the running iterations for different particle sizes. It is observed that at the beginning, the errors for the ones with less particles decrease faster than those with more particles. This is reflected in the overall bound given in Theorem~\ref{theo:fixed}, which are dominated by the bound in Theorem~\ref{theonew:num_error} (indicating larger $M$ results in larger errors at the beginning). When more running time/iterations are given, the impact of the exponentially-decaying term in Theorem~\ref{theonew:Bt2} could be ignored. We also observe a trend of increasing errors when number of iterations are large enough, which is not drawn in the figure for simplicity.

\begin{figure}[t!]
	\centering
	\hspace{-0.0cm}
	\includegraphics[width=0.7\linewidth]{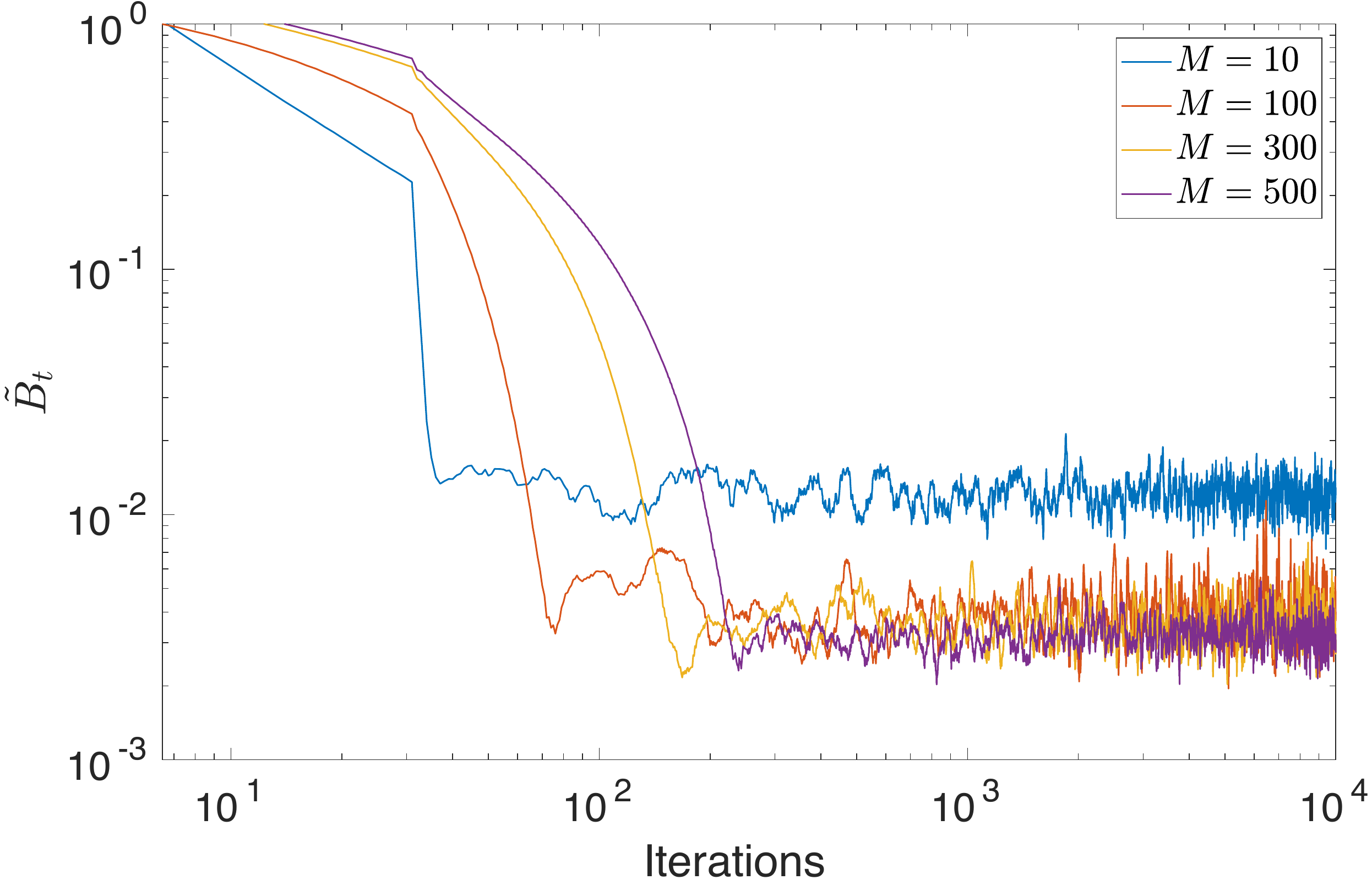}\\
	\vspace{-0.2cm}
	\caption{Errors versus \#iterations on a simple posterior Gaussian model.}\label{fig:gau_model}
\end{figure}

\subsection{Toy Experiments}

We compare the proposed SPOS with other popular methods such as SVGD and standard SGLD on four mutil-mode toy examples. We aim to sample from four unnormalized 2D densities $p(z) / \exp\{U(z)\}$, with the functional form provided in \cite{RezendeM:ICML15}. We optimize/sample 50 and 2000 particles to approximate the target distributions. The results are illustrated in Figure~\ref{fig:50toy} and Figure~\ref{fig:2000toy}, respectively.

\begin{figure}[h!] \centering
	\begin{tabular}{ccccc}
		\hspace{-4mm}
		\includegraphics[width=2cm]{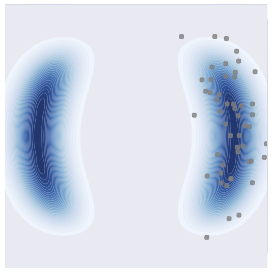}  
		&   \hspace{-5mm}
		\includegraphics[width=2cm]{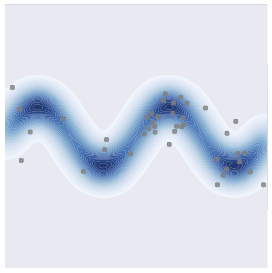} 
		& 		\hspace{-5mm}
		\includegraphics[width=2cm]{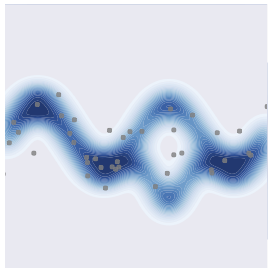}   
		&		\hspace{-5mm}
		\includegraphics[width=2cm]{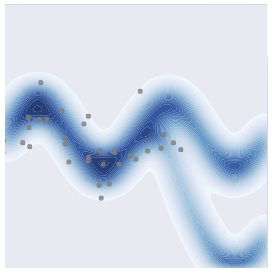}
		\vspace{-1mm}
		\\
		\hspace{-4mm}
		\includegraphics[width=2cm]{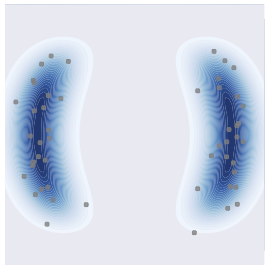}  
		&   \hspace{-5mm}
		\includegraphics[width=2cm]{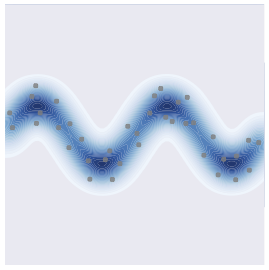} 
		& 		\hspace{-5mm}
		\includegraphics[width=2cm]{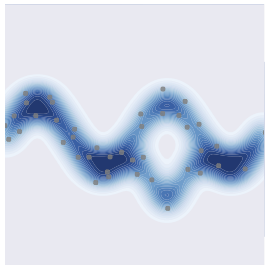}   
		&		\hspace{-5mm}
		\includegraphics[width=2cm]{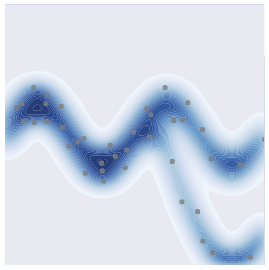}
		\vspace{-1mm}
		\\
		\hspace{-4mm}
		\includegraphics[width=2cm]{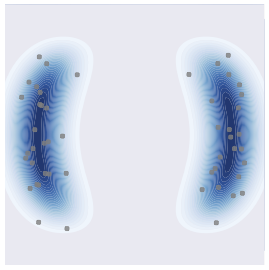}  
		&   \hspace{-5mm}
		\includegraphics[width=2cm]{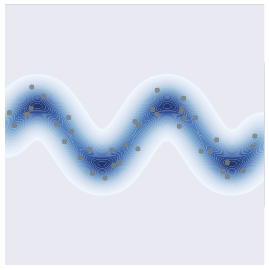} 
		& 		\hspace{-5mm}
		\includegraphics[width=2cm]{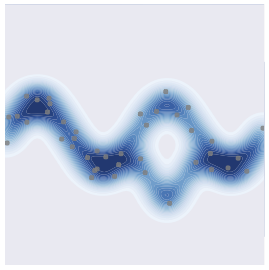}   
		&		\hspace{-5mm}
		\includegraphics[width=2cm]{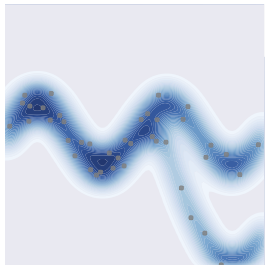} 
		\vspace{-1mm}
		\\
	\end{tabular}
	\caption{{Illustration of different algorithms on toy distributions. Dots are the final particles; the blue regions represent ground true densities. Each column is a distribution case. First row: standard SGLD; Second row: SVGD; Third row: SPOS}.}
	\label{fig:50toy}
	\vspace{-4mm}
\end{figure}

\begin{figure}[h!] \centering
	\begin{tabular}{ccccc}
		\hspace{-4mm}
		\includegraphics[width=2cm]{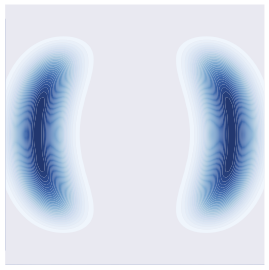}  
		&   \hspace{-5mm}
		\includegraphics[width=2cm]{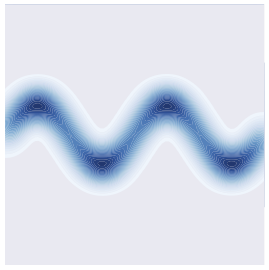} 
		& 		\hspace{-5mm}
		\includegraphics[width=2cm]{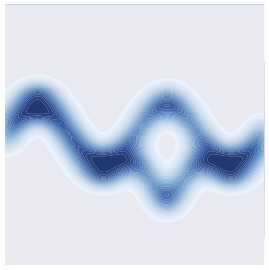}   
		&		\hspace{-5mm}
		\includegraphics[width=2cm]{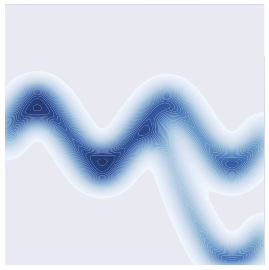}   
		\vspace{-1mm}
		\\
		\hspace{-4mm}
		\includegraphics[width=2cm]{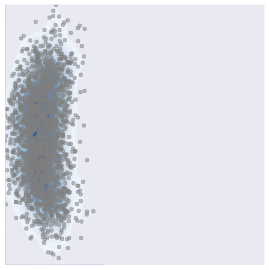}  
		&   \hspace{-5mm}
		\includegraphics[width=2cm]{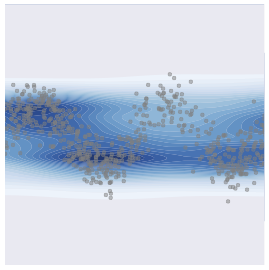} 
		& 		\hspace{-5mm}
		\includegraphics[width=2cm]{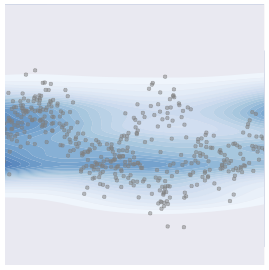}   
		&		\hspace{-5mm}
		\includegraphics[width=2cm]{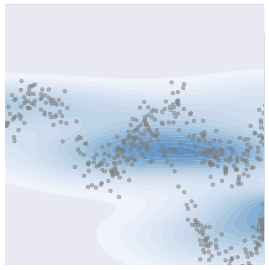}
		\vspace{-1mm}
		\\
		\hspace{-4mm}
		\includegraphics[width=2cm]{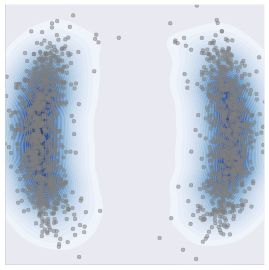}  
		&   \hspace{-5mm}
		\includegraphics[width=2cm]{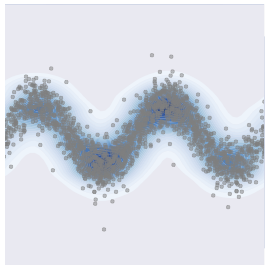} 
		& 		\hspace{-5mm}
		\includegraphics[width=2cm]{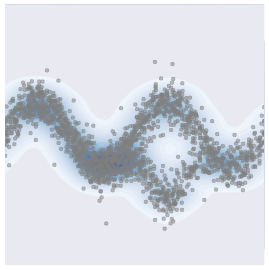}   
		&		\hspace{-5mm}
		\includegraphics[width=2cm]{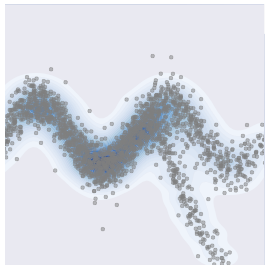}
		\vspace{-1mm}
		\\
		\hspace{-4mm}
		\includegraphics[width=2cm]{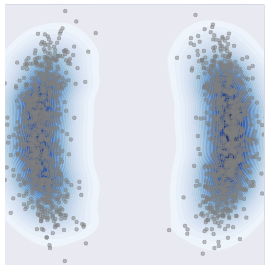}  
		&   \hspace{-5mm}
		\includegraphics[width=2cm]{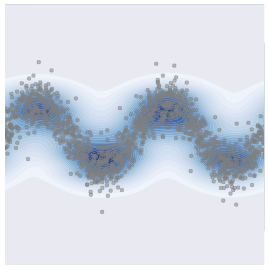} 
		& 		\hspace{-5mm}
		\includegraphics[width=2cm]{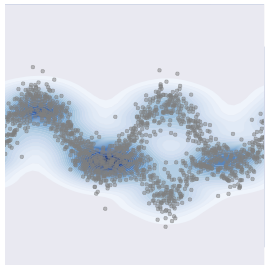}   
		&		\hspace{-5mm}
		\includegraphics[width=2cm]{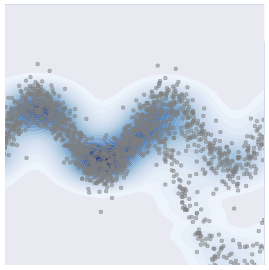} 
		\vspace{-1mm}
		\\
	\end{tabular}
	\caption{Illustration of different algorithms on toy distributions. Dots are the final particles; the blue regions represent densities estimated by the particles. Each column is a distribution case. First row: ground true densities; Second row: standard SGLD; Third row: SVGD; Fourth row: SPOS.}
	\label{fig:2000toy}
\end{figure}

\subsection{More details on Bayesian neural networks for regression}
The Bayesian DNNs are used to model weight uncertainty of neural networks, an important topic that has been well explored \cite{HernaandezLobatoA:icml15,BlundellCKW:icml15,PSGLD:AAAI16,LouizosW:ICML16}. We assign simple isotropic Gaussian priors to the weights, and perform posterior sampling with different methods. 
For SVGD and SPOS methods, we use a RBF kernel $K(\thetab,\thetab') = \exp(-\|\thetab-\thetab^\prime\|_2^2/\eta^2)$, with the bandwidth set to $\eta=\mathtt{med}^2/\log M$. Here $\mathtt{med}$ is the median of the pairwise distance between particles.
We use a single-layer BNN for regression tasks. Following \cite{li2015stochastic}, 10 UCI public datasets are considered: 100 hidden units for 2 large datasets (Protein and YearPredict), and 50 hidden units for the other 8 small datasets. 
Following \cite{ZhangLCC:AISTATS18}, we repeat the experiments 20 times with batchsize 100 for all datasets except for Protein and YearPredict, which we repeat 5 times and once with batchsize 1000. The datasets are randomly split into 90\% training and 10\% testing. For a fair comparison, we use the same split of data (train, val and test) for the three methods. The test results are reported on the best model on the validation set.
We adopt the root mean squared error (RMSE) and test log-likelihood as the evaluation criteria.
The experimental results are shown in Table~\ref{tab:reg}, from which we can see the proposed SPOS outperforms SVGD and other existing methods presented in \cite{ZhangLCC:AISTATS18} (results not shown due to space limit), achieving state-of-the-art results.

\begin{table}[t!]
	\centering
	\caption{\small Averaged predictions with standard deviations in terms of RMSE and log-likelihood on test sets.}
	\vskip 0.02in
	\scalebox{0.75}{
		\begin{tabular}{c|ccc|ccc}
			\hline
			& \multicolumn{3}{c}{Test RMSE} &\multicolumn{3}{c}{Test Log likelihood} \\
			Dataset & SGLD & SVGD&SPOS&SGLD&SVGD&SPOS\\
			\hline
			Boston\_Housing & 3.114 $\pm$ 0.144 & 2.961 $\pm$ 0.109 & $\mathbf{ 2.829 \pm 0.126 }$ &
			$-2.633 \pm 0.083 $ & -2.591 $\pm$ 0.029 & $\mathbf{ -2.532 \pm 0.082 }$\\
			Concrete& $ 5.508 \pm 0.275 $ & 5.157 $\pm$ 0.082 &$\mathbf{ 5.071 \pm 0.1495 }$ & $ -3.133 \pm 0.087 $ & -3.247 $\pm$ 0.01 &$\mathbf{ -3.062 \pm 0.037 }$
			\\
			Energy& $ 0.842 \pm 0.060 $ & 1.291 $\pm$ 0.029&$\mathbf{ 0.752 \pm 0.0285 }$  & $ -1.268 \pm 0.143 $ & -1.534 $\pm$ 0.026&$\mathbf{ -1.158 \pm 0.073 }$
			\\
			Kin8nm& $ 0.080 \pm 0.001 $ & 0.090 $\pm$ 0.001& $\mathbf{ 0.079 \pm 0.001 }$& $ 1.080 \pm 0.025 $ & 0.986 $\pm$ 0.004&$\mathbf{ 1.092 \pm 0.013 }$\\
			Naval& $ 0.004 \pm 0.000 $ & 0.004 $\pm$ 0.000 &$\mathbf{ 0.004 \pm 0.000 }$ & $ 4.127 \pm 0.028 $ & 4.032 $\pm$ 0.008 &$\mathbf{ 4.145 \pm 0.02 }$\\
			CCPP& $ 4.059 \pm 0.080 $ & 4.127 $\pm$ 0.027&$\mathbf{ 3.939 \pm 0.0495 }$ & $ -2.823 \pm 0.039 $ & -2.843 $\pm$ 0.006 &$\mathbf{ -2.794 \pm 0.025 }$\\
			Winequality & $ 0.632 \pm 0.022 $ & 0.604 $\pm$ 0.007 & $\mathbf{0.598 \pm 0.014}$ & $ -0.962 \pm 0.067 $ &-0.926 $\pm$ 0.009 & $\mathbf{ -0.911 \pm 0.041 }$ \\
			Yacht& $ 1.183 \pm 0.263 $ & 1.597 $\pm$ 0.099 & $\mathbf{ 0.84 \pm 0.0865 }$ & $ -1.680 \pm 0.393 $ & -1.818 $\pm$ 0.06 &$\mathbf{ -1.446 \pm 0.121 }$\\
			Protein& $ 4.281 \pm 0.011 $ & 4.392 $\pm$ 0.015 & $\mathbf{ 4.254 \pm 0.005 }$ & $ -2.877 \pm 0.002 $ & -2.905 $\pm$ 0.010 &$\mathbf{ -2.876 \pm 0.009 }$\\
			YearPredict& 8.707$\pm$ NA & 8.684 $\pm$ NA & $\mathbf{ 8.681 \pm NA }$ & -3.582  $\pm$ NA & -3.580 $\pm$ NA &$\mathbf{ -3.576 \pm NA }$\\
			\hline
		\end{tabular}
	}
	\label{tab:reg}
\end{table}

\subsubsection{Bayesian Neural Networks for MNIST classification}

We perform the classification tasks on the standard MNIST dataset.
A two-layer MLP 784-X-X-10 with $\ReLU$ activation function  is used, with X being the number of hidden units for each layer. The training epoch is set to 100. The test errors are reported in Table \ref{tab:fnn}. Surprisingly, the proposed SPOS outperforms other algorithms such as SVGD at a significant level, though it is just a simple modification of SVGD by adding in random Gaussian noise. This is partly due to the fact that our SPOS algorithm can jump out of local modes efficiently, as explained in Section~\ref{sec:svgd}.

\begin{table}
	\center
	\caption{Classification error of FNN on MNIST. }
	\vskip 0.1in
	\scalebox{0.9}{
		\begin{tabular}{llccc}	
			\hline
			\multirow{2}{*}{Method} & \multicolumn{2}{c}{Test Error}  \\
			& 400-400 & 800-800 \\
			\hline
			SPOS  & {\bf 1.32\%} & {\bf  1.24\%} \\
			SVGD  & 1.56\% &  1.47\% &   \\
			\hline
			SGLD	    	 &  1.64\% &    1.41\% \\
			RMSprop  	   &  1.59\% &    1.43\%  \\
			RMSspectral 	   		  &  1.65\% &    1.56\%  \\
			SGD 	   		  &   1.72\% &   1.47\%   \\
			\hline
			BPB, Gaussian    &  1.82\% & 1.99\%     \\
			SGD, dropout  & 1.51\%  & 1.33\%  \\
			\hline
	\end{tabular}}
	\label{tab:fnn}
\end{table}

\subsection{Bayesian exploration in deep RL}

We denote the policy as $\pi_{\thetab}(\ab|\sbb)$ parameterized by $\thetab$ with prior distribution $p(\thetab)$, where $\ab$ represent the action variable, and $\sbb$ the state variable. According to \cite{liu2017stein}, learning the optimal policy corresponds to calculating the following posterior distribution for $\thetab$: 
$q(\thetab)\propto \exp(J(\thetab)/\alpha) p(\thetab)$, 
where $J(\thetab)$ denotes the expected cumulative reward under the policy with parameter $\thetab$ and $\alpha$ a hyperparameter. Consequently, $\thetab$ could be updated by drawing samples from $q(\thetab)$ with the proposed SPOS. We denote this method as SPOS-PG. In addition, when drawing samples with SVGD, the resulting algorithm is called Stein variational policy gradient (SVPG) \cite{liu2017stein}. Note in implementation, the term $J(\thetab)$ can be approximated with REINFORCE~\cite{williams1992simple} or advantage actor critic~\cite{schulman2015high}, which we will investigate in our experiments.

The  policy is parameterized as a two-layer (25-10 hidden units) neural network with $\mathtt{tanh}$ as the activation function. The maximal length of horizon is set to 500. We use a sample size of 10000 for policy gradient estimation, and $M=16$, $\alpha = 10$. For the simplest task, Cartpole, all agents are trained for 100 episodes; whereas they are trained up to 1,000 episodes for the other two tasks. The average reward versus number of episodes are plotted
in Figure~\ref{fig:bayesian_exp}. It is observed that our SPOS-PG obtains much larger average rewards and smaller variance compared to SVPG, though the convergence behaviors are similar in the simplest Carpole task.

\begin{figure}
	\centering
	\includegraphics[width=0.49\linewidth]{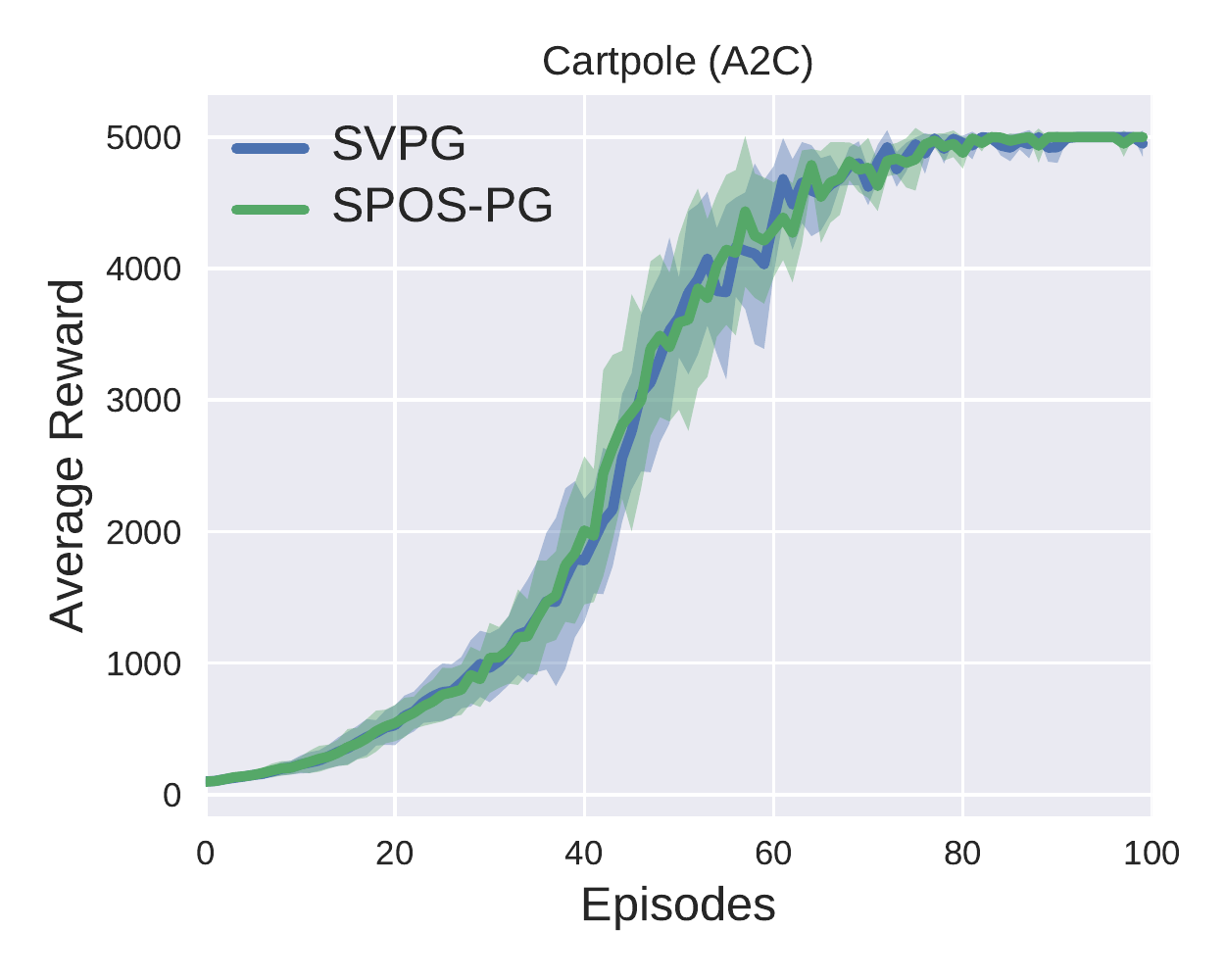}
	\includegraphics[width=0.49\linewidth]{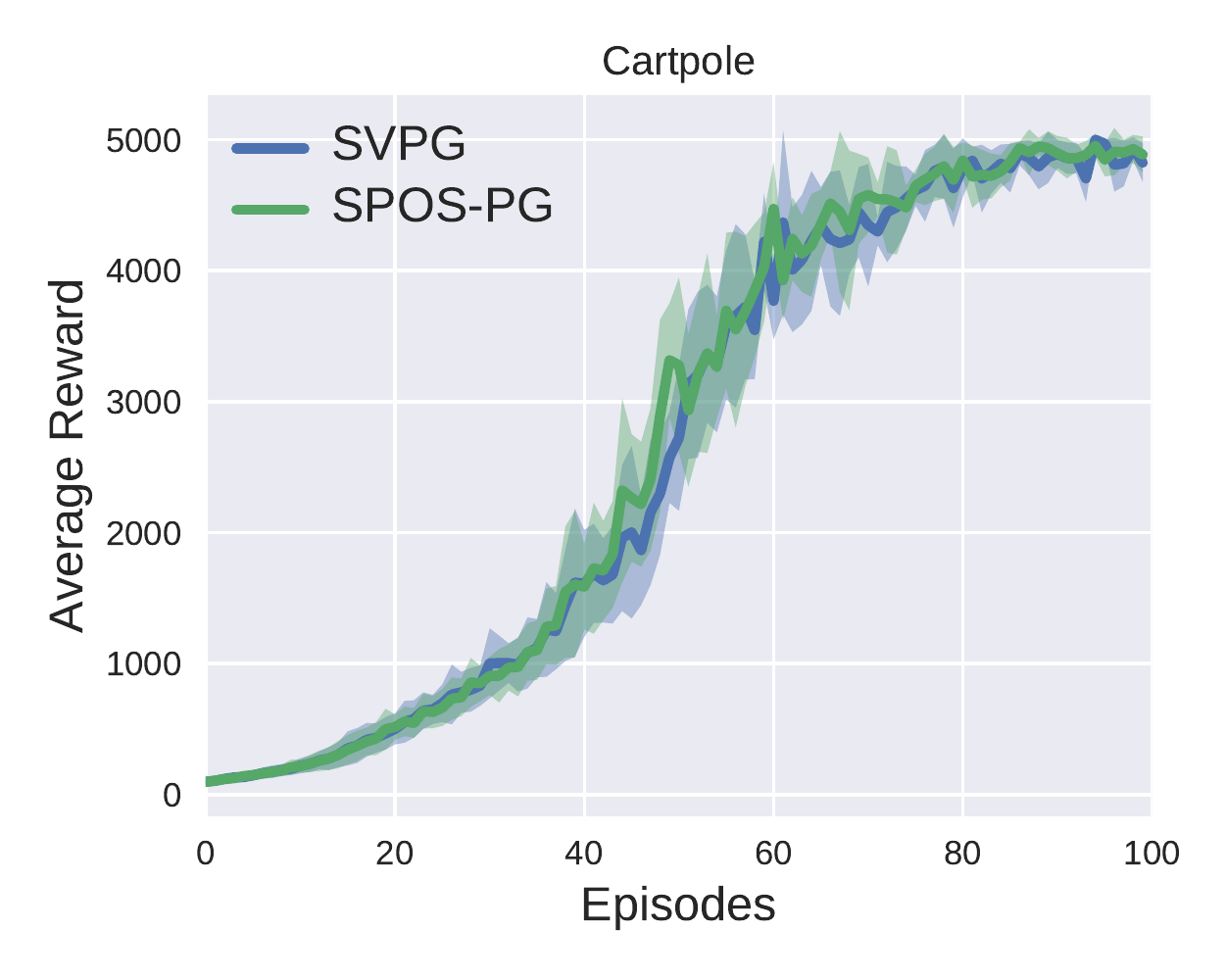}
	\includegraphics[width=0.49\linewidth]{./figures/A2CCartpoleSwingUp}
	\includegraphics[width=0.49\linewidth]{./figures/CartpoleSwingUp-workshop}
	\includegraphics[width=0.49\linewidth]{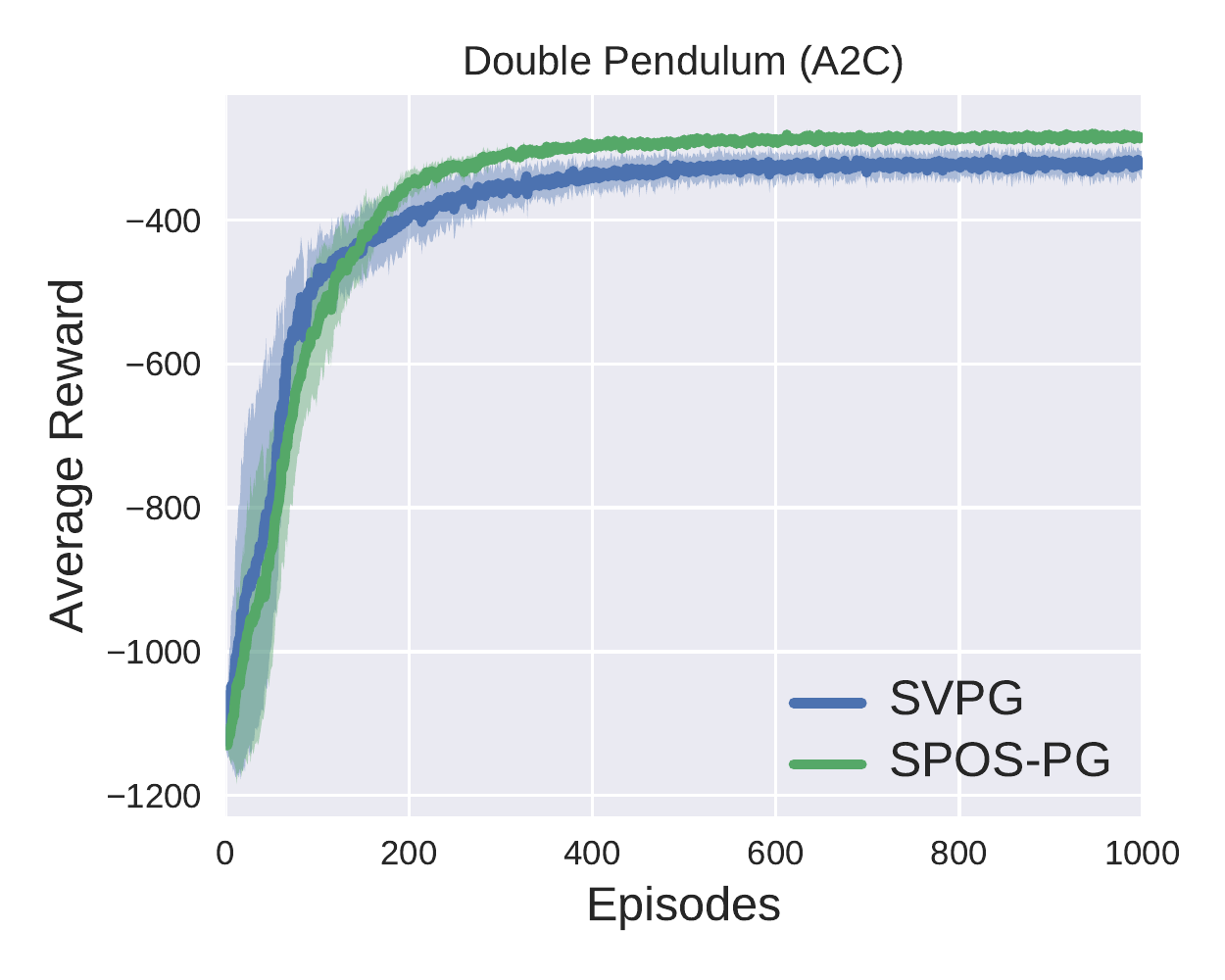}
	\includegraphics[width=0.49\linewidth]{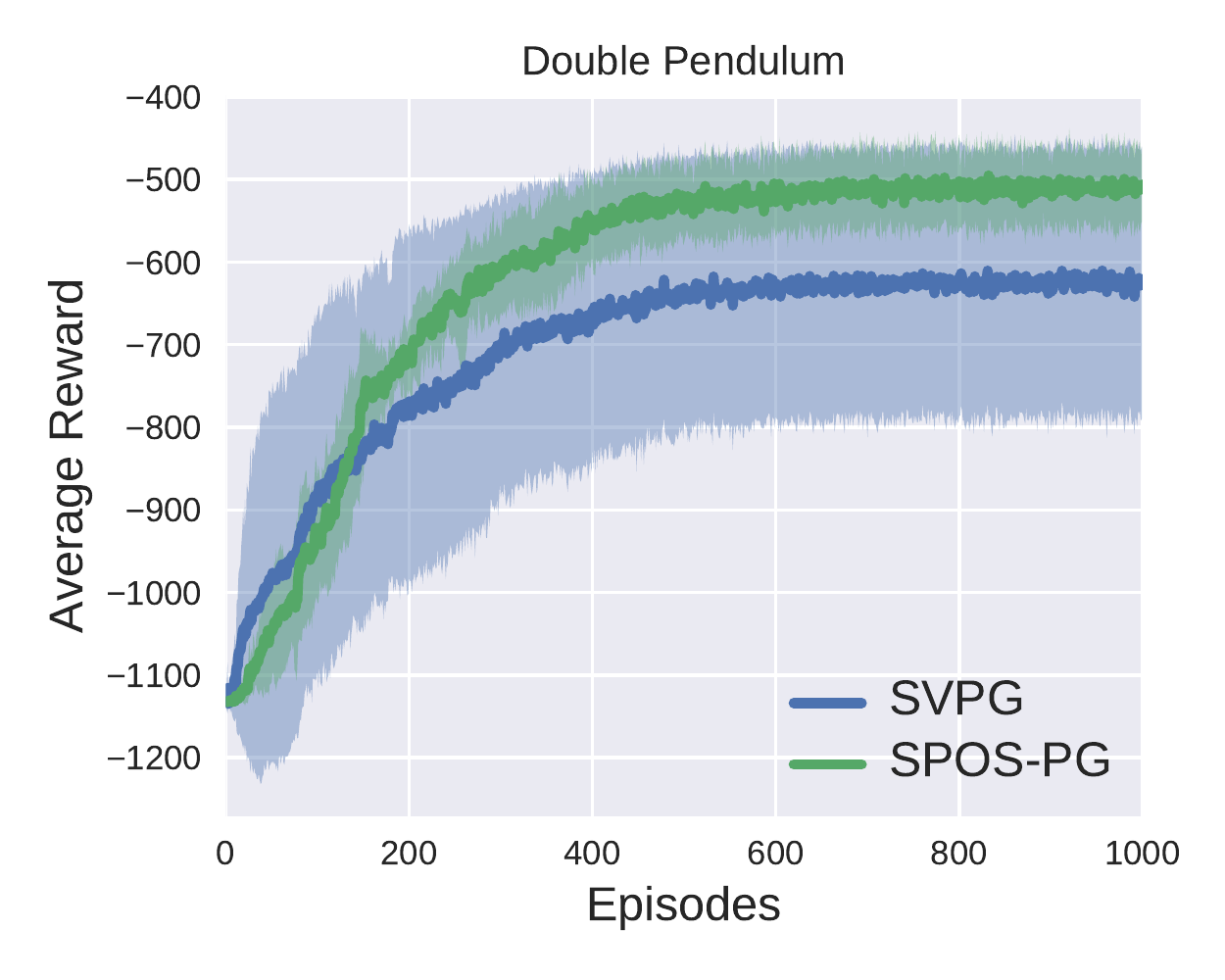}
	\vspace{-0.2cm}
	\caption{Policy learning with Bayesian exploration in policy-gradient methods on six scenarios with SVPG and SPOS-PG.}\label{fig:bayesian_exp}
	\vspace{-0.2cm}
\end{figure}

\end{document}